\documentclass{article}

\usepackage[preprint]{neurips_2026}
\usepackage[utf8]{inputenc}
\usepackage[T1]{fontenc}
\usepackage{hyperref}
\usepackage{url}
\usepackage{booktabs}
\usepackage{amsfonts}
\usepackage{nicefrac}
\usepackage{microtype}
\usepackage{xcolor}
\usepackage{amsmath,amssymb,amsthm,mathtools}
\usepackage{enumitem}
\usepackage{graphicx}
\usepackage{natbib}
\usepackage{booktabs}
\usepackage{tabularx}
\usepackage{array}
\newcolumntype{Y}{>{\raggedright\arraybackslash}X}

\newtheorem{theorem}{Theorem}
\newtheorem{lemma}{Lemma}
\newtheorem{proposition}{Proposition}
\newtheorem{corollary}{Corollary}

\newcommand{\Prob}{\mathbb{P}}

\newcommand{\Hnew}{\mathcal{H}_{\mathrm{new}}}

\newcommand{\Ltrain}{\mathcal{L}_{\mathrm{train}}}
\newcommand{\Ltest}{\mathcal{L}_{\mathrm{test}}}
\newcommand{\epsgen}{\epsilon_{\mathrm{gen}}^{\mathrm{norm}}}
\newcommand{\epsM}{\epsilon_M}
\newcommand{\epsK}{\epsilon_K}

\newcommand{\fold}{f^{*}_{\mathrm{old}}}
\newcommand{\fnew}{f_{\mathrm{new}}}

\newcommand{\ftilde}{\tilde{f}}

\newcommand{\DRtest}{\Delta_R^{\mathrm{test}}}
\newcommand{\DERM}{\Delta_{\mathrm{ERM}}}
\newcommand{\epsTT}{\epsilon_{\mathrm{TT}}}

\usepackage[utf8]{inputenc} 
\usepackage[T1]{fontenc}    
\usepackage{hyperref}       
\usepackage{url}            
\usepackage{booktabs}       
\usepackage{amsfonts}       
\usepackage{nicefrac}       
\usepackage{microtype}      
\usepackage{xcolor}         

\title{A Qualitative Test-Risk Mechanism for Scaling  Behavior in Normalized Residual Networks}

\author{%
   Cheng Daning\\
  ICT\\
  \texttt{chengdaning@ict.ac.cn} \\
  \And
    Liu Zeyu \\
   ICT \\
   \texttt{Lzy958354467@163.com} \\
   \And
   Sun Jun \\
   Zhejiang Lab \\
   \texttt{sunjun16sj@gmail.com} \\
   \And
   Xia Fen \\
   JD corporation \\
   \texttt{Xiafen3@jd.com} \\
   \And
   Zhang Boyang \\
   ICT \\
   \texttt{zhangby01@pcl.ac.cn} \\
   \And
   Liu Dongping \\
   AFS group \\
   \texttt{dongping.liu@ieee.org} \\
      \AND
  Zhang Yunquan \\
   ICT \\
   \texttt{zyq@ict.ac.cn} \\
}

\begin{document}

\maketitle

\begin{abstract}
 The scaling behavior, in which test performance often improves as model size and data increase, is a central empirical phenomenon in modern deep learning, yet its theoretical basis remains incomplete. In this paper, we study depth expansion in normalized residual networks: starting from a trained model in an old hypothesis class, we insert a new residual block at an intermediate layer and ask when such an expansion can yield a provable improvement in test risk. We develop a unified framework that decomposes this question into representational gain, optimization gain, and generalization transfer. First, under a first-order descent condition near zero initialization, we prove that the expanded hypothesis class contains an auxiliary jumpboard model with strictly smaller population risk than the original model. Second, under norm control tailored to post-normalized residual architectures, we establish a norm-based Rademacher complexity bound for the expanded model class. These ingredients lead to two complementary test-risk guarantees: one route passes through population risk and is tighter when a positive population margin is available, while the other works directly at the train/test level, avoids Hoeffding transfer, and is more robust in degenerate regimes.  Together, these results provide a theorem-driven mechanism under which residual depth expansion can improve test performance in normalized residual networks. More broadly, they suggest that scaling is inherently joint: depth creates new improving directions, width enhances the finite-sample observability of weak signals, and data determines whether the statistical cost of expansion can be controlled.  
\end{abstract}

\section{Introduction}
Large-scale deep learning systems exhibit a striking empirical regularity: as model size and training data increase, test performance often improves predictably. This regularity, often summarized by empirical scaling laws, has become central in modern machine learning,  yet its theoretical basis remains incomplete. In particular, explaining when increasing depth improves \emph{test risk} requires more than observing that a deeper model class is larger: one must also show that the enlarged class contains a genuinely better function, that optimization can realize this gain, and that the gain survives finite-sample error.

In this paper, we study depth expansion in normalized residual networks. Starting from a trained reference model $f_{old}^*$, we insert a new residual block after an intermediate layer and ask when this expansion can be converted into a provable improvement in test risk. Our analysis is built on the decomposition
\[
\text{jumpboard model}
\;\longrightarrow\;
\text{optimization gain}
\;\longrightarrow\;
\text{generalization transfer}.
\]
At the representation level, we show that if the inserted block creates a nontrivial first-order descent direction near zero initialization, then the expanded hypothesis class contains a \emph{jumpboard model} with strictly smaller population risk than $f_{old}^*$. At the optimization level, we ask whether training can realize this potential gain. At the generalization level, we control the statistical cost of the enlarged class and translate the population-level gain into a test-risk guarantee.

Under this framework, we obtain two complementary qualitative scaling-mechanism  results. The first route passes through population risk, yielding a tighter test-risk inequality when a positive deterministic margin is available. The second route works directly at the train/test level, avoids Hoeffding transfer from population risk, and is therefore more robust in degenerate regimes where the population margin becomes too small to transfer sharply. Together, these results show that when representational gain and optimization gain dominate generalization cost, increasing depth does not worsen test performance; under stronger conditions, it yields strict improvement.

Our analysis also clarifies the distinct roles of width and data. Width improves the finite-sample observability of weak improvement signals, especially near the deepest-model regime, while data controls the decay of the generalization penalty. Thus, the relevant scaling picture is inherently joint: depth creates new improving directions, width stabilizes their observability, and data pays for the statistical cost of model expansion.

Our contributions are threefold. First, we propose a unified analytical framework that decomposes depth scaling into representational gain, optimization gain, and generalization transfer, thereby avoiding the common shortcut of equating a stronger hypothesis class with better test performance. Second, we establish two complementary test-set scaling-mechanism  theorems: one is tighter because it relies on a deterministic population-risk margin, while the other is more robust because it works directly at the train/test level and remains meaningful in degenerate regimes. Third, we clarify the theoretical roles of width and data in scaling: width improves finite-sample signal observability, while data governs the decay of the generalization penalty and determines whether model expansion remains beneficial.  The framework naturally extends to post-normalized residual architectures, including ResNet-like and Transformer-like blocks, provided the corresponding regularity and first-order non-degeneracy conditions hold.

\section{Related Work}
 
Empirical studies show  power-law gains with increased scale \citep{hestness2017predictable, kaplan2020scaling, hoffmann2022training, rosenfeld2020constructive, hernandez2021transfer}. While recent theory explores global scaling trends \citep{bahri2024explaining}, we focus specifically on architecture-specific test-risk improvements. Prior work shows skip connections stabilize training \citep{he2016deep, he2016identity, veit2016residual, balduzzi2017shattered} and depth increases expressivity \citep{telgarsky2016benefits, eldan2016power}. Other analyses address trainability via mean-field limits \citep{yang2017meanfield, lu2020meanfield}, optimization effects of overparameterized depth \citep{arora2018optimization}, dynamical-systems and stability viewpoints \citep{haber2018stable, chang2018reversible}, or specific initializations \citep{zhang2019fixup, de2020batchnorm, bachlechner2021rezero}. Conversely, we prove when residual blocks directly improve \emph{test risk}. Classical \citep{bartlett2017spectral, golowich2018sizeindependent, neyshabur2018pacbayes, arora2018compression} and ResNet-specific Rademacher bounds \citep{e2020rademacher, kammonen2020smaller} establish complexity upper bounds.  The  comparison between these works and our work is shown in Appendix~\ref{app sec:related work}.

\section{Notation and Assumptions}
\label{sec:notation}

Let $\mathcal D$ be a distribution on $\mathcal X\times\mathcal Y$. Let
$S_{\mathrm{train}}=\{(x_i,y_i)\}_{i=1}^M\sim\mathcal D^M$ and
$S_{\mathrm{test}}=\{(\tilde x_j,\tilde y_j)\}_{j=1}^K\sim\mathcal D^K$
be independent training and test sets. The loss is
$\ell:\mathbb R^C\times\mathcal Y\to[0,B_\ell]$.
For any model $f$, define
$\mathcal L_{\mathrm{train}}(f)=\frac1M\sum_{i=1}^M \ell(f(x_i),y_i),\quad
\mathcal L_{\mathrm{test}}(f)=\frac1K\sum_{j=1}^K \ell(f(\tilde x_j),\tilde y_j),\quad
R(f)=\mathbb E_{(x,y)\sim\mathcal D}[\ell(f(x),y)]$.

We also define 
activation-gradient averages at layer \(l^\ast\) by $\bar\mu :=
\mathbb E_{(x,y)\sim\mathcal D}
\left[
\nabla_{x^{(l^\ast)}}\ell\!\left(f_{\mathrm{top}}(x^{(l^\ast)}),y\right)
\right]
$, $\mu
:=
\frac1M\sum_{i=1}^M
\nabla_{x_i^{(l^\ast)}}\ell\!\left(f_{\mathrm{top}}(x_i^{(l^\ast)}),y_i\right)$, $
g
:=
\frac1K\sum_{j=1}^K
\nabla_{\tilde x_j^{(l^\ast)}}\ell\!\left(f_{\mathrm{top}}(\tilde x_j^{(l^\ast)}),\tilde y_j\right)$.

We decompose the network at layer $l^\ast$ as $f=f_{\mathrm{top}}\circ f_{\mathrm{bot}}$, where
$f_{\mathrm{bot}}:\mathcal X\to\mathbb R^N$ and $f_{\mathrm{top}}:\mathbb R^N\to\mathbb R^C$.
The old hypothesis class is $\mathcal{H}_{\mathrm{old}}
=
\left\{
x\mapsto f_{\mathrm{top}}(f_{\mathrm{bot}}(x)):
f_{\mathrm{top}}\in\mathcal{F}_{\mathrm{top}},
\ f_{\mathrm{bot}}\in\mathcal{F}_{\mathrm{bot}}
\right\}$.
and $f_{\mathrm{old}}^\ast\in\mathcal H_{\mathrm{old}}$ denotes a trained old model.
After inserting a residual block $h$ after layer $l^\ast$, the new hypothesis class is $
\mathcal{H}_{\mathrm{new}}
=
\left\{
x\mapsto
f_{\mathrm{top}}\bigl(f_{\mathrm{bot}}(x)+h(f_{\mathrm{bot}}(x))\bigr):
f\in\mathcal{H}_{\mathrm{old}},\ h\in\mathcal{F}_{\mathrm{res}}
\right\}$.
 We parameterize the inserted block by $\theta$, $h_\theta(\cdot)$, and assume $h_0=0$, so zero initialization recovers the old model.

A \emph{jumpboard model} is a near-identity model in $\mathcal H_{\mathrm{new}}$ that improves over $f_{\mathrm{old}}^\ast$ along a local descent direction at zero initialization. 
We distinguish two versions.  The \emph{population jumpboard model}, denoted by $\tilde f_{pop}$, is
constructed from a population first-order descent direction and is used in the
population-risk route. Its population margin is denoted by $\Delta_R
:=
R(f_{\mathrm{old}}^\ast)-R(\tilde f_{pop})>0$. The \emph{training-set jumpboard model}, denoted by $\tilde f_S$, is constructed from
an empirical first-order descent direction on $S_{\mathrm{train}}$. Its empirical
jumpboard margin is denoted by $\Delta_{\mathrm{train},S}
:=
\mathcal L_{\mathrm{train}}(f_{\mathrm{old}}^\ast)
-
\mathcal L_{\mathrm{train}}(\tilde f_S)>0$.
The optimization gain relative to the comparison jumpboard model is denoted by $\Delta_{\mathrm{ERM}}\ge 0$.

Let $\epsilon_{\mathrm{gen}}^{\mathrm{norm}}(\mathcal H,n,\delta)$ denote the normalized-residual-network generalization bound. For brevity, we write $\epsilon_M=\epsilon_{\mathrm{gen}}^{\mathrm{norm}}(\mathcal H_{\mathrm{new}},M,\delta)$, 
$\epsilon_K=\epsilon_{\mathrm{gen}}^{\mathrm{norm}}(\mathcal H_{\mathrm{new}},K,\delta)$. In the main theorems, the total generalization cost is $2\epsilon_M$ for the population-risk route and $2(\epsilon_M+\epsilon_K)$ for the direct train/test route.  $\epsilon_{\mathrm{gen}}^{\mathrm{norm}}$ is a shorthand for the explicit norm-based generalization bound in Lemma~\ref{lem:gen-main}, Corollary \ref{cor:gen-main} in Appendix \ref{app:proof-lem-gen-main}, not an additional assumption.

\textbf{Assumption 1 (pathwise differentiability under expectation).} For any admissible parameter direction \(v\) and scalar step size \(\eta\)
in a neighborhood of \(0\), write $h_{\eta v}(z):=h_\theta(z)\big|_{\theta=\eta v}$. For every admissible direction \(v\), define $G_v(\eta)
:=
\mathbb E_{(x,y)\sim D}
\left[
\ell\left(
f_{\mathrm{top}}\left(
z+h_{\eta v}(z)
\right),y
\right)
\right], z=f_{\mathrm{bot}}(x)$. We assume that \(G_v\) is differentiable at \(\eta=0\), and that
differentiation may be exchanged with expectation: $G_v'(0)
=
\mathbb E_{(x,y)\sim D}
\left[
\nabla_z\ell(f_{\mathrm{top}}(z),y)^\top
\frac{\partial h_\theta(z)}{\partial\theta}
\bigg|_{\theta=0}v
\right]$.  

\textbf{Assumption 2 (Lipschitz boundedness).}
The loss $\ell$ is $L_\ell$-Lipschitz and bounded in $[0,B_\ell]$. The top map $f_{\mathrm{top}}$ is $ L_{\mathrm{top}}$-Lipschitz and satisfies $\|f_{\mathrm{top}}(0)\|\le B_0$.

\textbf{Assumption 3 (gradient-covariance control).}
Let $\zeta=\nabla_{x^{(l^\ast)}}\ell$ and $\Sigma=\mathrm{Cov}(\zeta)$. There exist a width-independent constant $C_\Sigma>0$ and a scalar $\tau^2$ such that
$\lambda_{\max}(\Sigma)\le C_\Sigma\tau^2, \tau^2\ge \max_{1\le i\le N} \mathrm{Var}(\zeta_i)$.
 
This assumption is intended as a mild structural condition in practical settings.  Our experiments suggest that different coordinates of activation gradients are nearly uncorrelated.

\textbf{Assumption 4 (normalized residual network regularity).}
Consider an $L$-layer normalized residual network with normalized function $\mathcal N_l$: $T_l(z)=\mathcal N_l(z+h_l(z)), l=0,\dots,L-1$. Assume:
(i) each residual branch $h_l(\cdot;W_l)$ is $s_l$-Lipschitz in the input;
(ii) $\|h_l(z;W_l)-h_l(z;W_l')\|\le L_l^{(p)}\|W_l-W_l'\|_F\,\|z\|$;
(iii) there exists $r_{\min}>0$ such that $\inf_{x,\theta,l}\|z_l(x)+h_l(z_l(x);\theta)\|\ge r_{\min}$;
(iv) $\|W_l\|_\sigma\le s_l$, $\|W_l\|_F\le b_l$, and $\|x\|\le B_x$.

\textbf{Assumption 4$\epsilon$ (engineering-$\epsilon$ alternative).}
Assumption~4(iii) is an idealized non-degeneracy condition used to ensure
that the normalization operators are Lipschitz on the reachable set. In
practical normalized architectures, possible singularities of normalization
operators are often removed by an engineering constant
\(\epsilon_{\mathrm{eng}}>0\). We assume that the resulting stabilized
normalization operators are uniformly Lipschitz and bounded on the reachable
set: for each layer \(l\), there exist finite constants
\(c_{l,\epsilon}>0\) and \(B_{\mathrm{norm},\epsilon}>0\) such that, for all
reachable normalization inputs \(a,a'\), $\|\mathcal N_{l,\epsilon}(a)-\mathcal N_{l,\epsilon}(a')\|
\le
c_{l,\epsilon}\|a-a'\|,
\|\mathcal N_{l,\epsilon}(a)\|\le B_{\mathrm{norm},\epsilon}$. Under this alternative, Assumption~4(iii) can be dropped, with the constants
inside \(\epsilon_{\mathrm{gen}}^{\mathrm{norm}}\) replaced by
\(\epsilon_{\mathrm{eng}}\)-dependent constants. Concrete stabilized
normalization forms and their constant-level effects are given in
Appendix \ref{sec:remove-N-by-eps}. 


\textbf{Condition 1 (first-order non-degeneracy of the inserted residual block).}
The inserted residual block class \(\mathcal F_{\mathrm{res}}\) is parameterized by
\(\theta\), satisfies \(h_0=0\), and is differentiable at \(\theta=0\). 
We assume that there exists at least one admissible insertion layer \(l^\ast\)
and one direction \(v_{\mathrm{pop}}\) such that $\mathbb E_{(x,y)\sim\mathcal D}
\left[
\nabla_{x^{(l^\ast)}}\ell
\bigl(
f_{\mathrm{top}}(x^{(l^\ast)}),y
\bigr)^\top
\left.
\frac{\partial h_\theta(x^{(l^\ast)})}{\partial\theta}
\right|_{\theta=0}
v_{\mathrm{pop}}
\right]
<0$.
This is the population first-order descent condition used in the population-risk
route.

For the empirical train/test route, we use the training-set analogue. Given
$S_{\mathrm{train}}=\{(x_i,y_i)\}_{i=1}^M$, write
$x_i^{(l^\ast)}:=f_{\mathrm{bot}}(x_i)$.
We say that \(v_S\) is an empirical jumpboard direction if $\frac1M\sum_{i=1}^M
\nabla_{x_i^{(l^\ast)}}\ell
\bigl(
f_{\mathrm{top}}(x_i^{(l^\ast)}),y_i
\bigr)^\top
\left.
\frac{\partial h_\theta(x_i^{(l^\ast)})}{\partial\theta}
\right|_{\theta=0}
v_S
<0$.
The population direction \(v_{\mathrm{pop}}\) is used in the population-risk
route, whereas the empirical direction \(v_S\) is used in the direct train/test
route.

This condition is a mild non-degeneracy constraint rather than a strong global assumption.  Condition~1 is not intended to hold for every trained model or every possible
insertion layer. Rather, it characterizes the improvable regime of depth
expansion: the inserted residual block must introduce a first-order direction
that is not orthogonal to the relevant activation-gradient signal. If no such
direction exists, the present first-order mechanism is exhausted and our
theorems make no strict-improvement claim.  Moreover, the parameter stationarity of the old trained model only makes the loss gradient orthogonal to the original tangent directions, which neither implies the activation gradient is zero nor orthogonal to the new tangent directions introduced by residual blocks. Further discussion of Condition 1 is provided in Appendix~\ref{app:assumption2-clarification} and Corollary \ref{cor:zero-output-residual-block}.  Corollary \ref{cor:zero-output-residual-block} further shows that this condition is compatible with standard zero-output residual blocks: it reduces to requiring that the new residual features are not orthogonal to the remaining activation-gradient signal.

\textbf{Condition 2 (selection  condition).} We do not attempt to prove convergence or algorithm-specific properties of the optimizer.
Instead, Condition 2 isolates the optimization issue from the statistical risk-comparison argument. 
Let $\widetilde f\in\mathcal H_{\mathrm{new}}$ denote the comparison jumpboard model,  $\widetilde f_{pop}$ for population-risk route, and $\widetilde f_{S}$ for empirical train/test route.
Let $f_{\mathrm{alg}}\in\mathcal H_{\mathrm{new}}$ be the model returned by the training
algorithm on the training dataset. We define the final selected model by the
fallback rule $f_{\mathrm{new}}
\in
\arg\min_{f\in\{f_{\mathrm{alg}},\widetilde f\}} \mathcal L_{\mathrm{train}}(f)$. Then, by construction, $
\mathcal L_{\mathrm{train}}(f_{\mathrm{new}})
\le
\mathcal L_{\mathrm{train}}(\widetilde f)$. Equivalently, there exists $\Delta_{\mathrm{ERM}}
:=
\mathcal L_{\mathrm{train}}(\widetilde f)
-
\mathcal L_{\mathrm{train}}(f_{\mathrm{new}})
\ge 0$
such that $ \mathcal L_{\mathrm{train}}(f_{\mathrm{new}})
=
\mathcal L_{\mathrm{train}}(\widetilde f)-\Delta_{\mathrm{ERM}}$. When $\Delta_{\mathrm{ERM}}=0$, the selected model matches the jumpboard model on
the training sample. When $\Delta_{\mathrm{ERM}}>0$, the training algorithm extracts
an additional empirical gain from the expanded class.

This is a very weak no-loss condition. If the model returned by the training
algorithm performs no better than the jumpboard model on the training set, we may
simply select the jumpboard model itself as $f_{\mathrm{new}}$. Therefore the
optimization comparison does not require global convergence or a strong
optimization guarantee; it only formalizes the fact that the final selected model
is not worse, on the training sample, than the near-identity jumpboard comparison
model. For the population-risk route, this condition is an analytical device.

\section{Main Results}
\label{sec:main-results}

Throughout this section, assume Assumptions 1 to 3 and either Assumption 4 or
Assumption \(4\epsilon\), together with Conditions 1 and 2. For each route,
\(f_{\mathrm{new}}\) denotes the final selected model obtained from Condition~2
with the corresponding comparison jumpboard model: \(\tilde f_{pop}\) in the
population-risk route and \(\widetilde f_S\) in the empirical train/test route.

\begin{theorem}[Qualitative scaling mechanism in population risk]
\label{thm:3}
Let $\tilde f_{pop}\in\mathcal H_{\mathrm{new}}$ be the jumpboard model selected along a local descent path, and suppose $\mathcal L_{\mathrm{train}}(f_{\mathrm{new}})\le \mathcal L_{\mathrm{train}}(\tilde f_{pop})-\Delta_{\mathrm{ERM}},\Delta_{\mathrm{ERM}}\ge 0$.  Then there exists a deterministic $\Delta_R>0$, depending only on $\mathcal D$ and the current old model, such that with probability at least $1-2\delta$,
\begin{small}
\begin{equation*}
  R(f_{\mathrm{new}})\le R(f_{\mathrm{old}}^\ast)-\Delta_R-\Delta_{\mathrm{ERM}}+2\epsilon_M.  
\end{equation*}
\end{small}

In particular, if $\Delta_R+\Delta_{\mathrm{ERM}}>2\epsilon_M$, then $R(f_{\mathrm{new}})<R(f_{\mathrm{old}}^\ast)$.
\end{theorem}

\begin{theorem}[Test-set qualitative scaling mechanism via population risk]
\label{thm:4A}
Let $\tilde f_{pop}\in\mathcal H_{\mathrm{new}}$ be the jumpboard model selected along a local descent path, and suppose  $\mathcal L_{\mathrm{train}}(f_{\mathrm{new}})\le \mathcal L_{\mathrm{train}}(\tilde f_{pop})-\Delta_{\mathrm{ERM}}, \Delta_{\mathrm{ERM}}\ge 0. $Then there exist deterministic quantities $\eta_0>0$ and $\Delta_R>0$, depending only on $\mathcal D$ and the current old model, such that with probability at least $1-2\delta-6\exp\!\left(-\frac{K\Delta_R^2}{8B_\ell^2}\right)$, one has

\begin{small}
\begin{equation*}
\mathcal L_{\mathrm{test}}(f_{\mathrm{new}})\le \mathcal L_{\mathrm{test}}(f_{\mathrm{old}}^\ast)-\frac{\Delta_R}{2}-\Delta_{\mathrm{ERM}}+2\epsilon_M.
\end{equation*}
\end{small}
 
Hence,  if $\Delta_R/2+\Delta_{\mathrm{ERM}}>2\epsilon_M$, then $\mathcal L_{\mathrm{test}}(f_{\mathrm{new}})<\mathcal L_{\mathrm{test}}(f_{\mathrm{old}}^\ast)$.
\end{theorem}

\begin{theorem}[Direct train/test comparison with a random test margin]
\label{thm:4B}
Let $\widetilde f_S\in\mathcal H_{\mathrm{new}}$ be the jumpboard model constructed
from a first-order descent direction, and suppose $\mathcal L_{\mathrm{train}}(f_{\mathrm{new}})
\le
\mathcal L_{\mathrm{train}}(\widetilde f_S)-\Delta_{\mathrm{ERM}}$, $\Delta_{\mathrm{ERM}}\ge 0$. 
Define the finite-test-set random margin $\Delta_R^{\mathrm{test}}
:=
\mathcal L_{\mathrm{test}}(f_{\mathrm{old}}^\ast)-\mathcal L_{\mathrm{test}}(\widetilde f_S)$.
Then, with probability at least $1-3\delta$,

\begin{small}
\begin{equation*}
\mathcal L_{\mathrm{test}}(f_{\mathrm{new}})
\le
\mathcal L_{\mathrm{test}}(f_{\mathrm{old}}^\ast)
-
\Delta_R^{\mathrm{test}}
-
\Delta_{\mathrm{ERM}}
+
2(\epsilon_M+\epsilon_K).
\end{equation*}
\end{small}
 
Consequently, on the event $\Delta_R^{\mathrm{test}}+\Delta_{\mathrm{ERM}}
>
2(\epsilon_M+\epsilon_K)$, one has $\mathcal L_{\mathrm{test}}(f_{\mathrm{new}})
<
\mathcal L_{\mathrm{test}}(f_{\mathrm{old}}^\ast)$.

 \end{theorem}

Theorem~\ref{thm:4B} does not assume that 
\(\Delta_R^{\mathrm{test}}\) is deterministically positive. 
For the training-set jumpboard model \(\tilde f_S\), positivity of this
finite-sample margin is supported probabilistically by the empirical
train/test alignment mechanism, i.e., $\mathbb P(\mu^\top g\le 0)$. We show this in Appendix  \ref{app:proof-thm4B}, Corollary \ref{cor:alignment-test-margin}. The relevant failure event has two sources:
the empirical direction \(\mu\) may fail to align with the population direction
\(\bar\mu\), and the independent test direction \(g\) may fluctuate away from
\(\bar\mu\). Under the covariance-control condition in Appendix~F, the empirical
alignment bound gives $ \mathbb P(\mu^\top g\le 0)
\le
\frac{4C_\Sigma\tau^2}{M\|\bar\mu\|^2}
+
\frac{4C_\Sigma\tau^2}{K\|\bar\mu\|^2}
+
\frac{4C_\Sigma\tau^2\operatorname{tr}(\Sigma)}
{KM\|\bar\mu\|^4}$.
Therefore the probability that the training-selected direction remains
positively aligned with the independent test direction increases with both the
training size \(M\), the test size \(K\), and the width-dependent signal
\(\|\bar\mu\|^2\). In the uniformly active regime
\(\|\bar\mu\|^2=\Theta(N)\),   this becomes $\mathbb P(\mu^\top g\le 0)
=
O\!\left(
\frac1{MN}
+
\frac1{KN}
\right)$, up to the mixed covariance term.
Thus width does not make the finite-test margin deterministic; rather, it
increases the probability that the empirical jumpboard gain is visible on the
test set, while the training size controls the reliability of the
training-selected direction.

\section{Proof Roadmap}
\label{sec:roadmap}

The three risk-transfer statements in Section~4 are supported by three ingredients. 

\textbf{Depth creates a better direction.}
The first ingredient is a local existence statement: under Assumption 1 and Condition 1, the inserted residual block is not first-order dead at zero initialization. Hence there exists a descent direction in parameter space and therefore a jumpboard model $\tilde f$ in $\mathcal H_{\mathrm{new}}$ whose population risk is strictly below that of $f_{\mathrm{old}}^\ast$. In the population-risk route, this produces a deterministic margin $\Delta_R>0$ and a selected jumpboard model $\tilde f_{pop}$.  This proof is shown in Appendix \ref{app:proof-thm1}. 

\textbf{Generalization error remains controllable.}
The second ingredient is a norm-based generalization bound for normalized residual networks. Under Assumption 2 and either 4 or 4$\epsilon$, Lemma~\ref{lem:gen-main} in Appendix \ref{app:proof-lem-gen-main} gives a uniform bound  for the loss class associated with $\mathcal H_{\mathrm{new}}$. In the main text, this bound appears only through the shorthand quantities $\epsilon_M$ and $\epsilon_K$. If one uses RMSNorm with an engineering constant $\epsilon>0$, the same role is played by the $\epsilon$-version of the lemma, which removes the pre-normalization non-degeneracy requirement from the main text.  The proof is shown in Appendix \ref{app:proof-lem-gen-main} and  \ref{sec:remove-N-by-eps}.


 

\textbf{Optimization converts existence into an actual final model.}
The third ingredient is an optimization comparison: the  model $f_{\mathrm{new}}$ is assumed to achieve training loss no worse than the jumpboard model, up to the gain term $\Delta_{\mathrm{ERM}}$. This is the point where the existential jumpboard model is connected to the actual final output. 

\textbf{Route A: through population risk.}
Theorem~\ref{thm:4A} follows the chain
 
\begin{small}
\begin{equation*}
R(f_{\mathrm{old}}^\ast)\xrightarrow{-\Delta_R}R(\tilde f_{pop})\xrightarrow{\epsilon_M}\mathcal L_{\mathrm{train}}(\tilde f_{pop})\xrightarrow{-\Delta_{\mathrm{ERM}}}\mathcal L_{\mathrm{train}}(f_{\mathrm{new}})\xrightarrow{\epsilon_M}R(f_{\mathrm{new}}), 
\end{equation*}
\end{small}

and then uses Hoeffding to connect the two ends to test risk. This route is tighter because its total generalization price is only $2\epsilon_M$, but it relies on a non-vanishing deterministic population-risk margin $\Delta_R>0$. The proof is shown in Appendix \ref{app:proof-thm4A}.

\textbf{Route B: directly on train/test risks.}
Theorem~\ref{thm:4B} bypasses population risk and Hoeffding altogether. Instead, it applies the normalized-network generalization bound once on the training sample and once on the independent test sample, yielding a uniform control of $|\mathcal L_{\mathrm{train}}(f)-\mathcal L_{\mathrm{test}}(f)|$ through $\epsilon_M+\epsilon_K$. The proof is shown in Appendix \ref{app:proof-thm4B}, 
and this gives the chain
\begin{small}
\begin{equation*}
  \mathcal L_{\mathrm{test}}(f_{\mathrm{old}}^\ast)\xrightarrow{-\Delta_R^{\mathrm{test}}}\mathcal L_{\mathrm{test}}(\tilde f_S)\xrightarrow{\epsilon_M+\epsilon_K}\mathcal L_{\mathrm{train}}(\tilde f_S)\xrightarrow{-\Delta_{\mathrm{ERM}}}\mathcal L_{\mathrm{train}}(f_{\mathrm{new}})\xrightarrow{\epsilon_M+\epsilon_K}\mathcal L_{\mathrm{test}}(f_{\mathrm{new}}).  
\end{equation*}
\end{small}

The two routes are logically independent and close the same scaling conclusion in complementary regimes. Route A is sharper in the ordinary regime where $\Delta_R$ is not too small and $M\gg K$, while Route B is more robust near the degenerate regime where the deterministic population margin becomes too small for Hoeffding-based transfer to remain useful. Together they form the double strictification underlying the main theorem section. 

\section{Roles of Depth, Width, and Data}
\label{sec:roles-depth-width-data}

\textbf{Depth.}
The role of depth in our framework is existential rather than merely parametric. Adding a residual block matters because it may create a nonzero first-order descent direction near zero initialization, which in turn yields a jumpboard model in the expanded hypothesis class. In this sense, depth is the source of \emph{new improving directions}: it enlarges the function class in a way that can be converted into a smaller population risk and, under the main theorems, into improved test performance.

At the same time, this mechanism need not persist with uniform strength. As the model becomes deeper and better adapted, the available first-order descent signal may shrink, and Condition~1 may eventually fail at the relevant insertion layers. This motivates a functional notion of the \emph{deepest model}: a model is deepest if inserting additional residual blocks can no longer improve performance through the first-order mechanism. Equivalently, the model has already reached representation optimality at the relevant layers, so that infinitesimal perturbations of the intermediate representations no longer decrease the loss to first order. In this regime, further depth may still enlarge the hypothesis class formally, but it no longer creates a useful first-order descent direction from zero initialization. This viewpoint suggests a natural progression of depth scaling. In the clearly improvable regime, increasing depth can still produce a strict gain because a stable first-order descent direction exists. Closer to the deepest-model regime, the remaining signal becomes weak: the model is not yet fully saturated, but the first-order gain becomes increasingly difficult to extract and certify. At the limiting deepest-model regime, the first-order mechanism is exhausted, and further depth no longer yields a meaningful strict improvement. Equivalently, the deepest-model regime is reached when Condition~1 fails at every admissible insertion layer: the activation-gradient signal is zero, or no zero-output residual block can generate a tangent direction non-orthogonal to the remaining activation gradient.

 \textbf{Data.}
Data determines whether the gain created by model expansion survives statistical uncertainty. 
Under the norm-controlled framework of Lemma~\ref{lem:gen-main}, the
generalization term decays at the standard \(M^{-1/2}\) statistical rate, but
with constants that depend on the architecture through norm, depth, width, and
normalization factors. When these architecture-dependent factors are controlled
or treated as fixed, the bound reduces to the familiar shorthand
\(\widetilde O(\sqrt{d/M})\), where \(d\) is the effective parameter dimension.
When depth, width, or normalization constants are themselves scaling variables,
the corresponding architecture-dependent factors must be kept explicit; see
Appendix~\ref{app:data-requirement} for the resulting data requirement. In other words, model expansion is beneficial only when the growth of data is fast enough to offset the statistical price of the enlarged hypothesis class.

 This viewpoint is structurally consistent with empirical data--parameter
coupling heuristics in large language models. For example, compute-optimal
scaling heuristics such as Chinchilla suggest an empirical operating point of
roughly \(20\) training tokens per parameter, i.e., $\frac{M}{d}\approx 20$,
when \(M\) is interpreted as the number of training tokens. This ratio should not be read as a theorem-level constant
predicted by our bound. Rather, our result supports the qualitative structure
behind such heuristics: data must grow with model complexity in order to offset
the statistical cost of expansion. The precise proportionality factor depends
on architecture-dependent norm amplification, the training objective, compute
budget, and optimization setting. Our bound supports only the qualitative need for data–complexity coupling; it does not predict the empirical ratio.

\textbf{Depth--width coupling.}
Near the deepest-model regime, the remaining first-order improvement signal becomes small. In Route A, this weakens the Hoeffding transfer because the deterministic margin shrinks; in Route B, the main difficulty becomes whether the favorable direction is still visible at finite sample size. This is quantified by Theorem~\ref{thm 2} in the Appendix \ref{app:proof-thm2}. As $\|\bar\mu\|\to 0$, train/test directional alignment becomes harder to observe reliably, even if a beneficial direction still exists. Under the uniformly active regime, $\|\bar\mu\|^2=\Theta(N)$, so the bound improves to $\mathbb P(\mu^\top g\le 0)=O\!\left(\frac{1}{KN}+\frac{1}{MN}\right)$. Therefore, as depth approaches the deepest-model regime, width acts as a compensating factor by improving finite-sample observability of weak first-order gains. This is why our theory naturally leads to a depth--width coupling: deeper models are not guaranteed to improve at fixed width, but can continue to improve in a jointly scaled regime where width and data grow sufficiently fast.

\section{Experiments}
\label{sec:experiments}

This section empirically examines the main theoretical ingredients of our framework. Our purpose is not to claim state-of-the-art performance, but to test whether the empirical phenomena predicted by the theory are visible in a controlled setting: near-diagonal activation-gradient covariance,  test-loss performance with scale, decay of first-order signals near the deepest-model regime.

For architectures, we use residual networks of depth $6n+2$, and enlarge width by multiplying the number of channels in each residual stage by a width factor. Unless otherwise stated, the normalization layer is chosen to match the theoretical setting in the main text.  For datasets, the main scaling and gradient-decay experiments are conducted on CIFAR-10 and CIFAR-100. We additionally include ImageNet results in the covariance experiment to test whether the near-diagonal activation-gradient covariance pattern also appears in a larger-scale setting. The training set is used for optimization and for estimating training-side activation gradients, while the test set is used both for standard evaluation and for estimating test-side activation gradients. For training details, all models are trained under the same optimization setup, including optimizer type, learning-rate schedule, batch size, data preprocessing, and augmentation pipeline, which are shown in Appendix \ref{app:experiments setting}.  Appendix \ref{app:transformer} discusses external Transformer evidence from prior work..

\subsection{Empirical support for Assumption 3}
\label{sec:exp-covariance}
\begin{figure}[!htbp]
\centering
\begin{minipage}{0.24\textwidth}
\centering
\includegraphics[width=\textwidth]{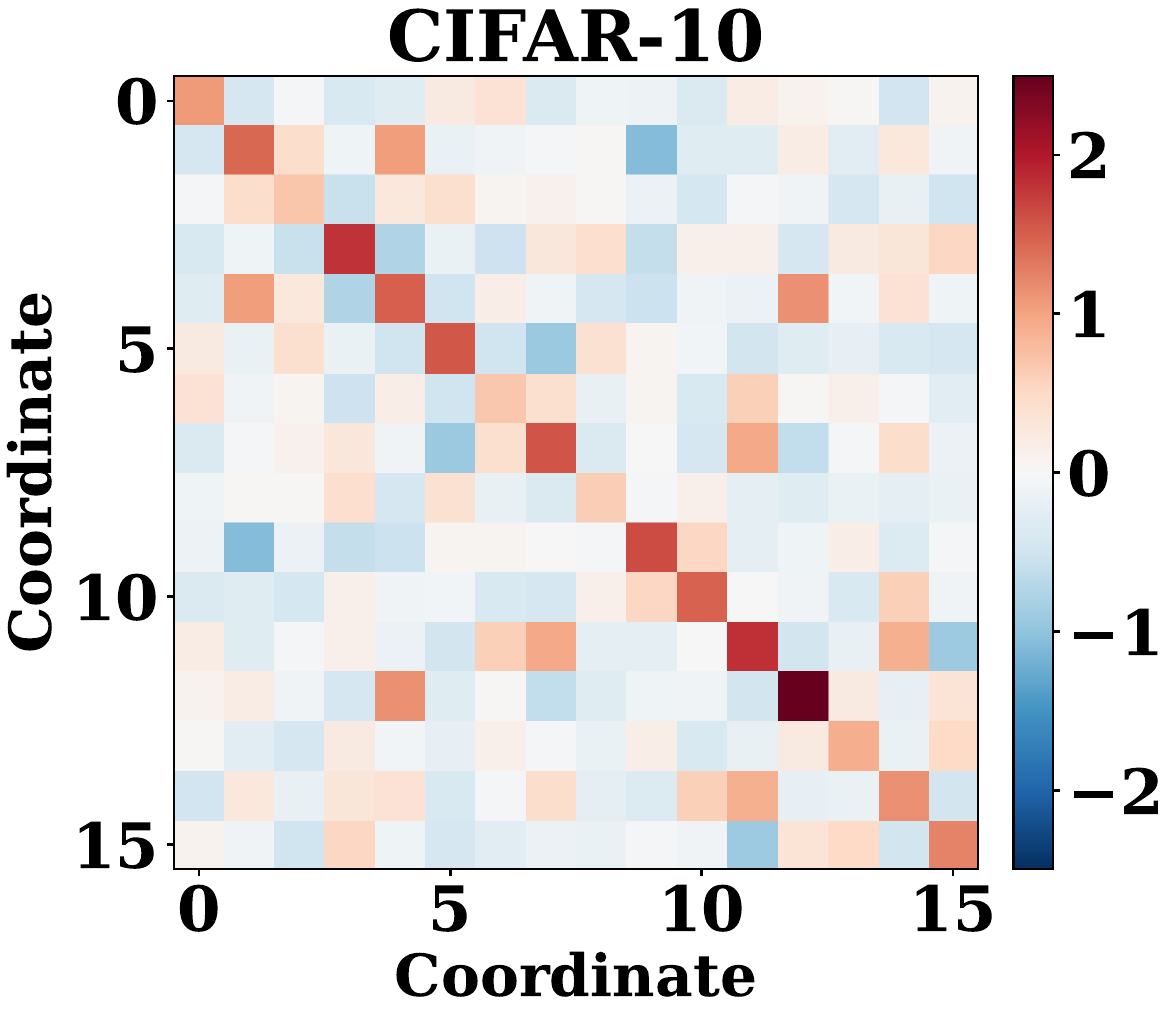}
(a)
\end{minipage}
\begin{minipage}{0.24\textwidth}
\centering
\includegraphics[width=\textwidth]{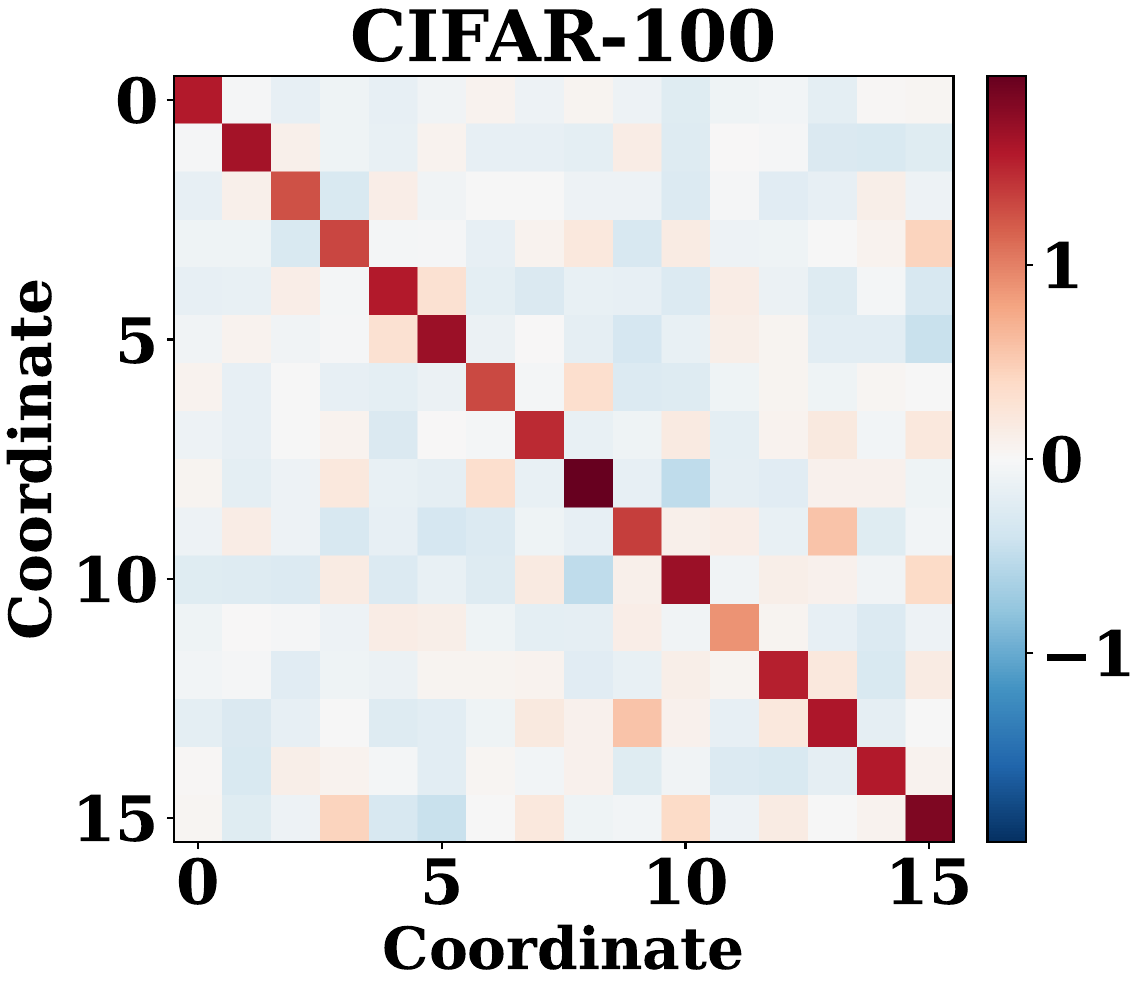}
(b)
\end{minipage}
\begin{minipage}{0.24\textwidth}
\centering
\includegraphics[width=\textwidth]{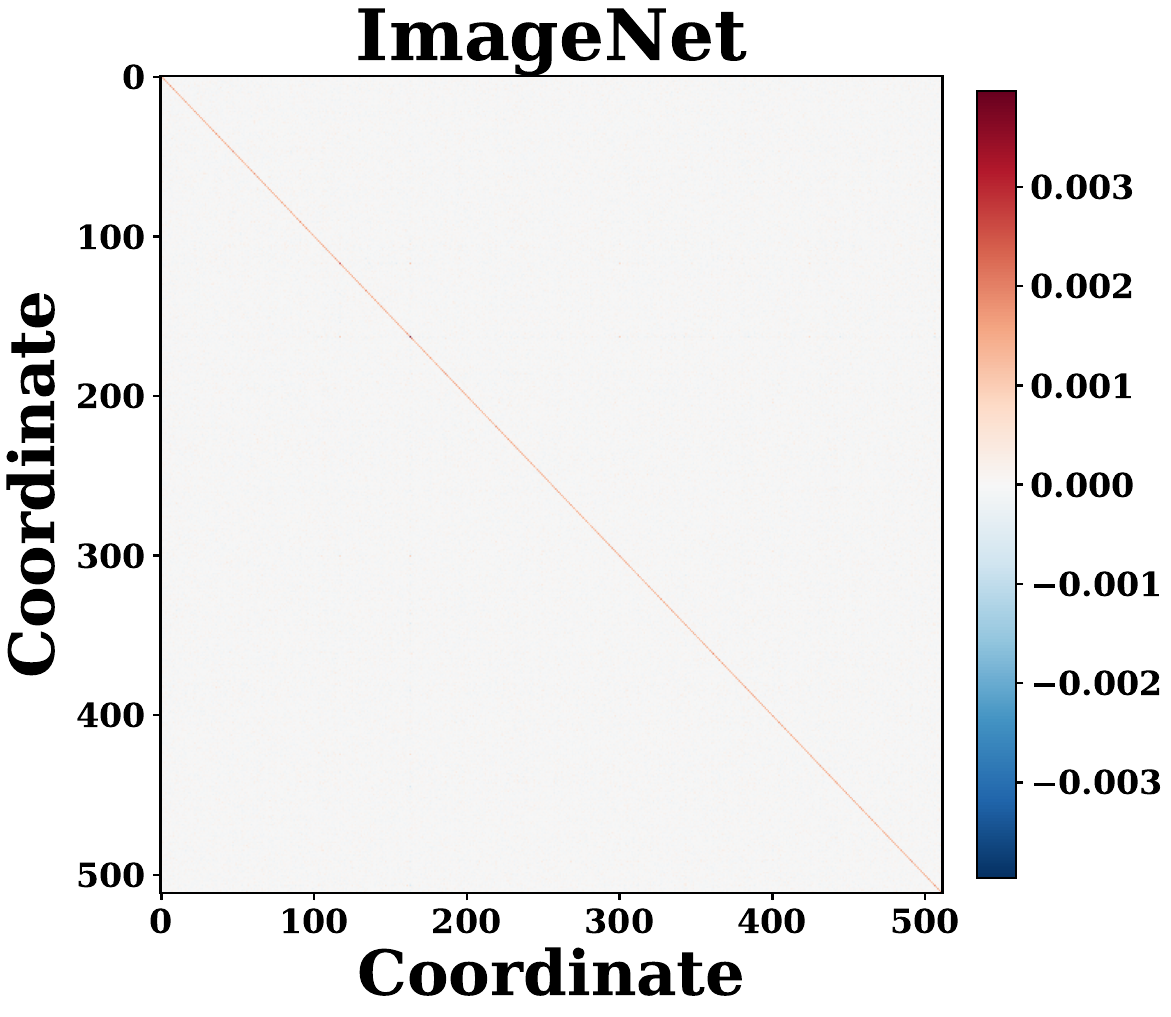}
(c)
\end{minipage}
\begin{minipage}{0.24\textwidth}
\centering
\includegraphics[width=\textwidth]{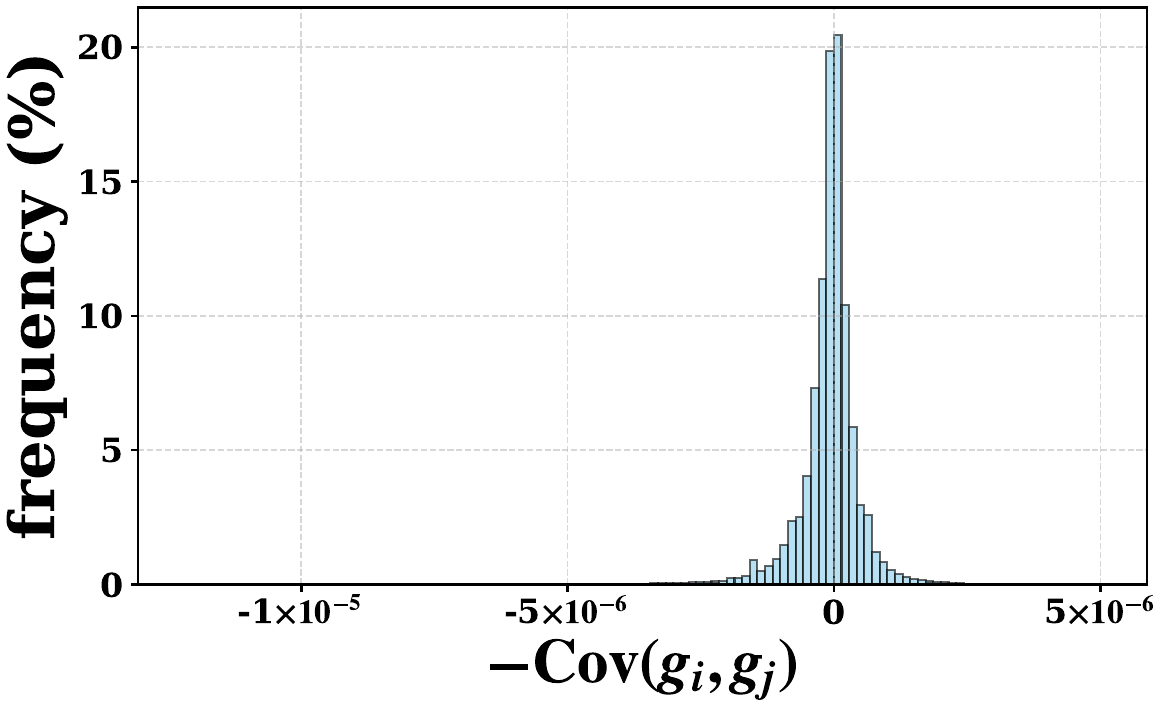}
(d)
\end{minipage}
\caption{
Empirical support for Assumption 3.
\textbf{Figure (a),(b),(c)}: A representative covariance heatmap of activation gradients. The activation before the fully connected layer. ((a) : CIFAR-10, ResNet-8, base channel 4; (b) : CIFAR-100, ResNet-8,  base channel 4; (c) :ImageNet: ResNet-18, channel 64).
\textbf{Figure (d):} Distribution of sampled off-diagonal covariance entries for a higher-dimensional case (ResNet-56, channel 16, activation before the first ResNet block, sampling $2*10^7$ pairs).
All figures suggest that activation-gradient covariance is close to diagonal in practice.
}
\label{fig:Assumption 3}
\vspace{-0.1cm}
\end{figure}

This experiment provides empirical evidence for the covariance-control behavior required by Assumption 3. Rather than directly verifying the full spectral condition, we test a stronger and more interpretable empirical property: whether different coordinates of the activation gradient are nearly uncorrelated across samples. If so, the covariance matrix is close to diagonal, and its spectral behavior is governed mainly by coordinate-wise variances rather than by large cross-coordinate correlations.

To test this, we fix a target layer \(l^\ast\), compute sample-level activation gradients for a trained model, and examine their empirical covariance structure. Since storing the full covariance matrix is expensive at larger widths, we use two complementary visualizations: Figure~\ref{fig:Assumption 3} (a) (b) and (c) show a representative covariance heatmap for a low-dimensional case, while Figure~\ref{fig:Assumption 3} (d) shows the distribution of sampled off-diagonal covariance entries over a much larger set of coordinate pairs. Besides CIFAR datasets, we also show ImageNet results in this part.

The results in Figure~\ref{fig:Assumption 3} show a clear and consistent pattern: the off-diagonal covariance entries are sharply concentrated near zero, with only a very small tail away from the origin. Empirically, this indicates that the off-diagonal part of the activation-gradient covariance matrix is extremely small, so the covariance structure is close to diagonal in practice. The heatmap supports this interpretation visually, while the histogram shows that near-zero off-diagonal covariance is not an isolated artifact of a few coordinates but a dominant distributional phenomenon.

From the perspective of our theory, this observation provides direct empirical support for Assumption 3. If the covariance matrix is approximately diagonal, then its spectral behavior is controlled primarily by coordinate-wise variances, making the spectral-control condition used in Theorem~\ref{thm 2} in Appendix \ref{app:proof-thm2} substantially more plausible in trained normalized residual networks. We emphasize that the experiment is not intended to prove that \(\Sigma\) is exactly diagonal in a strict mathematical sense. Rather, it shows that the off-diagonal correlations are sufficiently small in practice for the ``approximately diagonal, spectrally controlled'' approximation underlying our width-alignment analysis to be credible.

\subsection{Decay of first-order activation-gradient signals with depth}
\label{sec:exp-gradient-decay}
\begin{figure}[h]
\centering
\begin{minipage}{0.24\textwidth}
\centering
\includegraphics[width=\textwidth]{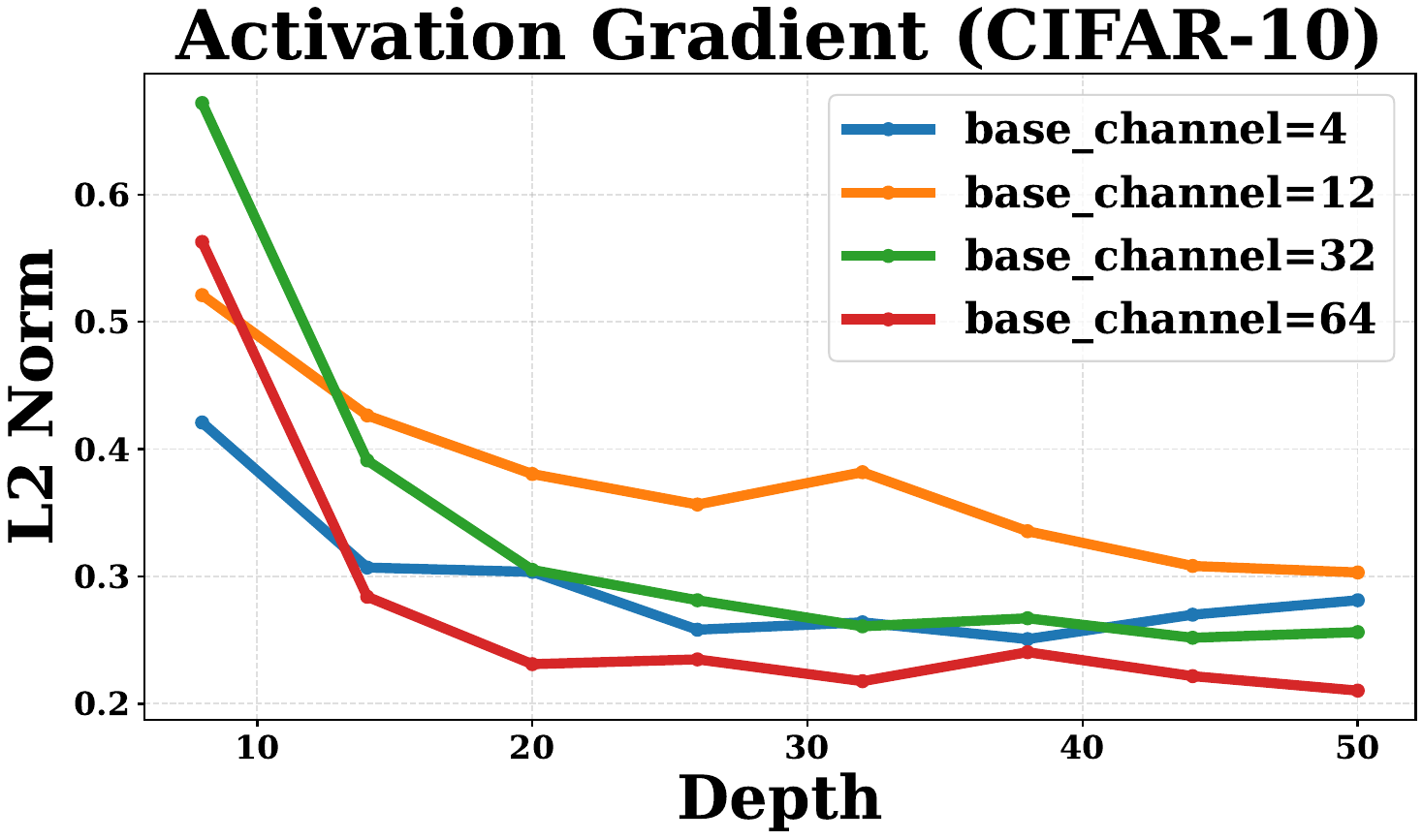}
(a)
\end{minipage}
\begin{minipage}{0.24\textwidth}
\centering
\includegraphics[width=\textwidth]{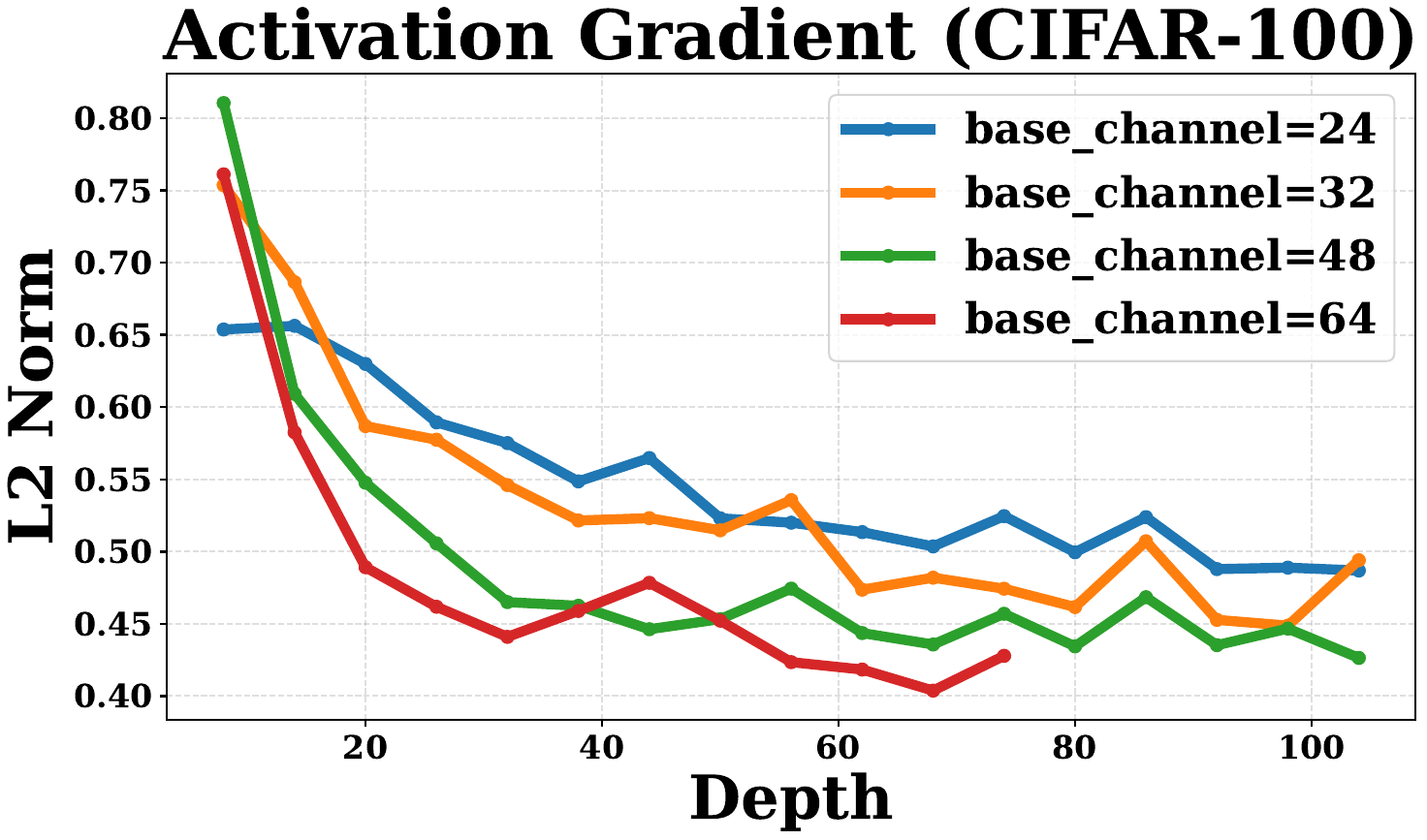}
(b)
\end{minipage}
\begin{minipage}{0.24\textwidth}
\centering
\includegraphics[width=\textwidth]{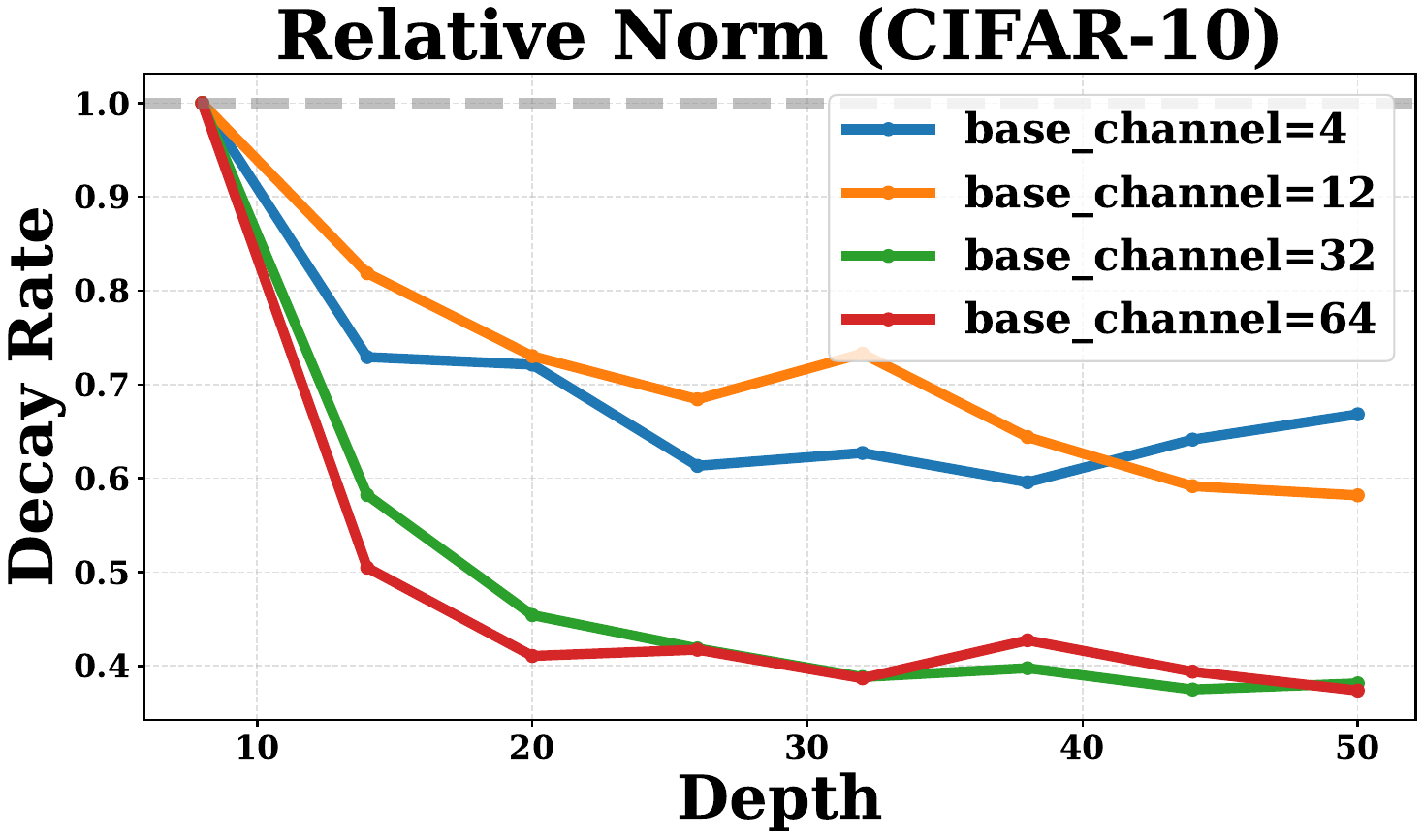}
(c)
\end{minipage}
\begin{minipage}{0.24\textwidth}
\centering
\includegraphics[width=\textwidth]{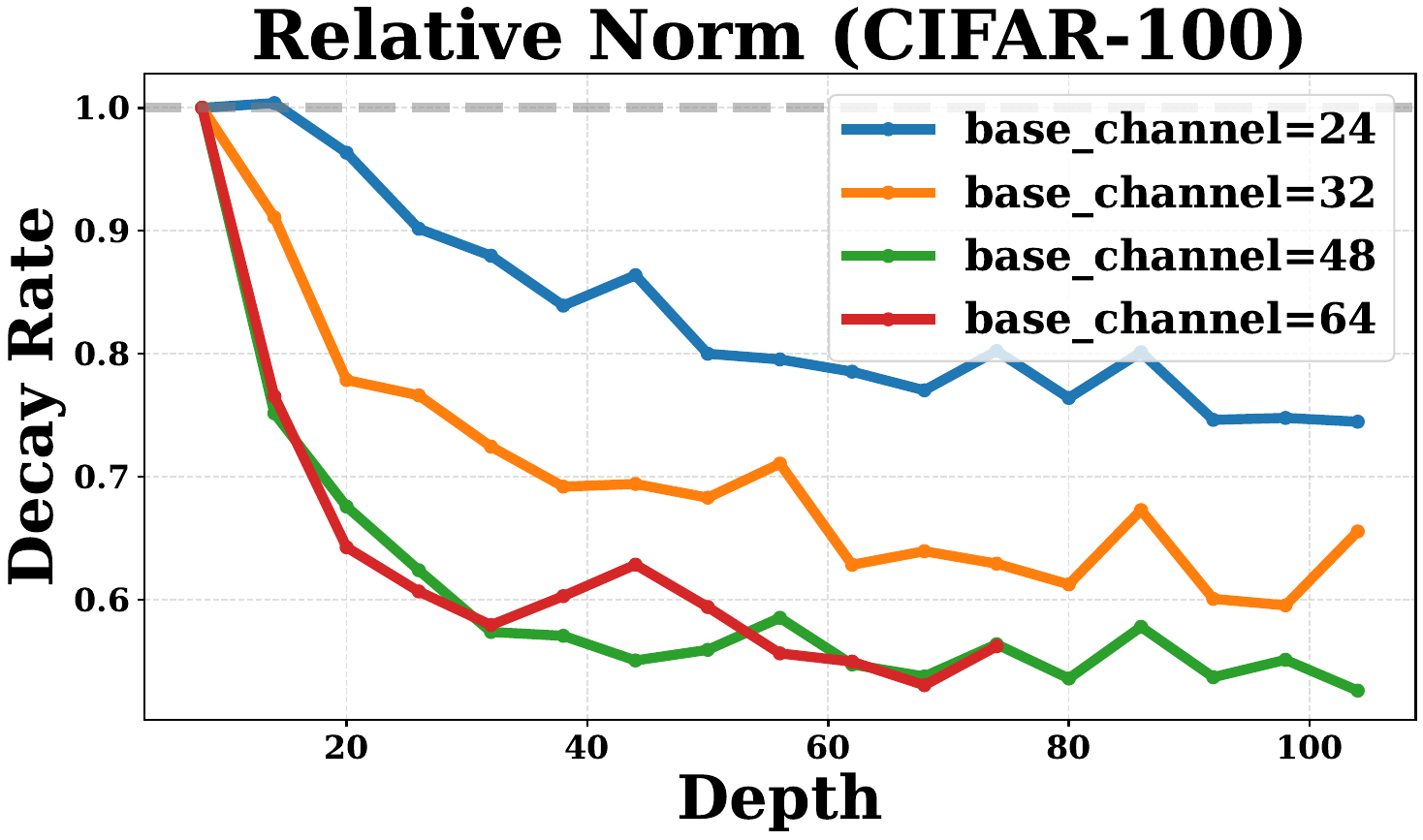}
(d)
\end{minipage}
\caption{
Decay of first-order average activation-gradient signals with depth.
\textbf{(a),(b)}: absolute \(\ell_2\)-norm of the activation gradient immediately before the final fully connected layer.
\textbf{(c),(d)}: the same quantity after width-wise normalization by the shallowest model.
The absolute plot shows that the remaining first-order signal decreases as depth increases, while the normalized plot shows that this decay pattern is structurally consistent across widths.
}
\label{fig:grad_decay}
\vspace{-0.1cm}
\end{figure}

This experiment is designed to probe the weakening of the first-order improvement signal as the model becomes deeper. According to the discussion in the main text, approaching the deepest-model regime should be accompanied by a shrinkage of the activation-gradient signal, which makes additional depth harder to exploit through the first-order mechanism.

To test this phenomenon, we consider ResNet models with different depths and varying the base channel width. For each trained model, we select the activation immediately before the final fully connected layer and compute the corresponding activation gradient with respect to the loss. We use the \(\ell_2\)-norm of the average activation gradient as an empirical proxy for the remaining first-order signal. In addition to the absolute norm, we also report a width-wise normalized version, obtained by dividing each curve by its value at the shallowest depth, in order to isolate the relative decay pattern across widths. 

Figure~\ref{fig:grad_decay} shows two complementary views of this phenomenon. The Figure~\ref{fig:grad_decay}(a), (b) report the absolute activation-gradient norm, while the   Figure~\ref{fig:grad_decay}(c), (d) shows the corresponding relative decay after width-wise normalization. In the absolute plot, the overall trend is clear: as depth increases, the activation-gradient norm decreases substantially across all tested widths, indicating that deeper models carry a weaker remaining first-order improvement signal. This directly supports the deepest-model interpretation in the main text.

The normalized plot further shows that this weakening is not merely an artifact of different initial signal magnitudes across widths. After normalization, the curves still exhibit a clear overall downward trend, which indicates that the decay pattern itself is induced primarily by increasing depth. In particular, much of the reduction occurs in the early stage of deepening, after which the curves flatten in a lower-signal regime. This suggests that the first-order signal is depleted progressively rather than disappearing abruptly at a single critical depth.

Another important observation is that width mainly affects the absolute signal level and the stability of the decay trajectory, rather than the existence of the decay itself. Across all tested widths, the remaining first-order signal weakens as depth grows. However, wider models tend to maintain a larger absolute signal, while narrower models show more visible late-stage fluctuation. This is consistent with our theoretical picture: width does not create the first-order signal, but it can make weak signals more observable and more stable at finite sample size.

From the viewpoint of the theory, Figure~\ref{fig:grad_decay} supports the claim that near the deepest-model regime, the population-level first-order signal becomes small, so the gain available from further depth expansion weakens. This helps explain why strict improvement becomes harder to certify in very deep models, and why width becomes increasingly important as a compensating factor in the weak-signal regime.

 \subsection{Test loss under joint depth and width scaling}
\label{sec:exp-scaling-performance}
\begin{figure}[!htbp]
\centering
\begin{minipage}{0.24\textwidth}
\centering
\includegraphics[width=\textwidth]{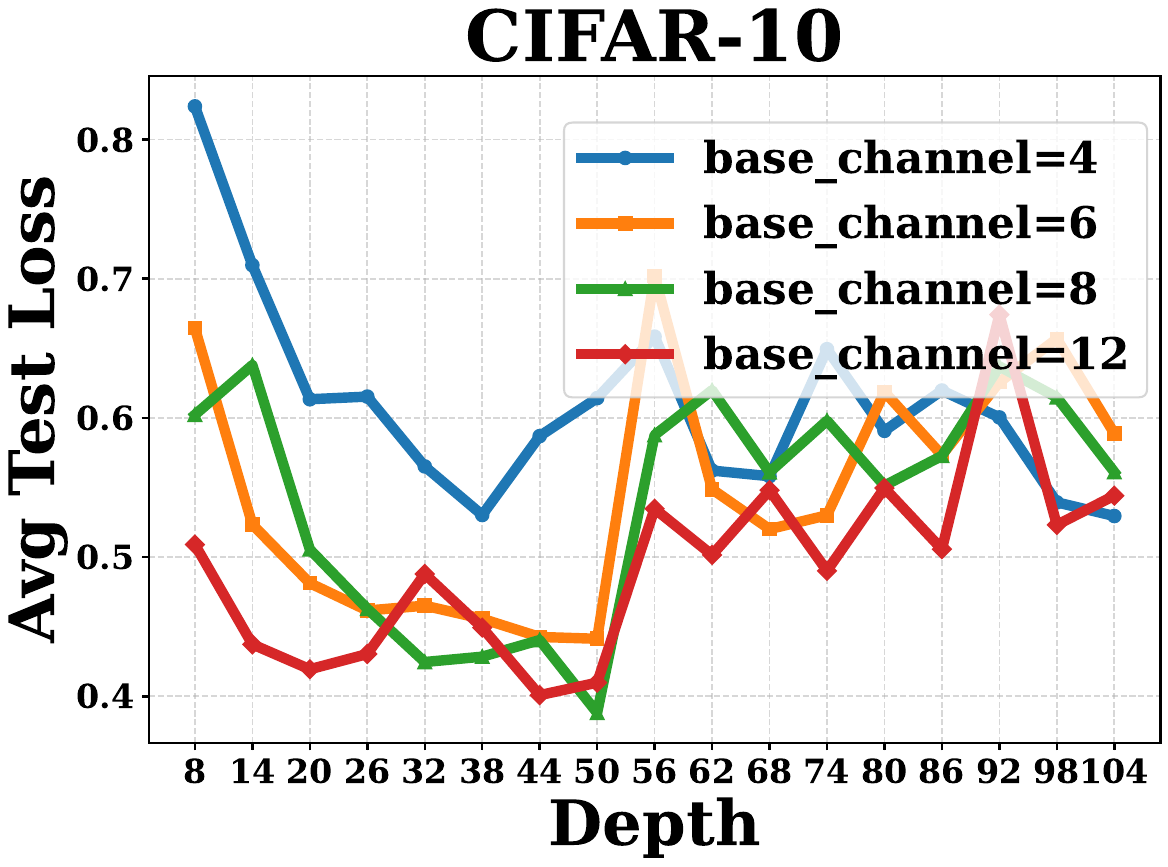}
(a)
\end{minipage}
\begin{minipage}{0.24\textwidth}
\centering
\includegraphics[width=\textwidth]{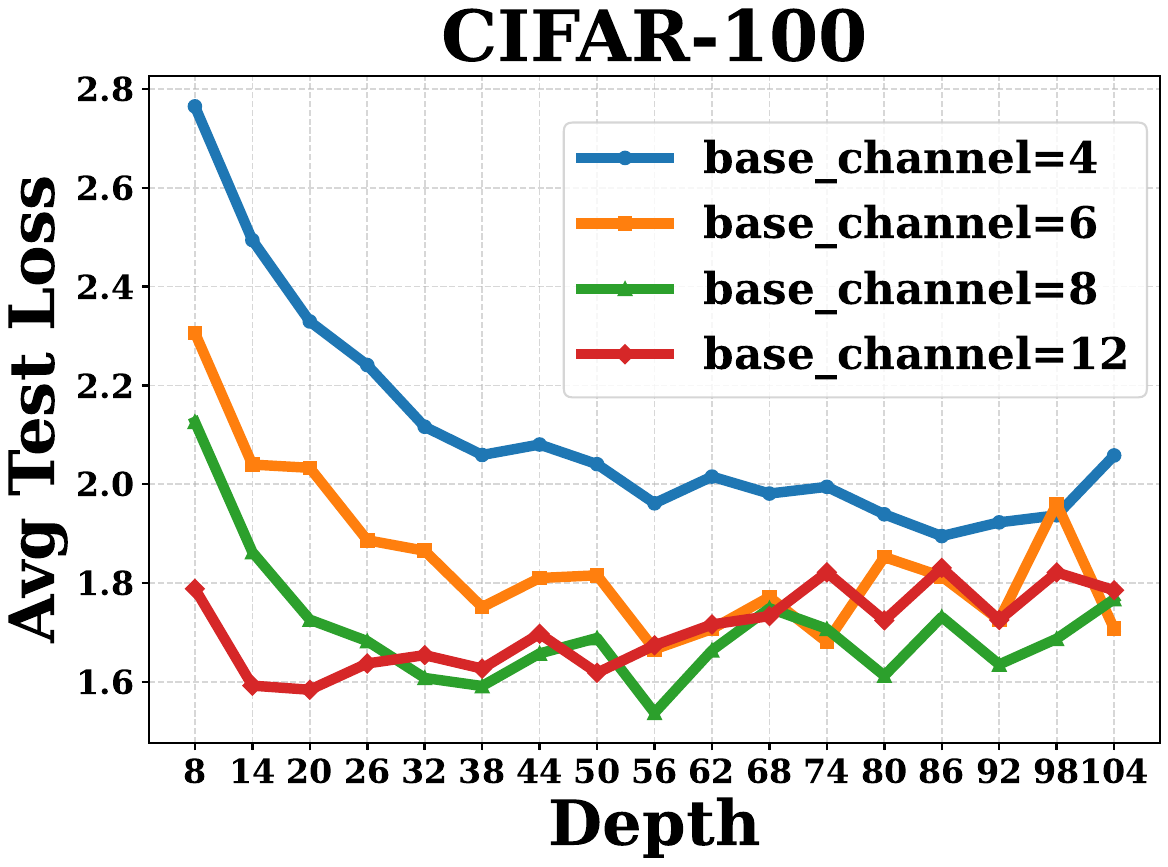}
(b)
\end{minipage}
\begin{minipage}{0.24\textwidth}
\centering
\includegraphics[width=\textwidth]{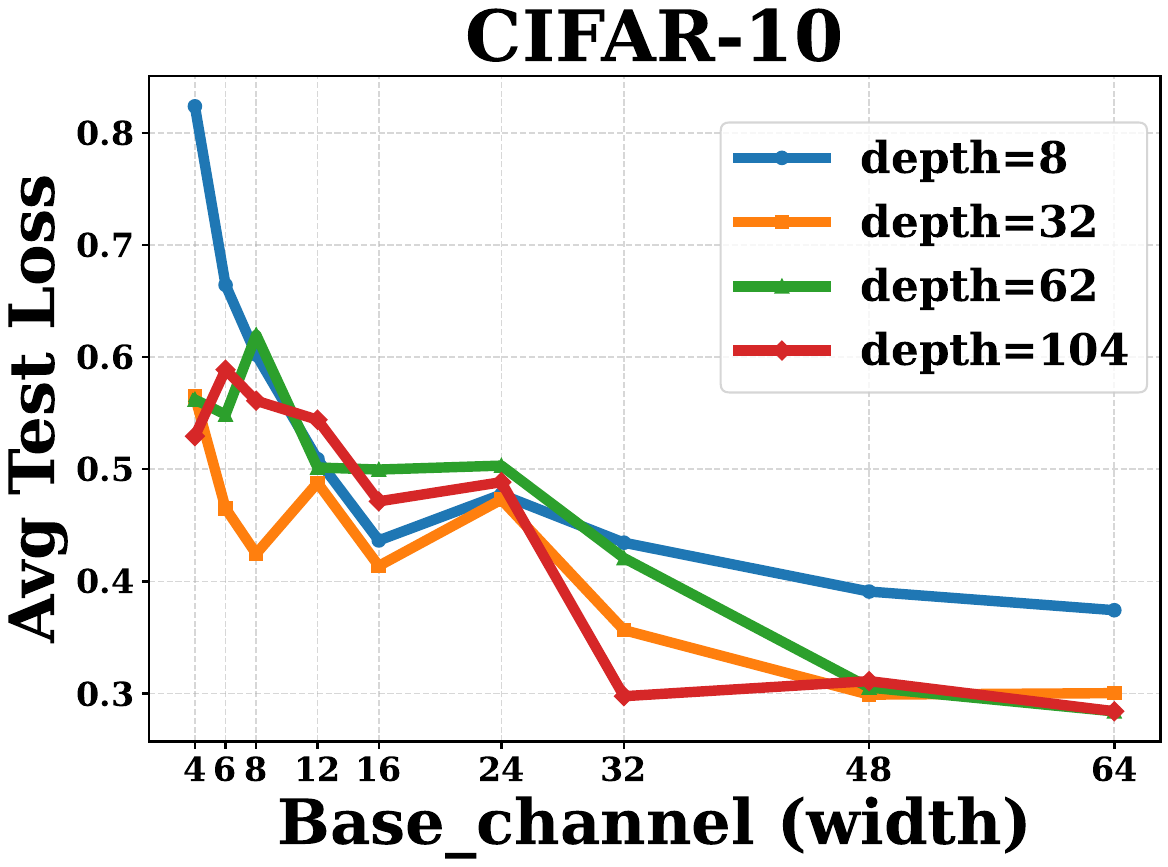}
(c)
\end{minipage}
\begin{minipage}{0.24\textwidth}
\centering
\includegraphics[width=\textwidth]{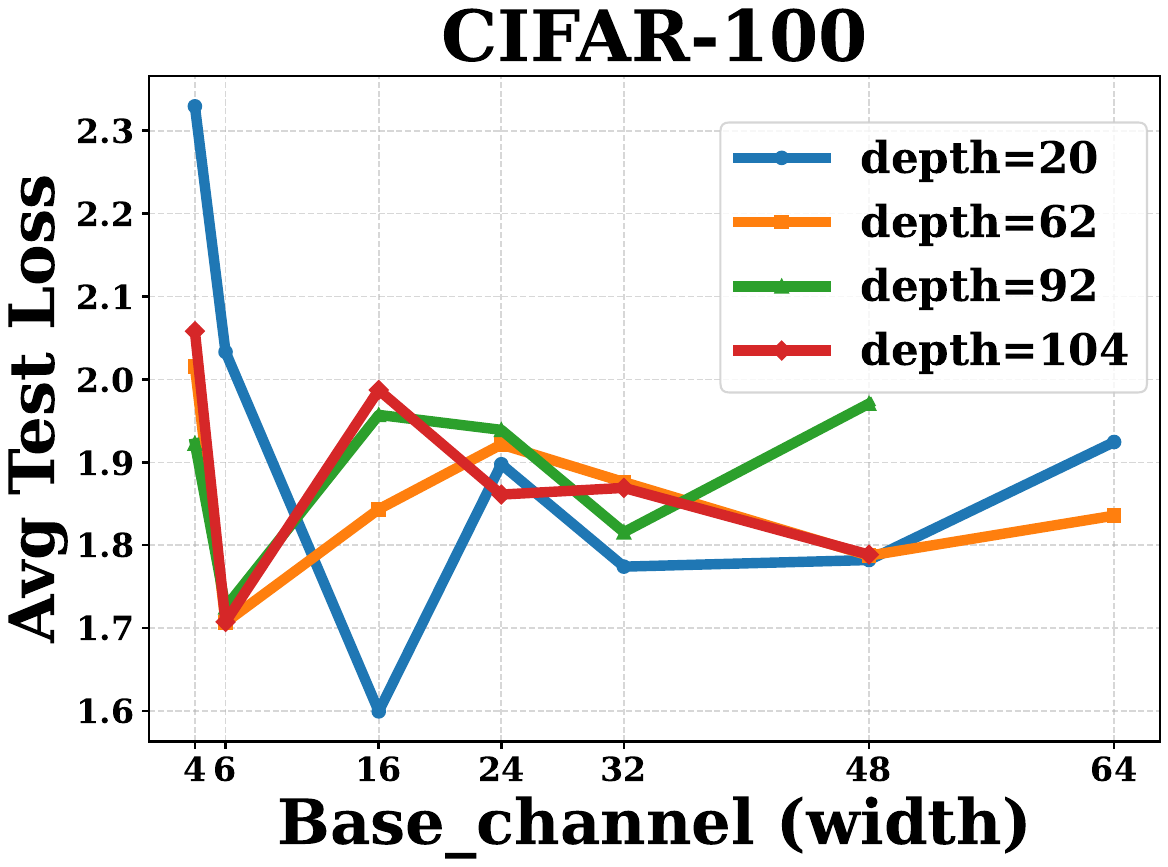}
(d)

\end{minipage}

\caption{Test loss under joint depth and width scaling. (a),(b) show test loss as a function of depth at fixed widths; (c), (d) show test loss as a function of width at fixed depths.}
\label{fig:loss scaling}
\end{figure}
This experiment is designed to test whether the empirical test-loss behavior is consistent with the qualitative predictions of the main theorems, namely that depth expansion can improve test performance in a jointly scaled regime, while width stabilizes the finite-sample observability of such gains. Figure~\ref{fig:loss scaling} provides two complementary views: test loss as a function of depth at several widths, and test loss as a function of width at several depths. Curves report the mean over 5 random seeds, with shaded regions showing one standard deviation.

The first observation is that the overall trend is favorable: increasing depth and width often reduces test loss, which is consistent with the theorem-level prediction that model expansion can improve test performance when the available gain dominates the statistical cost. At the same time, for CIFAR-10, this improvement is clearly non-monotonic. After an initial improvement, the test loss rebounds noticeably at larger depths. This suggests that increasing depth alone is not sufficient; when width does not grow adequately with depth, the test-side benefit of further deepening becomes unstable and may partially disappear. For CIFAR-100, we do not observe the same non-monotonic rebound within the tested depth range, possibly because the maximum depth explored here is not yet large enough.

The second observation is that width acts  as a stabilizing factor rather than as an unlimited source of gain. In the width-scaling plot, increasing width produces a substantial reduction in test loss in the narrow-to-moderate regime, especially for deeper models. However, once width becomes sufficiently large, the curves flatten and further widening yields only marginal improvement. This saturation behavior is  consistent with our theoretical picture.

A third conclusion is that the best depth depends on width. At small width, deeper models can become less reliable and may even show rebound in test loss, whereas at larger width, depth scaling becomes substantially more stable. This supports the joint-scaling interpretation developed in the main text: depth creates new improving directions, but width is needed to make those gains reliably visible on finite data. Therefore, Figure~\ref{fig:loss scaling} should be read not as evidence that ``deeper is always better at fixed width,'' but rather as empirical support for a depth--width coupling, in which the benefit of additional depth depends critically on whether width grows enough to support it.

\section{Discussion and Conclusion}

The main theorems first reveal a qualitative scaling mechanism. On this basis, if one further introduces model-dependent assumptions, this qualitative mechanism is possible to be converted into the power-law form commonly observed in empirical scaling laws. Specifically, in  Appendix \ref{app:power-law}
 we present one possible assumption-based route: Assumption G characterizes the coupling between the remaining excess risk and the population-level first-order improvement signal, while Assumption G' converts this first-order signal into the risk decrease obtained from each depth expansion step. Together, these assumptions induce a recursion on the remaining excess risk, which yields a power-law scaling relation in the asymptotic regime.

We develop a theorem-driven account of when depth expansion in normalized residual networks improves test risk. Our framework links three ingredients: a jumpboard model with lower population risk, optimization gain, and generalization transfer. This yields two complementary qualitative scaling-mechanism  results: a tighter population-risk route and a more robust train/test route. The analysis also clarifies that depth creates improving directions, width preserves their finite-sample observability, and data pays for the statistical cost of model expansion.
\bibliographystyle{plainnat}     
\bibliography{mybibtext}     

\begin{thebibliography}{32}
\providecommand{\natexlab}[1]{#1}
\providecommand{\url}[1]{\texttt{#1}}
\expandafter\ifx\csname urlstyle\endcsname\relax
  \providecommand{\doi}[1]{doi: #1}\else
  \providecommand{\doi}{doi: \begingroup \urlstyle{rm}\Url}\fi

\bibitem[Arora et~al.(2018{\natexlab{a}})Arora, Cohen, and Hazan]{arora2018optimization}
Sanjeev Arora, Nadav Cohen, and Elad Hazan.
\newblock On the optimization of deep networks: Implicit acceleration by overparameterization.
\newblock In \emph{International conference on machine learning}, pages 244--253. PMLR, 2018{\natexlab{a}}.

\bibitem[Arora et~al.(2018{\natexlab{b}})Arora, Ge, Neyshabur, and Zhang]{arora2018compression}
Sanjeev Arora, Rong Ge, Behnam Neyshabur, and Yi~Zhang.
\newblock Stronger generalization bounds for deep nets via a compression approach.
\newblock In \emph{International conference on machine learning}, pages 254--263. PMLR, 2018{\natexlab{b}}.

\bibitem[Bachlechner et~al.(2021)Bachlechner, Majumder, Mao, Cottrell, and McAuley]{bachlechner2021rezero}
Thomas Bachlechner, Bodhisattwa~Prasad Majumder, Henry Mao, Gary Cottrell, and Julian McAuley.
\newblock Rezero is all you need: Fast convergence at large depth.
\newblock In \emph{Uncertainty in artificial intelligence}, pages 1352--1361. PMLR, 2021.

\bibitem[Bahri et~al.(2024)Bahri, Dyer, Kaplan, Lee, and Sharma]{bahri2024explaining}
Yasaman Bahri, Ethan Dyer, Jared Kaplan, Jaehoon Lee, and Utkarsh Sharma.
\newblock Explaining neural scaling laws.
\newblock \emph{Proceedings of the National Academy of Sciences}, 121\penalty0 (27):\penalty0 e2311878121, 2024.

\bibitem[Balduzzi et~al.(2017)Balduzzi, Frean, Leary, Lewis, Ma, and McWilliams]{balduzzi2017shattered}
David Balduzzi, Marcus Frean, Lennox Leary, JP~Lewis, Kurt Wan-Duo Ma, and Brian McWilliams.
\newblock The shattered gradients problem: If resnets are the answer, then what is the question?
\newblock In \emph{International conference on machine learning}, pages 342--350. PMLR, 2017.

\bibitem[Bartlett et~al.(2017)Bartlett, Foster, and Telgarsky]{bartlett2017spectral}
Peter~L Bartlett, Dylan~J Foster, and Matus~J Telgarsky.
\newblock Spectrally-normalized margin bounds for neural networks.
\newblock \emph{Advances in neural information processing systems}, 30, 2017.

\bibitem[Brown et~al.(2020)Brown, Mann, Ryder, Subbiah, Kaplan, Dhariwal, Neelakantan, Shyam, Sastry, Askell, et~al.]{brown2020language}
Tom Brown, Benjamin Mann, Nick Ryder, Melanie Subbiah, Jared~D Kaplan, Prafulla Dhariwal, Arvind Neelakantan, Pranav Shyam, Girish Sastry, Amanda Askell, et~al.
\newblock Language models are few-shot learners.
\newblock \emph{Advances in neural information processing systems}, 33:\penalty0 1877--1901, 2020.

\bibitem[Chang et~al.(2018)Chang, Meng, Haber, Ruthotto, Begert, and Holtham]{chang2018reversible}
Bo~Chang, Lili Meng, Eldad Haber, Lars Ruthotto, David Begert, and Elliot Holtham.
\newblock Reversible architectures for arbitrarily deep residual neural networks.
\newblock In \emph{Proceedings of the AAAI conference on artificial intelligence}, volume~32, 2018.

\bibitem[De and Smith(2020)]{de2020batchnorm}
Soham De and Sam Smith.
\newblock Batch normalization biases residual blocks towards the identity function in deep networks.
\newblock \emph{Advances in Neural Information Processing Systems}, 33:\penalty0 19964--19975, 2020.

\bibitem[Eldan and Shamir(2016)]{eldan2016power}
Ronen Eldan and Ohad Shamir.
\newblock The power of depth for feedforward neural networks.
\newblock In \emph{Conference on learning theory}, pages 907--940. PMLR, 2016.

\bibitem[Golowich et~al.(2018)Golowich, Rakhlin, and Shamir]{golowich2018sizeindependent}
Noah Golowich, Alexander Rakhlin, and Ohad Shamir.
\newblock Size-independent sample complexity of neural networks.
\newblock In \emph{Conference on learning theory}, pages 297--299. PMLR, 2018.

\bibitem[Haber and Ruthotto(2018)]{haber2018stable}
Eldad Haber and Lars Ruthotto.
\newblock Stable architectures for deep neural networks.
\newblock \emph{Inverse problems}, 34\penalty0 (1):\penalty0 014004, 2018.

\bibitem[He et~al.(2016{\natexlab{a}})He, Zhang, Ren, and Sun]{he2016deep}
Kaiming He, Xiangyu Zhang, Shaoqing Ren, and Jian Sun.
\newblock Deep residual learning for image recognition.
\newblock In \emph{Proceedings of the IEEE conference on computer vision and pattern recognition}, pages 770--778, 2016{\natexlab{a}}.

\bibitem[He et~al.(2016{\natexlab{b}})He, Zhang, Ren, and Sun]{he2016identity}
Kaiming He, Xiangyu Zhang, Shaoqing Ren, and Jian Sun.
\newblock Identity mappings in deep residual networks.
\newblock In \emph{European conference on computer vision}, pages 630--645. Springer, 2016{\natexlab{b}}.

\bibitem[Hernandez et~al.(2021)Hernandez, Kaplan, Henighan, and McCandlish]{hernandez2021transfer}
Danny Hernandez, Jared Kaplan, Tom Henighan, and Sam McCandlish.
\newblock Scaling laws for transfer.
\newblock \emph{arXiv preprint arXiv:2102.01293}, 2021.

\bibitem[Hestness et~al.(2017)Hestness, Narang, Ardalani, Diamos, Jun, Kianinejad, Patwary, Yang, and Zhou]{hestness2017predictable}
Joel Hestness, Sharan Narang, Newsha Ardalani, Gregory Diamos, Heewoo Jun, Hassan Kianinejad, Md~Mostofa~Ali Patwary, Yang Yang, and Yanqi Zhou.
\newblock Deep learning scaling is predictable, empirically.
\newblock \emph{arXiv preprint arXiv:1712.00409}, 2017.

\bibitem[Hoffmann et~al.(2022)Hoffmann, Borgeaud, Mensch, Buchatskaya, Cai, Rutherford, Casas, Hendricks, Welbl, Clark, et~al.]{hoffmann2022training}
Jordan Hoffmann, Sebastian Borgeaud, Arthur Mensch, Elena Buchatskaya, Trevor Cai, Eliza Rutherford, DDL Casas, Lisa~Anne Hendricks, Johannes Welbl, Aidan Clark, et~al.
\newblock Training compute-optimal large language models.
\newblock \emph{arXiv preprint arXiv:2203.15556}, 10, 2022.

\bibitem[Huang et~al.(2026)Huang, Sun, and Wu]{huang2026generalization}
Jinshu Huang, Mingfei Sun, and Chunlin Wu.
\newblock On the generalization behavior of deep residual networks from a dynamical system perspective.
\newblock \emph{arXiv preprint arXiv:2602.20921}, 2026.

\bibitem[Kammonen et~al.(2020)Kammonen, Kiessling, Plech{\'a}{\v{c}}, Sandberg, Szepessy, and Tempone]{kammonen2020smaller}
Aku Kammonen, Jonas Kiessling, Petr Plech{\'a}{\v{c}}, Mattias Sandberg, Anders Szepessy, and Ra{\'u}l Tempone.
\newblock Smaller generalization error derived for a deep residual neural network compared to shallow networks.
\newblock \emph{arXiv preprint arXiv:2010.01887}, 2020.

\bibitem[Kammonen et~al.(2023)Kammonen, Kiessling, Plech{\'a}{\v{c}}, Sandberg, Szepessy, and Tempone]{kammonen2023smaller}
Aku Kammonen, Jonas Kiessling, Petr Plech{\'a}{\v{c}}, Mattias Sandberg, Anders Szepessy, and Raul Tempone.
\newblock Smaller generalization error derived for a deep residual neural network compared with shallow networks.
\newblock \emph{IMA Journal of Numerical Analysis}, 43\penalty0 (5):\penalty0 2585--2632, 2023.

\bibitem[Kaplan et~al.(2020)Kaplan, McCandlish, Henighan, Brown, Chess, Child, Gray, Radford, Wu, and Amodei]{kaplan2020scaling}
Jared Kaplan, Sam McCandlish, Tom Henighan, Tom~B Brown, Benjamin Chess, Rewon Child, Scott Gray, Alec Radford, Jeffrey Wu, and Dario Amodei.
\newblock Scaling laws for neural language models.
\newblock \emph{arXiv preprint arXiv:2001.08361}, 2020.

\bibitem[Levine et~al.(2020)Levine, Wies, Sharir, Bata, and Shashua]{levine2020depth}
Yoav Levine, Noam Wies, Or~Sharir, Hofit Bata, and Amnon Shashua.
\newblock The depth-to-width interplay in self-attention.
\newblock \emph{arXiv preprint arXiv:2006.12467}, 2020.

\bibitem[Lu et~al.(2020)Lu, Ma, Lu, Lu, and Ying]{lu2020meanfield}
Yiping Lu, Chao Ma, Yulong Lu, Jianfeng Lu, and Lexing Ying.
\newblock A mean field analysis of deep resnet and beyond: Towards provably optimization via overparameterization from depth.
\newblock In \emph{International Conference on Machine Learning}, pages 6426--6436. PMLR, 2020.

\bibitem[Ma et~al.(2020)Ma, Wang, et~al.]{e2020rademacher}
Chao Ma, Qingcan Wang, et~al.
\newblock Rademacher complexity and the generalization error of residual networks.
\newblock \emph{Communications in Mathematical Sciences}, 18\penalty0 (6):\penalty0 1755--1774, 2020.

\bibitem[Neyshabur et~al.(2017)Neyshabur, Bhojanapalli, and Srebro]{neyshabur2018pacbayes}
Behnam Neyshabur, Srinadh Bhojanapalli, and Nathan Srebro.
\newblock A pac-bayesian approach to spectrally-normalized margin bounds for neural networks.
\newblock \emph{arXiv preprint arXiv:1707.09564}, 2017.

\bibitem[Rosenfeld et~al.(2019)Rosenfeld, Rosenfeld, Belinkov, and Shavit]{rosenfeld2020constructive}
Jonathan~S Rosenfeld, Amir Rosenfeld, Yonatan Belinkov, and Nir Shavit.
\newblock A constructive prediction of the generalization error across scales.
\newblock \emph{arXiv preprint arXiv:1909.12673}, 2019.

\bibitem[Telgarsky(2016)]{telgarsky2016benefits}
Matus Telgarsky.
\newblock Benefits of depth in neural networks.
\newblock In \emph{Conference on learning theory}, pages 1517--1539. PMLR, 2016.

\bibitem[Veit et~al.(2016)Veit, Wilber, and Belongie]{veit2016residual}
Andreas Veit, Michael~J Wilber, and Serge Belongie.
\newblock Residual networks behave like ensembles of relatively shallow networks.
\newblock \emph{Advances in neural information processing systems}, 29, 2016.

\bibitem[Wang et~al.(2023)Wang, Panda, Hennigen, Greengard, Karlinsky, Feris, Cox, Wang, and Kim]{wang2023learninggrowpretrainedmodels}
Peihao Wang, Rameswar Panda, Lucas~Torroba Hennigen, Philip Greengard, Leonid Karlinsky, Rogerio Feris, David~Daniel Cox, Zhangyang Wang, and Yoon Kim.
\newblock Learning to grow pretrained models for efficient transformer training, 2023.
\newblock URL \url{https://arxiv.org/abs/2303.00980}.

\bibitem[Yang and Schoenholz(2017)]{yang2017meanfield}
Ge~Yang and Samuel Schoenholz.
\newblock Mean field residual networks: On the edge of chaos.
\newblock \emph{Advances in neural information processing systems}, 30, 2017.

\bibitem[Zhang et~al.(2024)Zhang, Cheng, Zhang, Liu, and Chen]{zhang2024compression}
Boyang Zhang, Daning Cheng, Yunquan Zhang, Fangming Liu, and Wenguang Chen.
\newblock Compression for better: A general and stable lossless compression framework.
\newblock \emph{arXiv preprint arXiv:2412.06868}, 2024.

\bibitem[Zhang et~al.(2019)Zhang, Dauphin, and Ma]{zhang2019fixup}
Hongyi Zhang, Yann~N Dauphin, and Tengyu Ma.
\newblock Fixup initialization: Residual learning without normalization.
\newblock \emph{arXiv preprint arXiv:1901.09321}, 2019.

\end{thebibliography}


\appendix

\section{Related Work}
\label{app sec:related work}
Table~\ref{tab:appendix_novelty_comparison} positions our contribution relative to the closest prior theory work. Existing results mainly fall into three categories: broad explanations of neural scaling laws, generalization bounds for residual architectures, and optimization or mean-field analyses of deep ResNets. Our paper studies a different target: an architecture-specific old-vs-new comparison problem under normalized residual depth expansion. The key distinction is the combination of three ingredients within one framework: a constructive representational-gain step, an optimization comparison step, and a generalization-transfer step, which together yield two complementary test-risk comparison routes.
\begin{table*}[!th]
\centering
\small
\setlength{\tabcolsep}{4pt}
\renewcommand{\arraystretch}{1.15}
\resizebox{1\textwidth}{!}{%
\begin{tabularx}{\textwidth}{p{2.9cm} p{3.2cm} p{3.8cm} X}
\toprule
\textbf{Prior work} & \textbf{Main focus} & \textbf{Difference from our setting} & \textbf{Difference from our paper} \\
\midrule

\citet{bahri2024explaining} &
A general theory of neural scaling laws and scaling regimes for trained deep networks. &
Not specific to normalized residual architectures; does not study constructive \emph{depth expansion by inserting a residual block} into a trained reference model; does not provide an old-vs-new finite-sample \emph{test-risk comparison theorem}. &
We study an \emph{architecture-specific} question: when a normalized residual network is enlarged by inserting one residual block, under what conditions does the resulting model provably avoid worsening or strictly improve \emph{test risk}? \\\hline

\citet{e2020rademacher} &
Rademacher-complexity and generalization-error bounds for residual networks. &
Provides complexity control, but does not prove that depth expansion creates a strictly better comparison model, nor connect such existence to optimization output and then to a final old-vs-new test-risk inequality. &
Our norm-based generalization ingredient is embedded into a larger framework that additionally includes a constructive representational-gain step and a final test-risk comparison step. \\\hline

\citet{kammonen2023smaller} &
Generalization-error analysis showing that a deep residual network can enjoy a smaller bound than a shallow alternative in a specialized random-feature setting. &
The setting is specialized and is not formulated as residual-block insertion into a trained model; it does not provide our two-route comparison between old and expanded models on test risk. &
Our contribution is not merely that deeper residual models may generalize better, but that \emph{depth expansion from a trained reference model} can be translated into explicit old-vs-new test-risk inequalities. \\\hline

\citet{lu2020meanfield} &
Mean-field analysis and favorable optimization/global-convergence properties for deep ResNets in an overparameterized regime. &
Focuses on optimization and mean-field dynamics rather than finite-sample scaling-law comparisons after residual-block insertion; does not provide a representation--optimization--generalization decomposition ending in test-risk comparison. &
Our focus is to bridge constructive representational gain, optimization gain, and generalization transfer, rather than to prove a mean-field global-optimization result. \\\hline

\citet{huang2026generalization} &
Generalization bounds for discrete- and continuous-time ResNets from a dynamical-systems perspective. &
Still mainly a generalization-bound paper; it does not formulate the constructive block-insertion improvement question or provide our complementary population-risk and direct train/test comparison routes. &
Our paper studies a narrower but different target: \emph{architecture-specific test-risk improvement under normalized residual depth expansion}, including a direct train/test route that remains meaningful in degenerate regimes. \\

\bottomrule
\end{tabularx}%
}
\caption{
Comparison with the closest prior theory work. Existing theory studies global neural scaling laws, residual-network generalization bounds, and optimization or mean-field analyses of deep ResNets. Our paper instead focuses on an architecture-specific question: when inserting a residual block into a trained normalized residual network can be converted into a provable non-worsening or strict improvement in test risk.
}
\label{tab:appendix_novelty_comparison}
\end{table*}

 \section{Notation}
\label{app:notation}

For convenience, we collect here the notation used throughout the appendix.
Unless otherwise stated, $\|\cdot\|$ denotes the Euclidean norm for vectors,
$\|\cdot\|_{\sigma}$ denotes the spectral norm of a matrix, and
$\|\cdot\|_{F}$ denotes the Frobenius norm.

\paragraph{Data, losses, and risks.}
Let $\mathcal{D}$ be a distribution on $\mathcal{X}\times\mathcal{Y}$.
The training set and test set are
\[
S_{\mathrm{train}}=\{(x_i,y_i)\}_{i=1}^{M}\sim \mathcal{D}^{M},
\qquad
S_{\mathrm{test}}=\{(\tilde x_j,\tilde y_j)\}_{j=1}^{K}\sim \mathcal{D}^{K},
\]
where the two samples are independent and i.i.d.
The loss function is
\[
\ell:\mathbb{R}^{C}\times \mathcal{Y}\to [0,B_{\ell}].
\]
For a predictor $f$, we define
\[
\mathcal{L}_{\mathrm{train}}(f)
\triangleq
\frac{1}{M}\sum_{i=1}^{M}\ell(f(x_i),y_i),
\qquad
\mathcal{L}_{\mathrm{test}}(f)
\triangleq
\frac{1}{K}\sum_{j=1}^{K}\ell(f(\tilde x_j),\tilde y_j),
\]
and the population risk
\[
R(f)\triangleq \mathbb{E}_{(x,y)\sim\mathcal{D}}[\ell(f(x),y)].
\]

\paragraph{Network decomposition.}
At a designated intermediate layer $l^{*}$, the model is decomposed as
\[
f = f_{\mathrm{top}}\circ f_{\mathrm{bot}},
\]
where
\[
f_{\mathrm{bot}}:\mathcal{X}\to \mathbb{R}^{N},
\qquad
f_{\mathrm{top}}:\mathbb{R}^{N}\to \mathbb{R}^{C}.
\]
Here $N$ is the hidden width at layer $l^{*}$, and $C$ is the output dimension.
For a training sample $x_i$, its hidden representation at layer $l^{*}$ is
\[
x_i^{(l^{*})}\triangleq f_{\mathrm{bot}}(x_i).
\]
Similarly, for a test sample $\tilde x_j$, we write
\[
\tilde x_j^{(l^{*})}\triangleq f_{\mathrm{bot}}(\tilde x_j).
\]

\paragraph{Gradient quantities.}
The average training activation gradient at layer $l^{*}$ is
\[
\mu \triangleq \frac{1}{M}\sum_{i=1}^{M}
\nabla_{x_i^{(l^{*})}}
\ell\!\left(f_{\mathrm{top}}(x_i^{(l^{*})}),y_i\right)\in\mathbb{R}^{N}.
\]
Its population counterpart is
\[
\bar{\mu}\triangleq
\mathbb{E}_{(x,y)\sim\mathcal{D}}
\Bigl[\nabla_{x^{(l^{*})}}
\ell\!\left(f_{\mathrm{top}}(x^{(l^{*})}),y\right)\Bigr]\in\mathbb{R}^{N}.
\]
For the $j$-th test sample, define the single-sample test gradient
\[
\zeta^{(j)}
\triangleq
\nabla_{\tilde x_j^{(l^{*})}}
\ell\!\left(f_{\mathrm{top}}(\tilde x_j^{(l^{*})}),\tilde y_j\right),
\]
and the average test gradient
\[
g\triangleq \frac{1}{K}\sum_{j=1}^{K}\zeta^{(j)}.
\]
In Theorem~\ref{thm 2}, $\Sigma$ denotes the covariance matrix of the
single-sample activation gradient, namely
\[
\Sigma \triangleq \mathrm{Cov}(\zeta),
\]
and $\tau^{2}$, $C_{\Sigma}$ are constants such that
$\lambda_{\max}(\Sigma)\le C_{\Sigma}\tau^{2}$.

\paragraph{Hypothesis classes and model expansion.}
The old hypothesis class is
\[
\mathcal{H}_{\mathrm{old}}
=
\left\{
x\mapsto f_{\mathrm{top}}(f_{\mathrm{bot}}(x)):
f_{\mathrm{top}}\in\mathcal{F}_{\mathrm{top}},
\ f_{\mathrm{bot}}\in\mathcal{F}_{\mathrm{bot}}
\right\}.
\]
A trained model in the old class is denoted by
\[
f_{\mathrm{old}}^{*}\in \mathcal{H}_{\mathrm{old}}.
\]

After inserting a residual block $h$ after layer $l^{*}$, the new hypothesis class is
\[
\mathcal{H}_{\mathrm{new}}
=
\left\{
x\mapsto
f_{\mathrm{top}}\bigl(f_{\mathrm{bot}}(x)+h(f_{\mathrm{bot}}(x))\bigr):
f\in\mathcal{H}_{\mathrm{old}},\ h\in\mathcal{F}_{\mathrm{res}}
\right\}.
\]
Here $\mathcal{F}_{\mathrm{res}}$ denotes the residual-block function class.
The new trained model produced by the optimizer or fallback rule in Condition 2 is denoted by
\[
f_{\mathrm{new}}\in \mathcal{H}_{\mathrm{new}}.
\]
Given a comparison jumpboard model \(\widetilde f\), the final selected model is
\[
f_{\mathrm{new}}\in 
\arg\min_{f\in\{f_{\mathrm{alg}},\widetilde f\}}
\mathcal L_{\mathrm{train}}(f).
\]
Thus \(f_{\mathrm{new}}\) is the final selected model used in the main theorems, not necessarily the raw optimizer output.

\paragraph{Residual-block parameterization and jumpboard models.}
We parameterize the inserted residual block by \(\theta\), writing it as \(h_\theta\), with zero initialization \(h_0=0\). We distinguish two comparison jumpboard models.

For the population-risk route, let \(v_{\mathrm{pop}}\) be a population first-order descent direction. Along the path
\[
\theta(\eta)=\eta v_{\mathrm{pop}},
\]
the population-risk profile is
\[
G_{\mathrm{pop}}(\eta)
:=
\mathbb E_{(x,y)\sim\mathcal D}
\left[
\ell\!\left(
f_{\mathrm{top}}\bigl(x^{(l^\ast)}+
h_{\eta v_{\mathrm{pop}}}(x^{(l^\ast)})\bigr),y
\right)
\right].
\]
For a sufficiently small step size \(\eta_{\mathrm{pop}}>0\), the corresponding population jumpboard model is
\[
\widetilde f_{\mathrm{pop}}(x)
:=
f_{\mathrm{top}}\!\left(
f_{\mathrm{bot}}(x)+
h_{\eta_{\mathrm{pop}}v_{\mathrm{pop}}}(f_{\mathrm{bot}}(x))
\right).
\]

For the direct train/test route, let \(v_S\) be an empirical first-order descent direction selected on \(S_{\mathrm{train}}\). Along the path
\[
\theta(\eta)=\eta v_S,
\]
we define the training-set jumpboard model by
\[
\widetilde f_S(x)
:=
f_{\mathrm{top}}\!\left(
f_{\mathrm{bot}}(x)+
h_{\eta_S v_S}(f_{\mathrm{bot}}(x))
\right),
\]
where \(\eta_S>0\) is chosen so that the empirical training loss decreases along this direction.
The model \(\widetilde f_S\) is random through its dependence on \(S_{\mathrm{train}}\).

\paragraph{Improvement margins.}
The population jumpboard margin used in the population-risk route is
\[
\Delta_R
:=
R(f_{\mathrm{old}}^\ast)
-
R(\widetilde f_{\mathrm{pop}})
>0.
\]

The empirical training-set jumpboard margin used in the direct train/test route is
\[
\Delta_{\mathrm{train},S}
:=
L_{\mathrm{train}}(f_{\mathrm{old}}^\ast)
-
L_{\mathrm{train}}(\widetilde f_S)
>0.
\]

Given the comparison jumpboard model
\[
\widetilde f=
\begin{cases}
\widetilde f_{\mathrm{pop}}, & \text{in the population-risk route},\\
\widetilde f_S, & \text{in the direct train/test route},
\end{cases}
\]
the optimization gain is
\[
\Delta_{\mathrm{ERM}}
:=
L_{\mathrm{train}}(\widetilde f)
-
L_{\mathrm{train}}(f_{\mathrm{new}})
\ge 0.
\]

The finite-test-set jumpboard margin used in Theorem~\ref{thm:4B} is
\[
\Delta_R^{\mathrm{test}}
:=
L_{\mathrm{test}}(f_{\mathrm{old}}^\ast)
-
L_{\mathrm{test}}(\widetilde f_S).
\]
\paragraph{Rademacher complexity and generic generalization notation.}
For a function class $\mathcal{G}$ and a sample
$S=\{z_i\}_{i=1}^{m}$, the empirical Rademacher complexity is
\[
\hat{\mathfrak{R}}_{S}(\mathcal{G})
\triangleq
\mathbb{E}_{\sigma}\left[
\sup_{g\in\mathcal{G}}
\frac{1}{m}\sum_{i=1}^{m}\sigma_i g(z_i)
\right],
\]
where $\sigma_i$ are i.i.d.\ Rademacher random variables.
We write
\[
\epsilon_{\mathrm{gen}}(\mathcal{F},m,\delta)
\]
for a standard uniform generalization bound for a hypothesis class $\mathcal{F}$.

\paragraph{Norm-based generalization bound for normalized residual networks.}
The appendix uses the specialized bound from Lemma~\ref{lem:gen-main}, abbreviated as
\[
\epsilon_{\mathrm{gen}}^{\mathrm{norm}}(\mathcal{H},m,\delta).
\]
When the class is $\mathcal{H}_{\mathrm{new}}$, we use the shorthand
\[
\epsilon_{M}\triangleq
\epsilon_{\mathrm{gen}}^{\mathrm{norm}}(\mathcal{H}_{\mathrm{new}},M,\delta),
\qquad
\epsilon_{K}\triangleq
\epsilon_{\mathrm{gen}}^{\mathrm{norm}}(\mathcal{H}_{\mathrm{new}},K,\delta),
\]
and
\[
\epsilon_{TT}\triangleq \epsilon_{M}+\epsilon_{K}.
\]
In Theorem~\ref{thm:4A}, the total generalization cost appears as $2\epsilon_{M}$;
in Theorem~\ref{thm:4B}, it appears as $2(\epsilon_{M}+\epsilon_{K})=2\epsilon_{TT}$.
The population-level qualitative scaling conclusion is stated in Theorem~\ref{thm:3}.

\paragraph{Layerwise residual-network notation in Lemma~\ref{lem:gen-main}.}
The $l$-th normalized residual layer is written as
\[
T_{l}(z)=\mathcal{N}_{l}(z+h_{l}(z)),
\qquad l=0,\dots,L-1,
\]
where $L$ is the network depth and $\mathcal{N}_{l}$ denotes the normalization
operator (e.g., RMSNorm or LayerNorm).
The layer input is denoted by $z_{l}$, and the residual branch parameters by $W_{l}$.

The structural constants used in Lemma~\ref{lem:gen-main} are:
\begin{itemize}
    \item $s_{l}$: the input-Lipschitz constant of the residual branch $h_{l}(\cdot;W_{l})$;
    \item $L_{l}^{(p)}$: the parameter-Lipschitz constant of $h_{l}$, namely
    \[
    \|h_{l}(z;W_{l})-h_{l}(z;W_{l}')\|
    \le
    L_{l}^{(p)}\|W_{l}-W_{l}'\|_{F}\,\|z\|;
    \]
    \item $r_{\min}>0$: the non-degeneracy lower bound before normalization,
    \[
    \inf_{x,\theta,l}\|z_{l}(x)+h_{l}(z_{l}(x);\theta)\|\ge r_{\min};
    \]
    \item $b_{l}$: the Frobenius-norm radius of the parameter constraint
    $\|W_{l}\|_{F}\le b_{l}$;
    \item $\bar b \triangleq \max_{l} b_{l}$;
    \item $d_{l}$: the parameter dimension of layer $l$;
    \item $d\triangleq \sum_{l=0}^{L-1} d_{l}$: the total parameter dimension;
    \item $B_{x}$: a uniform input norm bound such that $\|x\|\le B_{x}$;
    \item $\gamma_{\max}$: the maximum coordinate scale of the normalization layer;
    \item $B_{\mathrm{norm}}\triangleq \gamma_{\max}\sqrt{N}$: the post-normalization
    representation norm bound;
    \item $B_{*}\triangleq \max\{B_{x},B_{\mathrm{norm}}\}$.
\end{itemize}

\paragraph{Top map and loss regularity constants.}
The loss $\ell$ is assumed to be $L_{\ell}$-Lipschitz and bounded by $B_{\ell}$.
The top map $f_{\mathrm{top}}$ is assumed to be $\mathcal L_{\mathrm{top}}$-Lipschitz, with
reference-point bound
\[
\|f_{\mathrm{top}}(0)\|\le B_{0}.
\]

\paragraph{Layerwise Lipschitz amplification constants.}
For the normalization operator in layer $l$, we denote its Lipschitz factor by $c_{l}$.
The output sensitivity to a perturbation at layer $l$ is summarized by
\[
\Lambda_{l}
\triangleq
L_{\mathrm{top}}\cdot L_{l}^{(p)}\cdot B_{*}\cdot c_{l}
\prod_{k>l} c_{k}(1+s_{k}),
\]
and
\[
\Lambda_{\max}\triangleq \max_{0\le l\le L-1}\Lambda_{l}.
\]
These constants enter the covering-number bound and hence the complexity estimate
in Lemma~\ref{lem:gen-main}.

\paragraph{Covering parameter.}
In the proof of Lemma~\ref{lem:gen-main}, $\epsilon_{0}>0$ denotes the covering accuracy for the
output/loss class. A common choice is
\[
\epsilon_{0}=M^{-1/2}.
\]

\paragraph{Failure probabilities.}
We use $\delta,\delta_{1},\delta_{2},\delta_{3}$ for generic confidence parameters.
In Theorem~\ref{thm:4A}, the total failure probability is decomposed as
\[
\delta_{\mathrm{gen}}=\delta_{1}+\delta_{2},
\qquad
\delta_{\mathrm{test}}
\]
for the generalization and test-concentration parts, respectively.

\paragraph{Standing convention.}
Throughout the appendix, symbols with subscript ``old'' refer to the model or
hypothesis class before depth expansion, symbols with subscript ``new'' refer to
their counterparts after inserting the new residual block, and a tilde (as in
$\tilde f$) denotes the auxiliary jumpboard model used only for comparison in the proof.

\section{Clarification of Condition 1: From Hidden-Space Descent to Residual-Block Realizability}
\label{app:assumption2-clarification}

This appendix clarifies the meaning of Condition 1 and explains how it should
be interpreted in standard residual architectures. The key point is that
Condition 1 combines two logically separate requirements: first, the existence
of a descent direction in the hidden representation space; second, the ability
of the inserted residual block to realize such a direction to first order near
the zero residual function.

Throughout this section, let
\[
z=x^{(l^\ast)}=f_{\mathrm{bot}}(x)\in\mathbb{R}^N,
\]
and define the activation-gradient signal
\[
q(z,y)
:=
\nabla_z \ell(f_{\mathrm{top}}(z),y)
\in\mathbb{R}^N.
\]
The old model is
\[
f_{\mathrm{old}}^\ast(x)=f_{\mathrm{top}}(z),
\]
whereas the expanded model obtained by inserting a residual block is
\[
f_\theta(x)
=
f_{\mathrm{top}}\bigl(z+h_\theta(z)\bigr).
\]

\subsection{Hidden-space descent direction}

The most direct way to understand the first-order mechanism is to first ignore
the parameterization of the residual block and ask whether there exists a
hidden-space perturbation direction
\[
\phi:\mathbb{R}^N\to\mathbb{R}^N
\]
such that moving the representation \(z\) along \(\phi(z)\) decreases the
population risk to first order.

For such a perturbation, define
\[
G_\phi(\eta)
:=
\mathbb{E}_{(x,y)\sim\mathcal{D}}
\left[
\ell\bigl(f_{\mathrm{top}}(z+\eta\phi(z)),y\bigr)
\right].
\]
By the local \(C^1\) regularity assumption,
\[
G_\phi'(0)
=
\mathbb{E}_{(x,y)\sim\mathcal{D}}
\left[
q(z,y)^\top \phi(z)
\right].
\]
Therefore, if
\[
\mathbb{E}_{(x,y)\sim\mathcal{D}}
\left[
q(z,y)^\top \phi(z)
\right]
\neq 0,
\]
then, by replacing \(\phi\) with \(-\phi\) if necessary, we may assume
\[
\mathbb{E}_{(x,y)\sim\mathcal{D}}
\left[
q(z,y)^\top \phi(z)
\right]
<0.
\]
It follows that there exists a sufficiently small \(\eta_0>0\) such that
\[
G_\phi(\eta_0)<G_\phi(0).
\]
In other words,
\[
\mathbb{E}_{(x,y)\sim\mathcal{D}}
\left[
\ell\bigl(f_{\mathrm{top}}(z+\eta_0\phi(z)),y\bigr)
\right]
<
\mathbb{E}_{(x,y)\sim\mathcal{D}}
\left[
\ell\bigl(f_{\mathrm{top}}(z),y\bigr)
\right].
\]
This is the hidden-space descent mechanism underlying the jumpboard model.

\paragraph{Empirical hidden-space descent direction.}
The same construction has a training-set analogue. Given
\(S_{\mathrm{train}}=\{(x_i,y_i)\}_{i=1}^M\), write
\(z_i=x_i^{(l^\ast)}=f_{\mathrm{bot}}(x_i)\) and
\[
q_i:=\nabla_{z_i}\ell(f_{\mathrm{top}}(z_i),y_i).
\]
For a hidden-space perturbation \(\phi:\mathbb R^N\to\mathbb R^N\), define the empirical profile
\[
G_{\phi,S}(\eta)
:=
\frac1M\sum_{i=1}^M
\ell\!\left(f_{\mathrm{top}}(z_i+\eta\phi(z_i)),y_i\right).
\]
Then
\[
G_{\phi,S}'(0)
=
\frac1M\sum_{i=1}^M q_i^\top \phi(z_i).
\]
Hence, if
\[
\frac1M\sum_{i=1}^M q_i^\top \phi(z_i)<0,
\]
then for sufficiently small \(\eta_S>0\),
\[
G_{\phi,S}(\eta_S)<G_{\phi,S}(0).
\]
This gives the empirical hidden-space descent mechanism used to construct the training-set jumpboard model
\(\widetilde f_S\) in the direct train/test route.

\subsection{Residual-block realizability}

The previous argument only shows that a hidden-space perturbation
\(\phi(z)\) would reduce the population risk. To use this direction inside the
expanded network, we must require that the inserted residual block can realize
such a perturbation, at least to first order near zero initialization.

Let the inserted residual block be parameterized by
\[
\theta\in\mathbb{R}^p,
\]
and assume that
\[
h_0(z)=0
\]
for all relevant hidden states \(z\). For each fixed \(z\), define the parameter
Jacobian at zero by
\[
J_\theta h_0(z)
:=
\left.
\frac{\partial h_\theta(z)}{\partial \theta}
\right|_{\theta=0}
\in\mathbb{R}^{N\times p}.
\]
Here \(J_\theta h_0(z)\) is a Jacobian with respect to the residual-block
parameters, not with respect to the hidden input \(z\). Given a parameter
direction
\[
u\in\mathbb{R}^p,
\]
the induced first-order hidden-space perturbation is
\[
\phi_u(z)
:=
J_\theta h_0(z)u
\in\mathbb{R}^N.
\]
Equivalently, along the parameter path
\[
\theta(\eta)=\eta u,
\]
we have the first-order expansion
\[
h_{\eta u}(z)
=
h_0(z)+\eta J_\theta h_0(z)u+r_\eta(z),
\]
where the remainder satisfies
\[
\frac{\mathbb{E}_{(x,y)\sim\mathcal{D}}\|r_\eta(z)\|}{|\eta|}
\to 0
\qquad
\text{as } \eta\to 0.
\]
Since \(h_0(z)=0\), this becomes
\[
h_{\eta u}(z)
=
\eta\phi_u(z)+r_\eta(z).
\]

Thus the inserted residual block realizes the hidden-space perturbation
\(\phi_u\) to first order. This is the precise meaning of the parameter-space
form of Condition 1.

\subsection{Equivalent Jacobian form of Condition 1}

Using the notation above, Condition 1 can be equivalently written as follows.

\paragraph{Condition 1' (first-order residual non-degeneracy).}
There exists a parameter direction \(u\in\mathbb{R}^p\) such that
\[
a_u
:=
\mathbb{E}_{(x,y)\sim\mathcal{D}}
\left[
q(x^{(l^\ast)},y)^\top
J_\theta h_0(x^{(l^\ast)})u
\right]
\neq 0.
\]
If \(a_u>0\), we replace \(u\) by \(-u\). Hence, without loss of generality,
we may assume
\[
a_u<0.
\]

This form makes the dimensions explicit:
\[
q(x^{(l^\ast)},y)\in\mathbb{R}^N,
\qquad
J_\theta h_0(x^{(l^\ast)})\in\mathbb{R}^{N\times p},
\qquad
u\in\mathbb{R}^p.
\]
Therefore,
\[
J_\theta h_0(x^{(l^\ast)})u\in\mathbb{R}^N,
\]
and the inner product
\[
q(x^{(l^\ast)},y)^\top
J_\theta h_0(x^{(l^\ast)})u
\]
is a well-defined scalar.

Equivalently, Condition 1' states that the first-order tangent class
\[
\mathcal{T}_{h}
:=
\left\{
z\mapsto J_\theta h_0(z)u:
u\in\mathbb{R}^p
\right\}
\]
is not orthogonal to the population activation-gradient functional
\[
\phi
\mapsto
\mathbb{E}_{(x,y)\sim\mathcal{D}}
\left[
q(z,y)^\top \phi(z)
\right].
\]
Importantly, this assumption does not require the residual block to span all
possible hidden-space perturbations. It only requires that its first-order
tangent class contains at least one direction that has a nonzero projection
onto the population activation-gradient signal.

\paragraph{Population version.}
The population residual-block realizability condition is that there exists a parameter direction
\(v_{\mathrm{pop}}\) such that
\[
\mathbb E_{(x,y)\sim\mathcal D}
\left[
q(z,y)^\top
\left.
\frac{\partial h_\theta(z)}{\partial \theta}
\right|_{\theta=0}
v_{\mathrm{pop}}
\right]
<0.
\]
This is the condition used to construct the population jumpboard model
\(\widetilde f_{\mathrm{pop}}\).

\paragraph{Training-set version.}
The empirical residual-block realizability condition is that there exists a parameter direction
\(v_S\), possibly depending on \(S_{\mathrm{train}}\), such that
\[
\frac1M\sum_{i=1}^M
q_i^\top
\left.
\frac{\partial h_\theta(z_i)}{\partial \theta}
\right|_{\theta=0}
v_S
<0.
\]
This is the condition used to construct the training-set jumpboard model
\(\widetilde f_S\).
The direction \(v_S\) is random through its dependence on the training sample.

\subsection{Construction of the jumpboard model}

Under Condition 1', choose \(u\) such that
\[
a_u
=
\mathbb{E}_{(x,y)\sim\mathcal{D}}
\left[
q(z,y)^\top J_\theta h_0(z)u
\right]
<0.
\]
Define the parameter path
\[
\theta(\eta)=\eta u,
\]
and the corresponding expanded model
\[
\widetilde f_{pop}
:=
f_{\mathrm{top}}\bigl(z+h_{\eta u}(z)\bigr).
\]
The associated population-risk profile is
\[
G(\eta)
:=
R(\widetilde f_\eta)
=
\mathbb{E}_{(x,y)\sim\mathcal{D}}
\left[
\ell\bigl(f_{\mathrm{top}}(z+h_{\eta u}(z)),y\bigr)
\right].
\]
Using
\[
h_{\eta u}(z)
=
\eta J_\theta h_0(z)u+r_\eta(z),
\]
and the local differentiability of the loss-composition map, we obtain
\[
G'(0)
=
\mathbb{E}_{(x,y)\sim\mathcal{D}}
\left[
q(z,y)^\top J_\theta h_0(z)u
\right]
=
a_u
<0.
\]
Therefore, there exists a sufficiently small \(\eta_0>0\) such that
\[
G(\eta_0)<G(0).
\]
Equivalently,
\[
R(\widetilde f_{pop})
<
R(f_{\mathrm{old}}^\ast).
\]
The model
\[
\widetilde f_{pop}(x)
=
f_{\mathrm{top}}\bigl(x^{(l^\ast)}+h_{\eta_0u}(x^{(l^\ast)})\bigr)
\]
is the jumpboard model used in the main proof.

By analogy with the procedure described above, we can construct an empirical (training-set) jumpboard model $\widetilde f_S$, $\mathcal L_{train}(\widetilde f_{S})
<
\mathcal L_{train}(f_{\mathrm{old}}^\ast)$, for the training set following the same construction logic.

\subsection{Relation to standard zero-output residual blocks}
\label{app:standard-zero-output-residual-blocks}

We now explain how Condition 1' maps to standard residual blocks. The key
observation is that many residual branches can be written in the form
\[
h_{U,V}(z)=V\psi_U(z),
\]
where
\[
z\in\mathbb{R}^N
\]
is the input hidden representation,
\[
\psi_U:\mathbb{R}^N\to\mathbb{R}^m
\]
is the feature map produced by all operations in the residual branch before the
last output projection, \(U\) denotes the parameters of these preceding
operations, and
\[
V\in\mathbb{R}^{N\times m}
\]
is the final linear projection back to the residual-stream dimension. Thus
\[
h_{U,V}:\mathbb{R}^N\to\mathbb{R}^N,
\]
so that \(h_{U,V}(z)\) has the same dimension as \(z\) and can be added
residually.

This form is not a special architecture. It is an abstraction of the standard
``feature extraction plus output projection'' structure used in residual
networks. The map \(\psi_U\) collects all preceding transformations, such as
normalization, activation functions, MLP layers, attention operations, or
convolutional layers, while \(V\) denotes the last projection that maps the
residual-branch features back to the dimension of the residual stream.

\paragraph{Example 1: MLP or convolutional residual branch.}
A two-layer MLP residual branch has the form
\[
h(z)=W_2\sigma(W_1 z).
\]
This is a special case of \(h_{U,V}(z)=V\psi_U(z)\) by taking
\[
\psi_U(z)=\sigma(W_1 z),
\qquad
V=W_2,
\]
where \(U=W_1\). Similarly, a convolutional residual branch can be written as
\[
h(z)=\mathrm{Conv}_{\mathrm{out}}(\psi_U(z)),
\]
where \(\psi_U(z)\) denotes all preceding convolutions, normalizations, and
activation functions, and \(\mathrm{Conv}_{\mathrm{out}}\) is the final
convolutional projection. Since a convolution is a linear map with parameter
sharing, it plays the same role as \(V\) in the abstract expression
\[
h_{U,V}(z)=V\psi_U(z).
\]
Thus standard ResNet-style residual branches fall into this form once the final
linear or convolutional output projection is separated from the preceding
feature extractor.

\paragraph{Example 2: Transformer/LLM residual branch.}
The same structure appears in Transformer blocks used in large language models.
For an MLP sublayer, the residual branch often has the schematic form
\[
h(z)=W_{\mathrm{out}}\,
\sigma(W_{\mathrm{in}}\,\mathrm{Norm}(z)),
\]
or, in a gated MLP,
\[
h(z)
=
W_{\mathrm{out}}
\left(
\sigma(W_{\mathrm{gate}}\mathrm{Norm}(z))
\odot
W_{\mathrm{up}}\mathrm{Norm}(z)
\right).
\]
Both are of the form \(h_{U,V}(z)=V\psi_U(z)\). In the first case,
\[
\psi_U(z)=\sigma(W_{\mathrm{in}}\mathrm{Norm}(z)),
\qquad
V=W_{\mathrm{out}}.
\]
In the gated case,
\[
\psi_U(z)
=
\sigma(W_{\mathrm{gate}}\mathrm{Norm}(z))
\odot
W_{\mathrm{up}}\mathrm{Norm}(z),
\qquad
V=W_{\mathrm{out}}.
\]
The attention sublayer has the same output-projection structure:
\[
h(z)=W_O\,\mathrm{Attn}_U(\mathrm{Norm}(z)),
\]
where \(\mathrm{Attn}_U\) contains the query, key, and value projections and the
attention operation, while \(W_O\) is the final output projection. Hence it also
fits the abstract form with
\[
\psi_U(z)=\mathrm{Attn}_U(\mathrm{Norm}(z)),
\qquad
V=W_O.
\]

We now return to the first-order non-degeneracy condition. Suppose the final
output projection is initialized at zero:
\[
V=0,
\]
while the preceding feature extractor is fixed at some initialization \(U_0\).
Then
\[
h_{U_0,0}(z)=0.
\]
Therefore, inserting this residual block initially recovers the old model
exactly.

However, this zero-output initialization does not make the first-order tangent
space vanish. For a perturbation direction
\[
\Delta V\in\mathbb{R}^{N\times m},
\]
consider the parameter path
\[
V(\eta)=\eta\Delta V,
\qquad
U(\eta)=U_0.
\]
Then
\[
h_{U_0,\eta\Delta V}(z)
=
\eta\Delta V\psi_{U_0}(z).
\]
Thus the first-order hidden-space perturbation induced by changing the final
projection is
\[
\phi_{\Delta V}(z)
=
\Delta V\psi_{U_0}(z).
\]

Condition 1' holds if there exists \(\Delta V\) such that
\[
\mathbb{E}_{(x,y)\sim\mathcal D}
\left[
q(z,y)^\top \Delta V\psi_{U_0}(z)
\right]
\neq 0.
\]
A sufficient condition is
\[
C
:=
\mathbb{E}_{(x,y)\sim\mathcal D}
\left[
q(z,y)\psi_{U_0}(z)^\top
\right]
\neq 0.
\]
Indeed, choosing
\[
\Delta V=-C
\]
gives
\[
\mathbb{E}_{(x,y)\sim\mathcal D}
\left[
q(z,y)^\top \Delta V\psi_{U_0}(z)
\right]
=
-
\left\|
\mathbb{E}_{(x,y)\sim\mathcal D}
\left[
q(z,y)\psi_{U_0}(z)^\top
\right]
\right\|_F^2
<0.
\]

The empirical analogue is obtained by replacing the population expectation with the training average:
\[
C_S
:=
\frac1M\sum_{i=1}^M q_i\psi_{U_0}(z_i)^\top.
\]
If \(C_S\neq 0\), choosing
\[
\Delta V=-C_S
\]
gives
\[
\frac1M\sum_{i=1}^M
q_i^\top \Delta V\psi_{U_0}(z_i)
=
-\left\|
\frac1M\sum_{i=1}^M q_i\psi_{U_0}(z_i)^\top
\right\|_F^2
<0.
\]
Therefore, the same zero-output residual block construction also yields an empirical jumpboard direction \(v_S\) whenever the residual features are not training-set-orthogonal to the empirical activation-gradient signal.
Therefore, a standard residual branch with a zero-initialized final projection
satisfies Condition 1' whenever its residual features are not
population-orthogonal to the activation-gradient signal.

This also explains why zeroing the entire residual branch is not the intended
condition. If all layers in the branch were initialized to zero, the first-order
Jacobian with respect to the parameters could vanish. In contrast,
zero-initializing only the final output projection ensures \(h_{U_0,0}(z)=0\)
while preserving nontrivial first-order directions through perturbations of
\(V\). This is precisely the zero-output residual parameterization needed for
the jumpboard construction.
\paragraph{Corollary: standard zero-output residual blocks satisfy Condition 1.}
Based on the above analysis, we obtain the following corollary.

\begin{corollary}[First-order non-degeneracy for standard zero-output residual blocks]
\label{cor:zero-output-residual-block}
Consider an inserted residual branch of the form
\[
h_{U,V}(z)=V\psi_U(z),
\]
where \(z=x^{(l^\ast)}\in\mathbb R^N\), \(\psi_U(z)\in\mathbb R^m\), and
\(V\in\mathbb R^{N\times m}\) is the final output projection. Suppose the block is initialized by
\[
U=U_0,\qquad V=0.
\]
Then \(h_{U_0,0}(z)=0\), so inserting the block initially recovers the old model exactly.

For the population-risk route, define
\[
C_{\mathrm{pop}}
:=
\mathbb E_{(x,y)\sim\mathcal D}
\left[
q(z,y)\psi_{U_0}(z)^\top
\right],
\qquad
q(z,y):=\nabla_z\ell(f_{\mathrm{top}}(z),y).
\]
If
\[
C_{\mathrm{pop}}\neq 0,
\]
then the inserted residual block satisfies the population first-order non-degeneracy condition. In particular, choosing
\[
\Delta V=-C_{\mathrm{pop}}
\]
gives
\[
\mathbb E_{(x,y)\sim\mathcal D}
\left[
q(z,y)^\top \Delta V\psi_{U_0}(z)
\right]
=
-\|C_{\mathrm{pop}}\|_F^2
<0.
\]
Hence there exists a population descent direction and therefore a population jumpboard model
\(\widetilde f_{\mathrm{pop}}\).

For the empirical train/test route, given
\(S_{\mathrm{train}}=\{(x_i,y_i)\}_{i=1}^M\), define
\[
z_i:=f_{\mathrm{bot}}(x_i),
\qquad
q_i:=\nabla_{z_i}\ell(f_{\mathrm{top}}(z_i),y_i),
\]
and
\[
C_S
:=
\frac1M\sum_{i=1}^M q_i\psi_{U_0}(z_i)^\top .
\]
If
\[
C_S\neq 0,
\]
then the inserted residual block satisfies the empirical first-order non-degeneracy condition. In particular, choosing
\[
\Delta V=-C_S
\]
gives
\[
\frac1M\sum_{i=1}^M
q_i^\top \Delta V\psi_{U_0}(z_i)
=
-\|C_S\|_F^2
<0.
\]
Hence there exists an empirical descent direction and therefore a training-set jumpboard model
\(\widetilde f_S\).
\end{corollary}

\begin{proof}
Since \(V=0\), we have \(h_{U_0,0}(z)=0\), so the expanded network initially represents the old model.
Now perturb only the final projection by setting
\[
V(\eta)=\eta\Delta V,\qquad U(\eta)=U_0.
\]
Then
\[
h_{U_0,\eta\Delta V}(z)
=
\eta\Delta V\psi_{U_0}(z),
\]
so the induced first-order hidden-space perturbation is
\[
\phi_{\Delta V}(z)=\Delta V\psi_{U_0}(z).
\]
For the population route, the first-order variation of the risk along this direction is
\[
\mathbb E_{(x,y)\sim\mathcal D}
\left[
q(z,y)^\top \Delta V\psi_{U_0}(z)
\right].
\]
Taking \(\Delta V=-C_{\mathrm{pop}}\) gives
\[
-\|C_{\mathrm{pop}}\|_F^2<0.
\]
The empirical statement follows identically by replacing the population expectation with the training average and choosing \(\Delta V=-C_S\).
\end{proof}

This corollary shows that Condition~1 is not an artificial functional assumption. 
It is satisfied by standard zero-output residual branches whenever the residual features are not orthogonal to the relevant activation-gradient signal, either in population or on the training set.
\subsection{Why the derivative is taken with respect to parameters}

The Jacobian in Condition 1' is the parameter Jacobian
\[
J_\theta h_0(z)
=
\left.
\frac{\partial h_\theta(z)}{\partial\theta}
\right|_{\theta=0},
\]
not the input Jacobian
\[
\frac{\partial h_\theta(z)}{\partial z}.
\]
This is because the jumpboard construction fixes the old representation
\[
z=x^{(l^\ast)}
\]
and moves only the parameters of the newly inserted residual block. The path
used to construct the jumpboard model is
\[
\theta(\eta)=\eta u,
\]
so the first-order change in the model comes from changing \(\theta\), not from
changing the input \(z\) to the residual branch.

For example, in the zero-last-layer block
\[
h_V(z)=V\psi_U(z),
\qquad
V=0,
\]
the input Jacobian satisfies
\[
\left.
\frac{\partial h_V(z)}{\partial z}
\right|_{V=0}
=
0,
\]
because the output layer \(V\) is zero. However, the parameter Jacobian with
respect to \(V\) is nonzero:
\[
\left.
\frac{\partial h_V(z)}{\partial V}
\right|_{V=0}
[\Delta V]
=
\Delta V\psi_U(z).
\]
Thus using the input Jacobian would incorrectly rule out standard
zero-last-layer residual blocks, while the parameter Jacobian captures exactly
the first-order hidden-space perturbations created by training the inserted
block.
\subsection{Why Condition 1 is mild.}
We further explain why Condition 1 is a generic non-orthogonality requirement.
Suppose the old model is trained to stationarity with respect to its own
parameters. This only gives
\[
\nabla_{\theta_{\mathrm{old}}} R(f_{\mathrm{old}}^\ast)=0,
\]
or, at the empirical level,
\[
\nabla_{\theta_{\mathrm{old}}} \mathcal L_{\mathrm{train}}(f_{\mathrm{old}}^\ast)=0.
\]
By the chain rule, this means that the activation-gradient signal is orthogonal
to perturbations of the hidden representation that can already be generated by
the old parameters. It does not imply that the activation-gradient signal itself
is zero.

Now consider a standard zero-output residual branch of the form
\[
h_{U,V}(z)=V\psi_U(z),
\]
where the output projection is initialized at \(V=0\). Then \(h_{U,0}(z)=0\), so
the old model is exactly embedded in the expanded class. However, perturbing
\(V\) gives nontrivial first-order residual directions:
\[
h_{U,\eta \Delta V}(z)
=
\eta \Delta V\psi_U(z).
\]
The first-order population variation is therefore
\[
\mathbb E_{(x,y)\sim\mathcal D}
\left[
q(z,y)^\top \Delta V\psi_U(z)
\right],
\qquad
q(z,y):=\nabla_z\ell(f_{\mathrm{top}}(z),y).
\]
Equivalently, Condition 1 holds whenever
\[
\mathbb E_{(x,y)\sim\mathcal D}
\left[
q(z,y)\psi_U(z)^\top
\right]
\neq 0.
\]
If this matrix is nonzero, choosing the sign of \(\Delta V\) appropriately gives
a strict first-order descent direction. Thus the condition fails only in the
special case where the activation-gradient signal is population-orthogonal to
all residual features realizable by the inserted block. This is a degeneracy
condition, not the typical case.
 
\section{Population Jumpboard Existence: Proof of Theorem~\ref{thm:1}}
\label{app:proof-thm1}

\begin{theorem}[Value of Depth]
\label{thm:1}
Under Assumptions 1 and 2(a), there exists a jumpboard model
$\tilde f \in \mathcal{H}_{\mathrm{new}}$ such that
\[
R(\tilde f_{pop}) < R(f_{\mathrm{old}}^{*}).
\]
\end{theorem}

\begin{proof}
We divide the proof into three steps.

\paragraph{Step 1: Construct a descent path in parameter space.}
By Condition 1, the residual block class $\mathcal{F}_{\mathrm{res}}$
admits a parameterization $h_{\theta}$ with zero initialization $h_{0}=0$,
and there exists a parameter direction $v_{pop}$ such that
\[
a \triangleq
\mathbb{E}_{(x,y)\sim\mathcal{D}}
\left[
\nabla_{x^{(l^{*})}}
\ell\!\left(f_{\mathrm{top}}(x^{(l^{*})}),y\right)^{\top}
\cdot
\frac{\partial h_{\theta}(x^{(l^{*})})}{\partial \theta}\Big|_{\theta=0}
\cdot v_{pop}
\right]
\neq 0.
\]
If $a>0$, we replace $v$ by $-v$. Hence, without loss of generality, we may assume
\[
a<0.
\]

Now consider the one-dimensional parameter path
\[
\theta(t)=tv,
\qquad t\ge 0,
\]
and define the corresponding population-risk profile
\[
G(t)
\triangleq
\mathbb{E}_{(x,y)\sim\mathcal{D}}
\Bigl[
\ell\!\left(
f_{\mathrm{top}}\bigl(x^{(l^{*})}+h_{tv}(x^{(l^{*})})\bigr),y
\right)
\Bigr].
\]
Since $h_{0}=0$, at $t=0$ the inserted residual block vanishes, so the model reduces
to the old model. Therefore,
\[
G(0)=R(f_{\mathrm{old}}^{*}).
\]

\paragraph{Step 2: Compute the directional derivative at the origin.}
By Assumption 1, the map
\[
v\mapsto
\mathbb{E}_{(x,y)\sim\mathcal{D}}
\left[
\ell\!\left(f_{\mathrm{top}}(x^{(l^{*})}+v),y\right)
\right]
\]
is continuously differentiable in a neighborhood of the origin.
Hence, by the chain rule, $G$ is differentiable at $t=0$, and
\[
G'(0)
=
\mathbb{E}_{(x,y)\sim\mathcal{D}}
\left[
\nabla_{x^{(l^{*})}}
\ell\!\left(f_{\mathrm{top}}(x^{(l^{*})}),y\right)^{\top}
\cdot
\frac{\partial h_{tv_{pop}}(x^{(l^{*})})}{\partial t}\Big|_{t=0}
\right].
\]
Since $\theta=t v_{pop}$, we have
\[
\frac{\partial h_{tv}(x^{(l^{*})})}{\partial t}\Big|_{t=0}
=
\frac{\partial h_{\theta}(x^{(l^{*})})}{\partial \theta}\Big|_{\theta=0}\cdot v_{pop}.
\]
Substituting this identity into the previous display yields
\[
G'(0)
=
\mathbb{E}_{(x,y)\sim\mathcal{D}}
\left[
\nabla_{x^{(l^{*})}}
\ell\!\left(f_{\mathrm{top}}(x^{(l^{*})}),y\right)^{\top}
\cdot
\frac{\partial h_{\theta}(x^{(l^{*})})}{\partial \theta}\Big|_{\theta=0}
\cdot v_{pop}
\right]
=
a<0.
\]

Thus, the population risk decreases to first order along the path $\theta(t)=tv$
at the origin.

\paragraph{Step 3: Use negativity of the derivative to obtain a strictly better model.}
Since $G$ is differentiable at $0$ and $G'(0)<0$, there exists $t_{0}>0$ such that
\[
G(t_{0})<G(0).
\]
Let
\[
\theta_{0}=t_{0}v,
\]
and define the corresponding jumpboard model
\[
\tilde f(x)
\triangleq
f_{\mathrm{top}}
\Bigl(
f_{\mathrm{bot}}(x)+h_{\theta_{0}}(f_{\mathrm{bot}}(x))
\Bigr).
\]
Because $h_{\theta_{0}}\in \mathcal{F}_{\mathrm{res}}$, we have
$\tilde f\in\mathcal{H}_{\mathrm{new}}$. Moreover,
\[
R(\tilde f)=G(t_{0})<G(0)=R(f_{\mathrm{old}}^{*}).
\]
Therefore, there exists a model in the expanded hypothesis class
$\mathcal{H}_{\mathrm{new}}$ whose population risk is strictly smaller than that
of the old model $f_{\mathrm{old}}^{*}$.

This completes the proof.
\end{proof}

\paragraph{Remark.}
The proof is purely population-level. It does not require any concentration on the
training set or the test set. Its only purpose is to establish the \emph{existence}
of a strictly better model in the enlarged hypothesis class. In particular, the jump
model $\tilde f$ is an auxiliary comparison model; the actual optimizer is not
required to find this model exactly.

\section{Proof of Theorem~\ref{thm 2}}
\label{app:proof-thm2}

For completeness, we state Theorem~\ref{thm 2}.

\begin{theorem}[Train/test alignment for empirical jumpboards]
\label{thm 2}
Fix the trained reference model \(f_{\mathrm{old}}^\ast=f_{\mathrm{top}}\circ f_{\mathrm{bot}}\)
and an insertion layer \(l^\ast\). Let
\[
\zeta
:=
\nabla_{x^{(l^\ast)}}\ell
\bigl(
f_{\mathrm{top}}(x^{(l^\ast)}),y
\bigr)
\in\mathbb R^N
\]
be the activation-gradient random vector under \((x,y)\sim\mathcal D\). Define
\[
\bar\mu := \mathbb E[\zeta],
\qquad
\Sigma := \operatorname{Cov}(\zeta),
\]
and assume \(\bar\mu\neq 0\). Let
\[
\mu
:=
\frac1M\sum_{i=1}^M \zeta_i,
\qquad
g
:=
\frac1K\sum_{j=1}^K \widetilde\zeta_j,
\]
where \(\zeta_1,\ldots,\zeta_M\) are the training activation gradients and
\(\widetilde\zeta_1,\ldots,\widetilde\zeta_K\) are the independent test activation gradients.
Suppose Assumption~3 holds, namely
\[
\lambda_{\max}(\Sigma)\le C_\Sigma\tau^2 .
\]
Then
\[
\mathbb P(\mu^\top g\le 0)
\le
\frac{4C_\Sigma\tau^2}{M\|\bar\mu\|^2}
+
\frac{4C_\Sigma\tau^2}{K\|\bar\mu\|^2}
+
\frac{4C_\Sigma\tau^2\operatorname{tr}(\Sigma)}
{KM\|\bar\mu\|^4}.
\]
In particular, if
\[
\operatorname{tr}(\Sigma)\le N\tau^2
\qquad\text{and}\qquad
\|\bar\mu\|^2=\Theta(N),
\]
then
\[
\mathbb P(\mu^\top g\le 0)
=
O\!\left(
\frac1{MN}
+
\frac1{KN}
\right).
\]
Thus the probability that the training-selected empirical direction remains positively aligned with the independent test direction improves with width, training size, and test size.
\end{theorem}

\begin{proof}
The proof separates the failure event into two sources. 
First, the empirical training direction \(\mu\) may fail to align with the population direction \(\bar\mu\). 
Second, even when the training direction is reliable, the independent test direction \(g\) may fluctuate enough to make \(\mu^\top g\le 0\).

Define the good training-alignment event
\[
A
:=
\left\{
\mu^\top\bar\mu
\ge
\frac12\|\bar\mu\|^2
\right\}.
\]
Then
\[
\mathbb P(\mu^\top g\le 0)
\le
\mathbb P(A^c)
+
\mathbb P(\mu^\top g\le 0,\ A).
\]
We bound the two terms separately.

First, since
\[
\mu^\top\bar\mu
=
\|\bar\mu\|^2
+
(\mu-\bar\mu)^\top\bar\mu,
\]
the event \(A^c\) implies
\[
(\mu-\bar\mu)^\top\bar\mu
<
-\frac12\|\bar\mu\|^2.
\]
Therefore,
\[
\mathbb P(A^c)
\le
\mathbb P\left(
\left|
(\mu-\bar\mu)^\top\bar\mu
\right|
\ge
\frac12\|\bar\mu\|^2
\right).
\]
Because \(\mu\) is the average of \(M\) i.i.d. copies of \(\zeta\), we have
\[
\mathbb E[\mu]=\bar\mu,
\qquad
\operatorname{Cov}(\mu)=\frac{\Sigma}{M}.
\]
Hence
\[
\operatorname{Var}\bigl((\mu-\bar\mu)^\top\bar\mu\bigr)
=
\bar\mu^\top\operatorname{Cov}(\mu)\bar\mu
=
\frac1M\bar\mu^\top\Sigma\bar\mu.
\]
Using \(\lambda_{\max}(\Sigma)\le C_\Sigma\tau^2\), we obtain
\[
\bar\mu^\top\Sigma\bar\mu
\le
\lambda_{\max}(\Sigma)\|\bar\mu\|^2
\le
C_\Sigma\tau^2\|\bar\mu\|^2.
\]
Thus
\[
\operatorname{Var}\bigl((\mu-\bar\mu)^\top\bar\mu\bigr)
\le
\frac{C_\Sigma\tau^2}{M}\|\bar\mu\|^2.
\]
By Chebyshev's inequality,
\[
\mathbb P(A^c)
\le
\frac{
\frac{C_\Sigma\tau^2}{M}\|\bar\mu\|^2
}{
\frac14\|\bar\mu\|^4
}
=
\frac{4C_\Sigma\tau^2}{M\|\bar\mu\|^2}.
\]
This is the training-sample failure term.

We now bound the second term,
\[
\mathbb P(\mu^\top g\le 0,\ A).
\]
Condition on the training sample \(S_{\mathrm{train}}\). Then \(\mu\) is fixed, while \(g\) remains random and independent of the training sample. Since
\[
g=\frac1K\sum_{j=1}^K\widetilde\zeta_j,
\]
we have
\[
\mathbb E[g\mid S_{\mathrm{train}}]=\bar\mu,
\qquad
\operatorname{Cov}(g\mid S_{\mathrm{train}})=\frac{\Sigma}{K}.
\]
Therefore,
\[
\mathbb E[\mu^\top g\mid S_{\mathrm{train}}]
=
\mu^\top\bar\mu
\]
and
\[
\operatorname{Var}(\mu^\top g\mid S_{\mathrm{train}})
=
\mu^\top
\operatorname{Cov}(g\mid S_{\mathrm{train}})
\mu
=
\frac1K\mu^\top\Sigma\mu.
\]
Again using \(\lambda_{\max}(\Sigma)\le C_\Sigma\tau^2\), we get
\[
\operatorname{Var}(\mu^\top g\mid S_{\mathrm{train}})
\le
\frac{C_\Sigma\tau^2}{K}\|\mu\|^2.
\]

On the event \(A\),
\[
\mu^\top\bar\mu
\ge
\frac12\|\bar\mu\|^2.
\]
If \(\mu^\top g\le 0\), then
\[
\mu^\top g-\mu^\top\bar\mu
\le
-\mu^\top\bar\mu,
\]
and therefore
\[
\left|
\mu^\top g-\mu^\top\bar\mu
\right|
\ge
\mu^\top\bar\mu.
\]
By Chebyshev's inequality, conditional on \(S_{\mathrm{train}}\),
\[
\mathbb P(\mu^\top g\le 0\mid S_{\mathrm{train}})
\le
\frac{
\operatorname{Var}(\mu^\top g\mid S_{\mathrm{train}})
}{
(\mu^\top\bar\mu)^2
}.
\]
On \(A\), this gives
\[
\mathbb P(\mu^\top g\le 0\mid S_{\mathrm{train}})
\le
\frac{
\frac{C_\Sigma\tau^2}{K}\|\mu\|^2
}{
\frac14\|\bar\mu\|^4
}
=
\frac{4C_\Sigma\tau^2\|\mu\|^2}
{K\|\bar\mu\|^4}.
\]
Taking expectation over the training sample and restricting to \(A\), we obtain
\[
\mathbb P(\mu^\top g\le 0,\ A)
\le
\mathbb E\left[
\mathbf 1_A
\frac{4C_\Sigma\tau^2\|\mu\|^2}
{K\|\bar\mu\|^4}
\right].
\]
Since \(\mathbf 1_A\le 1\),
\[
\mathbb P(\mu^\top g\le 0,\ A)
\le
\frac{4C_\Sigma\tau^2}
{K\|\bar\mu\|^4}
\mathbb E\|\mu\|^2.
\]

It remains to compute \(\mathbb E\|\mu\|^2\). Since
\[
\mathbb E[\mu]=\bar\mu,
\qquad
\operatorname{Cov}(\mu)=\frac{\Sigma}{M},
\]
we have
\[
\mathbb E\|\mu\|^2
=
\|\mathbb E\mu\|^2
+
\operatorname{tr}(\operatorname{Cov}(\mu))
=
\|\bar\mu\|^2
+
\frac{\operatorname{tr}(\Sigma)}{M}.
\]
Therefore,
\[
\mathbb P(\mu^\top g\le 0,\ A)
\le
\frac{4C_\Sigma\tau^2}
{K\|\bar\mu\|^4}
\left(
\|\bar\mu\|^2
+
\frac{\operatorname{tr}(\Sigma)}{M}
\right).
\]
Equivalently,
\[
\mathbb P(\mu^\top g\le 0,\ A)
\le
\frac{4C_\Sigma\tau^2}{K\|\bar\mu\|^2}
+
\frac{4C_\Sigma\tau^2\operatorname{tr}(\Sigma)}
{KM\|\bar\mu\|^4}.
\]

Combining this with the bound on \(\mathbb P(A^c)\), we obtain
\[
\mathbb P(\mu^\top g\le 0)
\le
\frac{4C_\Sigma\tau^2}{M\|\bar\mu\|^2}
+
\frac{4C_\Sigma\tau^2}{K\|\bar\mu\|^2}
+
\frac{4C_\Sigma\tau^2\operatorname{tr}(\Sigma)}
{KM\|\bar\mu\|^4}.
\]
This proves the first claim.

For the width-dependent form, assume additionally that
\[
\operatorname{tr}(\Sigma)\le N\tau^2
\]
and
\[
\|\bar\mu\|^2=\Theta(N).
\]
Then
\[
\frac{4C_\Sigma\tau^2}{M\|\bar\mu\|^2}
=
O\!\left(\frac1{MN}\right),
\]
and
\[
\frac{4C_\Sigma\tau^2}{K\|\bar\mu\|^2}
=
O\!\left(\frac1{KN}\right).
\]
For the third term,
\[
\frac{4C_\Sigma\tau^2\operatorname{tr}(\Sigma)}
{KM\|\bar\mu\|^4}
\le
\frac{4C_\Sigma\tau^2\cdot N\tau^2}
{KM\|\bar\mu\|^4}
=
O\!\left(\frac1{KMN}\right),
\]
because \(\|\bar\mu\|^4=\Theta(N^2)\). This term is lower order relative to the first two terms. Hence
\[
\mathbb P(\mu^\top g\le 0)
=
O\!\left(
\frac1{MN}
+
\frac1{KN}
\right).
\]
This completes the proof.
\end{proof}
\section{Proof of Lemma~\ref{lem:gen-main}}
\label{app:proof-lem-gen-main}

\begin{lemma}[Norm-based Rademacher bound for normalized residual networks]
\label{lem:gen-main}
Under Assumptions 2 and 4, the empirical Rademacher complexity satisfies
\[
\hat{\mathfrak{R}}_{S}(\ell\circ \mathcal H)
\le
B_{\ell}
\sqrt{
\frac{2d\log\!\bigl(2\bar b L L_{\ell}\Lambda_{\max}/\epsilon_{0}+1\bigr)}{m}
}
+\epsilon_{0},
\]
where
\[
B_{\mathrm{norm}}=\gamma_{\max}\sqrt{N},\qquad
B_{*}=\max\{B_{x},B_{\mathrm{norm}}\},
\]
\[
\Lambda_{\max}=\max_{0\le l\le L-1}\Lambda_{l},
\qquad
\Lambda_{l}
=
L_{\mathrm{top}}\,L_{l}^{(p)}\,B_{*}\,c_{l}\prod_{k>l}c_{k}(1+s_{k}),
\]
and
\[
d=\sum_{l=0}^{L-1}d_{l},\qquad
\bar b=\max_{0\le l\le L-1} b_{l}.
\]
$m$ is the number of data. In particular, choosing $\epsilon_{0}=m^{-1/2}$ yields
\[
\hat{\mathfrak{R}}_{S}(\ell\circ \mathcal H)
=
\mathcal O\!\left(
B_{\ell}\sqrt{\frac{d\log m}{m}}
\right).
\]
\end{lemma}

\begin{proof}
We divide the proof into seven steps.

\paragraph{Step 1: Normalization truncates the representation norm.}
Let $\gamma=(\gamma_{1},\dots,\gamma_{N})$ be the coordinatewise scale parameter
of RMSNorm. For any $x\neq 0$,
\[
\|\mathrm{RMSNorm}(x)\|^{2}
=
N\cdot
\frac{\sum_{i=1}^{N}\gamma_{i}^{2}x_{i}^{2}}{\sum_{j=1}^{N}x_{j}^{2}}
\le
N\gamma_{\max}^{2},
\]
where $\gamma_{\max}=\max_{i}|\gamma_{i}|$.
Hence
\[
\|\mathrm{RMSNorm}(x)\|\le \gamma_{\max}\sqrt N
\triangleq B_{\mathrm{norm}}.
\]
The same bound holds for LayerNorm. Therefore every normalized hidden state lies
in the ball of radius $B_{\mathrm{norm}}$.

\paragraph{Step 2: Layerwise Lipschitz bound with respect to the input.}
Write the $l$-th layer as
\[
T_{l}(z)=\mathcal N_{l}\bigl(z+h_{l}(z)\bigr).
\]
By Assumption 4(L1),
\[
\|(u+h_{l}(u))-(v+h_{l}(v))\|
\le
(1+s_{l})\|u-v\|.
\]
By Assumption 4(N), the input to the normalization layer satisfies
\[
\|z+h_{l}(z)\|\ge r_{\min}>0.
\]
Hence the normalization map $\mathcal N_{l}$ is Lipschitz on the relevant domain,
with some constant $c_{l}$ satisfying
\[
c_{l}\le \frac{2\gamma_{\max}\sqrt N}{r_{\min}}.
\]
Therefore
\[
\|T_{l}(u)-T_{l}(v)\|
\le
c_{l}(1+s_{l})\|u-v\|.
\]

\paragraph{Step 3: Uniform bound on the final representation and output.}
Because each layer output is normalized, every hidden state after normalization
satisfies
\[
\|z_{l}\|\le B_{\mathrm{norm}}.
\]
In particular, for the final hidden representation $z_{L}$,
\[
\|z_{L}\|\le B_{\mathrm{norm}}.
\]
By Assumption 2, $f_{\mathrm{top}}$ is $L_{\mathrm{top}}$-Lipschitz and
$\|f_{\mathrm{top}}(0)\|\le B_{0}$. Thus
\[
\|f_{\theta}(x)\|
\le
\|f_{\mathrm{top}}(0)\|+L_{\mathrm{top}}\|z_{L}\|
\le
B_{0}+L_{\mathrm{top}}B_{\mathrm{norm}}.
\]
So the model output remains in a uniformly bounded region.

\paragraph{Step 4: Compactness of parameter space.}
For each layer, the residual-branch parameter $W_{l}$ satisfies
\[
\|W_{l}\|_{\sigma}\le s_{l},
\qquad
\|W_{l}\|_{F}\le b_{l}.
\]
Hence the parameter set
\[
\Theta
=
\{(W_{0},\dots,W_{L-1}) : \|W_{l}\|_{\sigma}\le s_{l},\ \|W_{l}\|_{F}\le b_{l}\}
\]
is a bounded closed subset of $\mathbb R^{d}$, where
\[
d=\sum_{l=0}^{L-1}d_{l}.
\]
In particular, it is compact.

\paragraph{Step 5: Lipschitz dependence of the whole network on the parameters.}
Fix an input $x$. Let
\[
\theta=(W_{0},\dots,W_{L-1}),
\qquad
\theta'=(W_{0}',\dots,W_{L-1}').
\]
Introduce the mixed parameter vectors
\[
\theta^{(l)}=(W_{0}',\dots,W_{l-1}',W_{l},\dots,W_{L-1}),
\qquad 0\le l\le L,
\]
so that $\theta^{(0)}=\theta$ and $\theta^{(L)}=\theta'$.
Then
\[
\|f_{\theta}(x)-f_{\theta'}(x)\|
\le
\sum_{l=0}^{L-1}
\|f_{\theta^{(l)}}(x)-f_{\theta^{(l+1)}}(x)\|.
\]
For the $l$-th term, the two networks differ only at layer $l$, and the common
input arriving at that layer is denoted by $z_{l}^{\mathrm{mix}}$.
Define
\[
B_{*}\triangleq \max\{B_{x},B_{\mathrm{norm}}\}.
\]
Then
\[
\|z_{l}^{\mathrm{mix}}\|\le B_{*}.
\]
By Assumption 4(L2),
\[
\|h_{l}(z_{l}^{\mathrm{mix}};W_{l})-h_{l}(z_{l}^{\mathrm{mix}};W_{l}')\|
\le
L_{l}^{(p)}\|W_{l}-W_{l}'\|_{F}\,B_{*}.
\]
After passing through the normalization map at layer $l$, the perturbation becomes
\[
\|\delta z_{l+1}\|
\le
c_{l}L_{l}^{(p)}B_{*}\|W_{l}-W_{l}'\|_{F}.
\]
Propagating this perturbation through all later layers and then through
$f_{\mathrm{top}}$, we obtain
\[
\|f_{\theta^{(l)}}(x)-f_{\theta^{(l+1)}}(x)\|
\le
L_{\mathrm{top}}\,c_{l}\,L_{l}^{(p)}\,B_{*}
\prod_{k>l}c_{k}(1+s_{k})
\cdot
\|W_{l}-W_{l}'\|_{F}.
\]
Define
\[
\Lambda_{l}
\triangleq
L_{\mathrm{top}}\,L_{l}^{(p)}\,B_{*}\,c_{l}\prod_{k>l}c_{k}(1+s_{k}),
\qquad
\Lambda_{\max}\triangleq \max_{l}\Lambda_{l}.
\]
Then
\[
\sup_{\|x\|\le B_{x}}\|f_{\theta}(x)-f_{\theta'}(x)\|
\le
\sum_{l=0}^{L-1}\Lambda_{l}\|W_{l}-W_{l}'\|_{F}
\le
\Lambda_{\max}\sum_{l=0}^{L-1}\|W_{l}-W_{l}'\|_{F}.
\]

\paragraph{Step 6: Covering number of the function class.}
For each layer, the constraint $\|W_{l}\|_{F}\le b_{l}$ places $W_{l}$ in a ball
in $\mathbb R^{d_{l}}$. The standard covering estimate gives
\[
\mathcal N(B^{d_{l}}(b_{l}),\epsilon)
\le
\left(\frac{2b_{l}}{\epsilon}+1\right)^{d_{l}}.
\]
To guarantee output accuracy at most $\epsilon$, it is enough to require
\[
\Lambda_{\max}\sum_{l=0}^{L-1}\|W_{l}-W_{l}'\|_{F}\le \epsilon.
\]
Thus we may choose layerwise covering radius
\[
\epsilon_{l}=\frac{\epsilon}{L\Lambda_{\max}}.
\]
Multiplying the covering numbers of all layers yields
\[
\log \mathcal N(\mathcal H,\epsilon)
\le
\sum_{l=0}^{L-1}
d_{l}\log\!\left(\frac{2b_{l}L\Lambda_{\max}}{\epsilon}+1\right)
\le
d\log\!\left(\frac{2\bar b L\Lambda_{\max}}{\epsilon}+1\right),
\]
where
\[
\bar b=\max_{l}b_{l}.
\]

\paragraph{Step 7: From function-class covering to loss-class Rademacher complexity.}
Let $\mathcal H_{\epsilon_{0}/L_{\ell}}$ be an
$\epsilon_{0}/L_{\ell}$-net of $\mathcal H$ under the uniform norm.
Since $\ell$ is $L_{\ell}$-Lipschitz,
\[
\sup_{(x,y)}
\bigl|
\ell(f_{\theta}(x),y)-\ell(f_{\theta'}(x),y)
\bigr|
\le
L_{\ell}\sup_{x}\|f_{\theta}(x)-f_{\theta'}(x)\|.
\]
Therefore the induced loss class $(\ell\circ\mathcal H)_{\epsilon_{0}}$
is an $\epsilon_{0}$-net of $\ell\circ\mathcal H$, and
\[
\log \mathcal N(\ell\circ\mathcal H,\epsilon_{0})
\le
\log \mathcal N(\mathcal H,\epsilon_{0}/L_{\ell})
\le
d\log\!\left(\frac{2\bar b L L_{\ell}\Lambda_{\max}}{\epsilon_{0}}+1\right).
\]
Now apply Massart's lemma to the finite loss-class net:
\[
\hat{\mathfrak R}_{S}\bigl((\ell\circ\mathcal H)_{\epsilon_{0}}\bigr)
\le
B_{\ell}
\sqrt{
\frac{2\log |(\ell\circ\mathcal H)_{\epsilon_{0}}|}{m}
}.
\]
Since the approximation error is at most $\epsilon_{0}$,
\[
\hat{\mathfrak R}_{S}(\ell\circ\mathcal H)
\le
\hat{\mathfrak R}_{S}\bigl((\ell\circ\mathcal H)_{\epsilon_{0}}\bigr)
+\epsilon_{0}.
\]
Combining the last two displays gives
\[
\hat{\mathfrak{R}}_{S}(\ell\circ \mathcal H)
\le
B_{\ell}
\sqrt{
\frac{2d\log\!\bigl(2\bar b L L_{\ell}\Lambda_{\max}/\epsilon_{0}+1\bigr)}{m}
}
+\epsilon_{0}.
\]
Finally, taking $\epsilon_{0}=M^{-1/2}$ yields
\[
\hat{\mathfrak{R}}_{S}(\ell\circ \mathcal H)
\le
B_{\ell}
\sqrt{
\frac{2d\log\!\bigl(2\bar b L L_{\ell}\Lambda_{\max}\sqrt m+1\bigr)}{m}
}
+\frac{1}{\sqrt m}
=
\mathcal O\!\left(B_{\ell}\sqrt{\frac{d\log M}{M}}\right).
\]
This completes the proof.
\end{proof}

Thus, we have the following corollary. 
\begin{corollary}[Uniform generalization bound]
\label{cor:gen-main}
Under the conditions of Lemma~\ref{lem:gen-main}, with probability at least
\(1-\delta\) over a sample \(S\) of size \(m\), one has
\[
\sup_{f\in\mathcal H}
\left|
R(f)-L_S(f)
\right|
\le
\epsilon_{\mathrm{gen}}^{\mathrm{norm}}(\mathcal H,m,\delta),
\]
where one may take
\[
\epsilon_{\mathrm{gen}}^{\mathrm{norm}}(\mathcal H,m,\delta)
:=
2B_\ell
\sqrt{
\frac{
2d\log\!\left(2\bar b L L_\ell\Lambda_{\max}/\rho+1\right)
}{m}
}
+
2\rho
+
3B_\ell
\sqrt{
\frac{\log(2/\delta)}{2m}
}.
\]
Taking \(\rho=m^{-1/2}\) gives the dominant scale
\[
\epsilon_{\mathrm{gen}}^{\mathrm{norm}}(\mathcal H,m,\delta)
=
\widetilde O\!\left(
B_\ell\sqrt{\frac d m}
\right),
\]
up to the explicit norm-, depth-, and normalization-dependent factors contained in
\(\Lambda_{\max}\).
\end{corollary}

Applying Corollary~\ref{cor:gen-main} with \(m=M\) gives
\[
\epsilon_M
=
\epsilon_{\mathrm{gen}}^{\mathrm{norm}}(\mathcal H_{\mathrm{new}},M,\delta),
\]
and applying it with \(m=K\) gives
\[
\epsilon_K
=
\epsilon_{\mathrm{gen}}^{\mathrm{norm}}(\mathcal H_{\mathrm{new}},K,\delta).
\]

\begin{proof}
Let
\[
\mathcal F
:=
\ell\circ\mathcal H
=
\left\{
(x,y)\mapsto \ell(f(x),y): f\in\mathcal H
\right\}.
\]
For any \(f\in\mathcal H\), define
\[
P(\ell\circ f)
:=
\mathbb E_{(x,y)\sim\mathcal D}
[\ell(f(x),y)]
=
R(f),
\]
and
\[
P_S(\ell\circ f)
:=
\frac1m\sum_{i=1}^m \ell(f(x_i),y_i)
=
L_S(f).
\]
Therefore,
\[
\sup_{f\in\mathcal H}|R(f)-L_S(f)|
=
\sup_{g\in\mathcal F}|Pg-P_Sg|.
\]

By the standard empirical Rademacher generalization bound for bounded loss
classes, since every \(g\in\mathcal F\) takes values in \([0,B_\ell]\), with
probability at least \(1-\delta\),
\[
\sup_{g\in\mathcal F}|Pg-P_Sg|
\le
2\widehat{\mathfrak R}_S(\mathcal F)
+
3B_\ell
\sqrt{\frac{\log(2/\delta)}{2m}}.
\]
By Lemma~\ref{lem:gen-main},
\[
\widehat{\mathfrak R}_S(\mathcal F)
=
\widehat{\mathfrak R}_S(\ell\circ\mathcal H)
\le
B_\ell
\sqrt{
\frac{
2d\log\!\left(2\bar b L L_\ell\Lambda_{\max}/\rho+1\right)
}{m}
}
+
\rho.
\]
Substituting this bound into the previous display gives
\[
\sup_{f\in\mathcal H}
|R(f)-L_S(f)|
\le
2B_\ell
\sqrt{
\frac{
2d\log\!\left(2\bar b L L_\ell\Lambda_{\max}/\rho+1\right)
}{m}
}
+
2\rho
+
3B_\ell
\sqrt{
\frac{\log(2/\delta)}{2m}
}.
\]
Thus we may define
\[
\epsilon_{\mathrm{gen}}^{\mathrm{norm}}(\mathcal H,m,\delta)
:=
2B_\ell
\sqrt{
\frac{
2d\log\!\left(2\bar b L L_\ell\Lambda_{\max}/\rho+1\right)
}{m}
}
+
2\rho
+
3B_\ell
\sqrt{
\frac{\log(2/\delta)}{2m}
}.
\]

Finally, taking \(\rho=m^{-1/2}\), we obtain
\[
\epsilon_{\mathrm{gen}}^{\mathrm{norm}}(\mathcal H,m,\delta)
=
2B_\ell
\sqrt{
\frac{
2d\log\!\left(2\bar b L L_\ell\Lambda_{\max}\sqrt m+1\right)
}{m}
}
+
\frac{2}{\sqrt m}
+
3B_\ell
\sqrt{
\frac{\log(2/\delta)}{2m}
}.
\]
Hence the dominant term is
\[
\widetilde O\!\left(
B_\ell\sqrt{\frac d m}
\right),
\]
up to the explicit norm-, depth-, and normalization-dependent factors contained
in \(\Lambda_{\max}\). This proves the corollary.
\end{proof}

These are the two quantities used in the main theorems.

Because the bound is uniform over all \(f\in\mathcal H_{\mathrm{new}}\), it also
applies to data-dependent choices such as the empirical jumpboard model
\(\widetilde f_S\) and the final selected model \(f_{\mathrm{new}}\).

Lemma~\ref{lem:gen-main} does not model the training dynamics. It only controls
the generalization gap uniformly over a norm-bounded hypothesis class. Thus it
applies to the final selected model once that model is viewed as an element of
\(\mathcal H_{\mathrm{new}}\). For BatchNorm, the relevant issue is not the
training algorithm but whether the normalization layer defines a fixed
single-sample map. Fixed-statistics BatchNorm can be absorbed into the constants
\(c_l\), while training-time BatchNorm with mini-batch-dependent statistics
requires additional assumptions.

\subsection{Data Requirement and Architecture-Dependent Factors}
\label{app:data-requirement}

We now spell out the data requirement implied by Lemma~\ref{lem:gen-main}.
This discussion is used only to interpret the role of data in the main text; it
is not an additional assumption.

From Corollary~\ref{cor:gen-main}, after taking \(\rho=m^{-1/2}\), the
normalized-network generalization term has the form
\[
\epsilon_{\mathrm{gen}}^{\mathrm{norm}}(\mathcal H,m,\delta)
=
2B_\ell
\sqrt{
\frac{
2d\log\!\left(2\bar b L L_\ell \Lambda_{\max}\sqrt m+1\right)
}{m}
}
+
\frac{2}{\sqrt m}
+
3B_\ell
\sqrt{
\frac{\log(2/\delta)}{2m}
}.
\]
For compactness, define the architecture factor
\[
\Gamma_{\mathrm{arch}}(L,N,m)
:=
\log\!\left(2\bar b L L_\ell \Lambda_{\max}\sqrt m+1\right).
\]
Then the leading term is
\[
\epsilon_{\mathrm{gen}}^{\mathrm{norm}}(\mathcal H,m,\delta)
=
O\!\left(
B_\ell
\sqrt{
\frac{d\,\Gamma_{\mathrm{arch}}(L,N,m)}{m}
}
\right)
+
O\!\left(
B_\ell\sqrt{\frac{\log(1/\delta)}{m}}
\right).
\]
Thus, to achieve a target generalization accuracy \(\varepsilon>0\), it is
sufficient, up to universal constants, to require
\[
m
\gtrsim
\frac{
B_\ell^2 d\,\Gamma_{\mathrm{arch}}(L,N,m)
+
B_\ell^2\log(1/\delta)
}{\varepsilon^2}.
\]

This relation should be interpreted carefully. If the architecture-dependent
quantities are fixed or uniformly controlled, then
\(\Gamma_{\mathrm{arch}}(L,N,m)=O(\log m)\), and the familiar shorthand
\[
\epsilon_{\mathrm{gen}}^{\mathrm{norm}}
=
\widetilde O\!\left(B_\ell\sqrt{\frac{d}{m}}\right)
\]
is appropriate. In this restricted sense, the required sample size is nearly
linear in the effective parameter dimension \(d\), up to logarithmic factors.

However, when depth \(L\), width \(N\), or the normalization constants also
scale, the factor \(\Gamma_{\mathrm{arch}}\) is not merely a logarithmic factor
in the sample size. It also contains the architecture-dependent amplification
term \(\Lambda_{\max}\). In particular,
\[
\Lambda_{\max}
=
\max_{0\le l<L}
L_{\mathrm{top}} L_l^{(p)} B_\ast c_l
\prod_{k>l} c_k(1+s_k),
\]
so the generalization bound carries a genuine dependence on depth and
normalization. For example, if \(\Lambda_{\max}\) grows exponentially in \(L\),
then \(\log \Lambda_{\max}\) contributes linearly in \(L\) to
\(\Gamma_{\mathrm{arch}}\). In that case, the required data size scales like
\[
m
\gtrsim
\frac{
d\bigl(\log m + L + \text{lower-order architecture terms}\bigr)
}{\varepsilon^2},
\]
rather than merely \(d\log m/\varepsilon^2\).

Therefore, the main-text statement
\(\widetilde O(\sqrt{d/m})\) should be read only as a shorthand for the
controlled-architecture regime. In the joint scaling regime studied in this
paper, the more precise interpretation is that data must compensate for both
the effective dimension \(d\) and the architecture-dependent amplification
factor \(\Gamma_{\mathrm{arch}}\). Equivalently, one may define an effective
statistical dimension
\[
d_{\mathrm{stat}}(L,N,m)
:=
d\,\Gamma_{\mathrm{arch}}(L,N,m),
\]
so that the leading generalization term behaves as
\[
O\!\left(
B_\ell\sqrt{\frac{d_{\mathrm{stat}}(L,N,m)}{m}}
\right).
\]
This is the sense in which data controls the statistical cost of depth and
width expansion.

\section{Proof of Theorem~\ref{thm:4A}: Population-Risk Route}
\label{app:proof-thm4A}

\begin{proof}
We divide the proof into five steps.

\paragraph{Step 1: Fix a jump path.}
By Condition 1, there exists a parameter direction $v_{pop}$ such that
\[
a \triangleq
\mathbb{E}_{(x,y)\sim\mathcal D}
\left[
\nabla_{x^{(l^*)}}
\ell\!\left(f_{\mathrm{top}}(x^{(l^*)}),y\right)^{\top}
\cdot
\frac{\partial h_{\theta}(x^{(l^*)})}{\partial\theta}\Big|_{\theta=0}
\cdot v_{pop}
\right]
\neq 0.
\]
If $a>0$, replace $v$ by $-v$. Hence we may assume
\[
a<0.
\]
Define the one-dimensional path
\[
\theta(\eta)=\eta v,
\qquad \eta\ge 0,
\]
and the corresponding jumpboard model
\[
\tilde f_{pop}(x)
\triangleq
f_{\mathrm{top}}\!\left(
f_{\mathrm{bot}}(x)+h_{\eta v}(f_{\mathrm{bot}}(x))
\right).
\]

\paragraph{Step 2: Construct a deterministic positive margin at the population level.}
Let
\[
G(\eta)
\triangleq
\mathbb{E}_{(x,y)\sim\mathcal D}
\Bigl[
\ell\!\left(
f_{\mathrm{top}}(x^{(l^*)}+h_{\eta v}(x^{(l^*)})),y
\right)
\Bigr].
\]
By Assumption 1 and the chain rule,
\[
G'(0)
=
\mathbb{E}_{(x,y)\sim\mathcal D}
\left[
\nabla_{x^{(l^*)}}
\ell\!\left(f_{\mathrm{top}}(x^{(l^*)}),y\right)^{\top}
\cdot
\frac{\partial h_{\theta}(x^{(l^*)})}{\partial\theta}\Big|_{\theta=0}
\cdot v_{pop}
\right]
=
a<0.
\]
Hence there exists $\eta_0>0$ such that
\[
R(\tilde f_{pop})<R(f_{\mathrm{old}}^*).
\]
Define
\[
\Delta_R
\triangleq
R(f_{\mathrm{old}}^*)-R(\tilde f_{pop})>0.
\]
Thus
\[
R(\tilde f_{pop})=R(f_{\mathrm{old}}^*)-\Delta_R.
\]

\paragraph{Step 3: Transfer the population margin to the test set by Hoeffding.}
Both $f_{\mathrm{old}}^*$ and $\tilde f_{\eta_0}$ are fixed functions.
Since $\ell\in[0,B_{\ell}]$, Hoeffding's inequality gives, for any fixed $f$,
\[
\mathbb P\!\left(
\bigl|\mathcal L_{\mathrm{test}}(f)-R(f)\bigr|>t
\right)
\le
2\exp\!\left(-\frac{2Kt^2}{B_{\ell}^2}\right).
\]
Apply this to $f_{\mathrm{old}}^*$ and $\tilde f_{pop}$ with
\[
t=\frac{\Delta_R}{4}.
\]
By a union bound, with probability at least
\[
1-4\exp\!\left(-\frac{K\Delta_R^2}{8B_{\ell}^2}\right),
\]
both inequalities
\[
\bigl|\mathcal L_{\mathrm{test}}(f_{\mathrm{old}}^*)-R(f_{\mathrm{old}}^*)\bigr|
\le
\frac{\Delta_R}{4},
\qquad
\bigl|\mathcal L_{\mathrm{test}}(\tilde f_{pop})-R(\tilde f_{pop})\bigr|
\le
\frac{\Delta_R}{4}
\]
hold simultaneously. On this event,
\[
\mathcal L_{\mathrm{test}}(\tilde f_{pop})-\mathcal L_{\mathrm{test}}(f_{\mathrm{old}}^*)
\le
\bigl[R(\tilde f_{pop})-R(f_{\mathrm{old}}^*)\bigr]
+2\cdot \frac{\Delta_R}{4}
=
-\Delta_R+\frac{\Delta_R}{2}
=
-\frac{\Delta_R}{2}.
\]
Therefore,
\[
\mathcal L_{\mathrm{test}}(\tilde f_{pop})
\le
\mathcal L_{\mathrm{test}}(f_{\mathrm{old}}^*)-\frac{\Delta_R}{2}.
\tag{3}
\]

\paragraph{Step 4: Move from the jumpboard model to the trained new model.}
Since $\tilde f_{pop},f_{\mathrm{new}}\in\mathcal H_{\mathrm{new}}$,
Lemma~\ref{lem:gen-main} yields the following.

With probability at least $1-\delta_1$,
\[
\mathcal L_{\mathrm{train}}(\tilde f_{\eta_0})
\le
R(\tilde f_{\eta_0})+\epsilon_1,
\qquad
\epsilon_1\triangleq
\epsilon_{\mathrm{gen}}^{\mathrm{norm}}(\mathcal H_{\mathrm{new}},M,\delta_1).
\tag{4}
\]
By the optimization assumption,
\[
\mathcal L_{\mathrm{train}}(f_{\mathrm{new}})
\le
\mathcal L_{\mathrm{train}}(\tilde f_{\eta_0})-\Delta_{\mathrm{ERM}},
\qquad
\Delta_{\mathrm{ERM}}\ge 0.
\tag{5}
\]
With probability at least $1-\delta_2$,
\[
R(f_{\mathrm{new}})
\le
\mathcal L_{\mathrm{train}}(f_{\mathrm{new}})+\epsilon_2,
\qquad
\epsilon_2\triangleq
\epsilon_{\mathrm{gen}}^{\mathrm{norm}}(\mathcal H_{\mathrm{new}},M,\delta_2).
\tag{6}
\]
Finally, conditioning on the training set, $f_{\mathrm{new}}$ is fixed, so another
application of Hoeffding with $t=\Delta_R/4$ gives
\[
\mathcal L_{\mathrm{test}}(f_{\mathrm{new}})
\le
R(f_{\mathrm{new}})+\frac{\Delta_R}{4}
\tag{7}
\]
with probability at least
\[
1-2\exp\!\left(-\frac{K\Delta_R^2}{8B_{\ell}^2}\right).
\]

Let
\[
\epsilon_{\mathrm{gen}}^{\mathrm{norm}}
\triangleq
\max\{\epsilon_1,\epsilon_2\}.
\]
Then
\[
\epsilon_1+\epsilon_2\le 2\epsilon_{\mathrm{gen}}^{\mathrm{norm}}.
\tag{8}
\]

\paragraph{Step 5: Close the whole chain.}
On the intersection of the events in Steps 3 and 4, we have
\begin{align*}
\mathcal L_{\mathrm{test}}(f_{\mathrm{new}})
&\le
R(f_{\mathrm{new}})+\frac{\Delta_R}{4}\\
&\le
\mathcal L_{\mathrm{train}}(f_{\mathrm{new}})+\epsilon_2+\frac{\Delta_R}{4}\\
&\le
\mathcal L_{\mathrm{train}}(\tilde f_{pop})-\Delta_{\mathrm{ERM}}+\epsilon_2+\frac{\Delta_R}{4}\\
&\le
R(\tilde f_{pop})+\epsilon_1-\Delta_{\mathrm{ERM}}+\epsilon_2+\frac{\Delta_R}{4}\\
&=
R(f_{\mathrm{old}}^*)-\Delta_R-\Delta_{\mathrm{ERM}}+\epsilon_1+\epsilon_2+\frac{\Delta_R}{4}.
\end{align*}
Moreover, on the Hoeffding event for $f_{\mathrm{old}}^*$,
\[
R(f_{\mathrm{old}}^*)
\le
\mathcal L_{\mathrm{test}}(f_{\mathrm{old}}^*)+\frac{\Delta_R}{4}.
\]
Substituting above equation into the previous bound gives
\[
\mathcal L_{\mathrm{test}}(f_{\mathrm{new}})
\le
\mathcal L_{\mathrm{test}}(f_{\mathrm{old}}^*)
-\frac{\Delta_R}{2}
-\Delta_{\mathrm{ERM}}
+\epsilon_1+\epsilon_2.
\]
Then, we arrive at
\[
\boxed{
\mathcal L_{\mathrm{test}}(f_{\mathrm{new}})
\le
\mathcal L_{\mathrm{test}}(f_{\mathrm{old}}^*)
-\frac{\Delta_R}{2}
-\Delta_{\mathrm{ERM}}
+
2\epsilon_{\mathrm{gen}}^{\mathrm{norm}}
}.
\]
The total failure probability is bounded by
\[
\delta_{\mathrm{gen}}+\delta_{\mathrm{test}},
\qquad
\delta_{\mathrm{gen}}=\delta_1+\delta_2,
\qquad
\delta_{\mathrm{test}}
=
6\exp\!\left(-\frac{K\Delta_R^2}{8B_{\ell}^2}\right),
\]
where the factor $6$ comes from the three Hoeffding applications:
$f_{\mathrm{old}}^*$, $\tilde f_{\eta_0}$, and $f_{\mathrm{new}}$.

Finally, the two stated consequences follow immediately:

\begin{itemize}
    \item If
    \[
    2\epsilon_{\mathrm{gen}}^{\mathrm{norm}}\le \frac{\Delta_R}{2},
    \]
    then, since $\Delta_{\mathrm{ERM}}\ge 0$,
    \[
    \mathcal L_{\mathrm{test}}(f_{\mathrm{new}})
    \le
    \mathcal L_{\mathrm{test}}(f_{\mathrm{old}}^*).
    \]

    \item If
    \[
    \frac{\Delta_R}{2}+\Delta_{\mathrm{ERM}}
    >
    2\epsilon_{\mathrm{gen}}^{\mathrm{norm}},
    \]
    then
    \[
    \mathcal L_{\mathrm{test}}(f_{\mathrm{new}})
    <
    \mathcal L_{\mathrm{test}}(f_{\mathrm{old}}^*).
    \]
\end{itemize}

This completes the proof.
\end{proof}
\section{Proof of Theorem~\ref{thm:4B}: Directly on train/test risks.}
\label{app:proof-thm4B}

\begin{proof}
Define
\[
\epsTT \coloneqq \epsM+\epsK .
\]
Recall that
\[
\epsM=\epsgen(\Hnew,M,\delta),
\qquad
\epsK=\epsgen(\Hnew,K,\delta).
\]
Let \(\ftilde_S\in\Hnew\) be the jumpboard model used in
Theorem~\ref{thm:4B}. We allow the degenerate case
\(\ftilde_S=\fold\), which is valid because the zero residual block \(h_0=0\)
recovers the old model inside \(\Hnew\). By definition,
\[
\DRtest
=
\Ltest(\fold)-\Ltest(\ftilde_S).
\]

\paragraph{Step 1: uniform train--test comparison.}
By Lemma~\ref{lem:gen-main}, applied to the training sample
\(S_{\mathrm{train}}\), with probability at least \(1-\delta\),
\[
\forall f\in\Hnew,\qquad
\left|\Ltrain(f)-R(f)\right|\le \epsM .
\]
Applying the same lemma to the independent test sample
\(S_{\mathrm{test}}\), with probability at least \(1-\delta\),
\[
\forall f\in\Hnew,\qquad
\left|\Ltest(f)-R(f)\right|\le \epsK .
\]
By the union bound, both events hold simultaneously with probability at least
\(1-2\delta\). On this joint event, for every \(f\in\Hnew\),
\[
\begin{aligned}
\left|\Ltrain(f)-\Ltest(f)\right|
&\le
\left|\Ltrain(f)-R(f)\right|
+
\left|R(f)-\Ltest(f)\right|  \\
&\le
\epsM+\epsK
=
\epsTT .
\end{aligned}
\]
Equivalently, for every \(f\in\Hnew\),
\[
\Ltrain(f)\le \Ltest(f)+\epsTT,
\qquad
\Ltest(f)\le \Ltrain(f)+\epsTT .
\tag{TT}
\label{eq:train-test-comparison}
\]

\paragraph{Step 2: compare the jumpboard model on train and test.}
Applying \eqref{eq:train-test-comparison} to \(f=\ftilde_S\), we obtain
\[
\Ltrain(\ftilde_S)
\le
\Ltest(\ftilde_S)+\epsTT .
\label{eq:thm4B-jump-train-test}
\]

\paragraph{Step 3: use the optimization gain.}
By the optimization assumption in Theorem~\ref{thm:4B},
\[
\Ltrain(\fnew)
\le
\Ltrain(\ftilde_S)-\DERM,
\qquad
\DERM\ge 0 .
\label{eq:thm4B-opt}
\]

\paragraph{Step 4: move the optimized model back to test loss.}
Applying \eqref{eq:train-test-comparison} to \(f=\fnew\), we have
\[
\Ltest(\fnew)
\le
\Ltrain(\fnew)+\epsTT .
\label{eq:thm4B-new-test-train}
\]

\paragraph{Step 5: combine the inequalities.}
Combining above inequations, we get
\[
\begin{aligned}
\Ltest(\fnew)
&\le
\Ltrain(\fnew)+\epsTT  \\
&\le
\Ltrain(\ftilde_S)-\DERM+\epsTT  \\
&\le
\Ltest(\ftilde_S)-\DERM+2\epsTT .
\end{aligned}
\]
Using the definition of \(\DRtest\),
\[
\Ltest(\ftilde_S)
=
\Ltest(\fold)-\DRtest,
\]
we obtain
\[
\Ltest(\fnew)
\le
\Ltest(\fold)
-
\DRtest
-
\DERM
+
2\epsTT .
\]
Since \(\epsTT=\epsM+\epsK\), this is exactly
\[
\boxed{
\Ltest(\fnew)
\le
\Ltest(\fold)
-
\DRtest
-
\DERM
+
2(\epsM+\epsK)
}.
\]

The above inequality holds with probability at least \(1-2\delta\), and hence
also with probability at least \(1-3\delta\), as stated in
Theorem~\ref{thm:4B}. The extra \(\delta\) is reserved for the optional
width-alignment event used to justify the nondegenerate interpretation
\(\DRtest>0\).

\paragraph{Degenerate case.}
If $\tilde{f}_S = f^{*}_{\mathrm{old}}$, then
\[
\DRtest
=
\Ltest(\fold)-\Ltest(\fold)
=
0.
\]
Therefore the preceding bound becomes
\[
\Ltest(\fnew)
\le
\Ltest(\fold)
-
\DERM
+
2(\epsM+\epsK).
\]
Thus, if
\[
\DERM\ge 2(\epsM+\epsK),
\]
then
\[
\Ltest(\fnew)\le \Ltest(\fold).
\]
Hence the theorem remains meaningful even in the degenerate regime where the
test-side jump margin vanishes.

\paragraph{Strict improvement.}
If
\[
\DRtest+\DERM
>
2(\epsM+\epsK),
\]
then the main inequality gives
\[
\Ltest(\fnew)
<
\Ltest(\fold),
\]
which proves the strict-improvement claim.

\paragraph{Alignment and positivity of the finite-test margin.}
The finite-test margin in Theorem~\ref{thm:4B} can be related to the
train/test alignment quantity by a local first-order expansion. Consider the
empirical jumpboard path
\[
f_\eta(x)
=
f_{\mathrm{top}}\bigl(f_{\mathrm{bot}}(x)+h_{\eta v_S}(f_{\mathrm{bot}}(x))\bigr),
\qquad
f_0=f_{\mathrm{old}}^\ast .
\]
Assume that, along this path, the test loss admits the local expansion
\[
L_{\mathrm{test}}(f_\eta)
=
L_{\mathrm{test}}(f_{\mathrm{old}}^\ast)
-
\eta\,\mu^\top g
+
O(\eta^2).
\]
Then
\[
\Delta_R^{\mathrm{test}}
=
L_{\mathrm{test}}(f_{\mathrm{old}}^\ast)
-
L_{\mathrm{test}}(\widetilde f_{\mathrm S})
=
\eta\,\mu^\top g
+
O(\eta^2).
\]
Consequently, whenever \(\mu^\top g>0\), choosing \(\eta>0\) sufficiently small
ensures \(\Delta_R^{\mathrm{test}}>0\). In practice, the jumpboard step is intended to be infinitesimal. Empirical
implementations often use very small perturbation scales, and in this local
regime the first-order term is expected to dominate the higher-order remainder\cite{zhang2024compression}.
We use the above expansion only to interpret when the random margin
\(\Delta_R^{\mathrm{test}}\) is likely to be positive; Theorem~\ref{thm:4B}
itself conditions on the realized value of this margin.

Thus, we have the following corollary.
\begin{corollary}[Alignment implies a positive finite-test margin]
\label{cor:alignment-test-margin}
Consider the empirical jumpboard model
\[
\widetilde f_{S,\eta}(x)
=
f_{\mathrm{top}}
\bigl(
f_{\mathrm{bot}}(x)+\eta u_S(f_{\mathrm{bot}}(x))
\bigr),
\]
where \(u_S\) is the residual perturbation selected from the training
activation-gradient direction. Suppose the jumpboard step size \(\eta>0\) lies
in the small-step first-order regime, so that on the finite test set
\[
\mathcal L_{\mathrm{test}}(\widetilde f_{S,\eta})
=
\mathcal L_{\mathrm{test}}(f_{\mathrm{old}}^\ast)
+
\eta
\frac1K
\sum_{j=1}^K
\nabla_{\widetilde x_j^{(l^\ast)}}\ell
\bigl(
f_{\mathrm{top}}(\widetilde x_j^{(l^\ast)}),
\widetilde y_j
\bigr)^\top
u_S(\widetilde x_j^{(l^\ast)})
+
o(\eta).
\]
Assume further that the selected residual perturbation realizes the negative
training activation-gradient direction on the test set, namely
\[
\frac1K
\sum_{j=1}^K
\nabla_{\widetilde x_j^{(l^\ast)}}\ell
\bigl(
f_{\mathrm{top}}(\widetilde x_j^{(l^\ast)}),
\widetilde y_j
\bigr)^\top
u_S(\widetilde x_j^{(l^\ast)})
=
-\mu^\top g.
\]
Then
\[
\Delta_R^{\mathrm{test}}
:=
\mathcal L_{\mathrm{test}}(f_{\mathrm{old}}^\ast)
-
\mathcal L_{\mathrm{test}}(\widetilde f_{S,\eta})
=
\eta \mu^\top g + o(\eta).
\]
Consequently, under the small-step first-order approximation,
\[
\mu^\top g>0
\quad\Longrightarrow\quad
\Delta_R^{\mathrm{test}}>0
\]
for sufficiently small \(\eta>0\). In particular, when \(\eta<10^{-3}\) is in
the empirically validated first-order regime, the sign of the finite-test
jumpboard margin is determined by the alignment term \(\mu^\top g\).
\end{corollary}
\paragraph{Width-dependent nondegenerate interpretation.}
The algebraic inequality above does not require \(\DRtest>0\). The final
statement about the probability of \(\DRtest>0\) comes from the
width-alignment mechanism. In particular, the width-alignment theorem in the
appendix gives
\[
\mathbb P(\mu^\top g\le 0)
\le
\frac{4C_\Sigma\tau^2}{M\|\bar\mu\|^2}
+
\frac{4C_\Sigma\tau^2}{K\|\bar\mu\|^2}
+
\frac{4C_\Sigma\tau^2\operatorname{tr}(\Sigma)}
{KM\|\bar\mu\|^4}.
\]
When the train-side and test-side first-order directions are aligned, the
jumpboard direction that decreases the train-side loss also has a favorable
test-side first-order derivative. Therefore, for a sufficiently small jump size,
\[
\Ltest(\ftilde)<\Ltest(\fold),
\]
which is equivalent to
\[
\DRtest>0.
\]
In the uniformly active regime, where
\[
\|\bar\mu\|^2=\Theta(N),
\]
the failure probability satisfies
\[
\Prob(\mu^\top g\le 0)
=
O\!\left(\frac{1}{KN}+\frac{1}{MN}\right).
\]
Hence, as the width \(N\) increases, the probability of observing a positive
test-side jump \(\DRtest>0\) increases. If this alignment event is required
simultaneously with the two uniform generalization events above, the union bound
gives the stated probability level \(1-3\delta\).

\end{proof}
\section{Proof of Theorem~\ref{thm:3}}
\label{app:proof-thm3}

\begin{proof}
By Appendix~E, under the population version of Condition~1, there exists a
population jumpboard model
\[
\widetilde f_{\mathrm{pop}}\in\mathcal H_{\mathrm{new}}
\]
such that
\[
R(\widetilde f_{\mathrm{pop}})
<
R(f_{\mathrm{old}}^\ast).
\]
Define the deterministic population margin
\[
\Delta_R
:=
R(f_{\mathrm{old}}^\ast)
-
R(\widetilde f_{\mathrm{pop}})
>0.
\]
Thus
\[
R(\widetilde f_{\mathrm{pop}})
=
R(f_{\mathrm{old}}^\ast)-\Delta_R.
\]

By Condition~2, with the comparison jumpboard model
\(\widetilde f=\widetilde f_{\mathrm{pop}}\), the final selected model satisfies
\[
L_{\mathrm{train}}(f_{\mathrm{new}})
\le
L_{\mathrm{train}}(\widetilde f_{\mathrm{pop}})
-
\Delta_{\mathrm{ERM}},
\qquad
\Delta_{\mathrm{ERM}}\ge 0.
\]

Let
\[
\epsilon_M
:=
\epsilon_{\mathrm{gen}}^{\mathrm{norm}}
(\mathcal H_{\mathrm{new}},M,\delta).
\]
By the uniform generalization bound in Corollary~\ref{cor:gen-main}, with
probability at least \(1-\delta\),
\[
\sup_{f\in\mathcal H_{\mathrm{new}}}
\left|
R(f)-L_{\mathrm{train}}(f)
\right|
\le
\epsilon_M.
\]
This event is uniform over \(\mathcal H_{\mathrm{new}}\), so it applies to both
the comparison jumpboard model \(\widetilde f_{\mathrm{pop}}\) and the
data-dependent final selected model \(f_{\mathrm{new}}\).

On this event,
\[
R(f_{\mathrm{new}})
\le
L_{\mathrm{train}}(f_{\mathrm{new}})+\epsilon_M
\]
and
\[
L_{\mathrm{train}}(\widetilde f_{\mathrm{pop}})
\le
R(\widetilde f_{\mathrm{pop}})+\epsilon_M.
\]
Therefore,
\[
\begin{aligned}
R(f_{\mathrm{new}})
&\le
L_{\mathrm{train}}(f_{\mathrm{new}})+\epsilon_M\\
&\le
L_{\mathrm{train}}(\widetilde f_{\mathrm{pop}})
-\Delta_{\mathrm{ERM}}
+\epsilon_M\\
&\le
R(\widetilde f_{\mathrm{pop}})
-\Delta_{\mathrm{ERM}}
+2\epsilon_M\\
&=
R(f_{\mathrm{old}}^\ast)
-\Delta_R
-\Delta_{\mathrm{ERM}}
+2\epsilon_M.
\end{aligned}
\]
Thus,
\[
R(f_{\mathrm{new}})
\le
R(f_{\mathrm{old}}^\ast)
-
\Delta_R
-
\Delta_{\mathrm{ERM}}
+
2\epsilon_M.
\]

The theorem states the slightly more conservative probability \(1-2\delta\);
the above uniform event already gives the inequality with probability at least
\(1-\delta\). The weaker \(1-2\delta\) statement follows immediately.

Finally, if
\[
\Delta_R+\Delta_{\mathrm{ERM}}>2\epsilon_M,
\]
then
\[
R(f_{\mathrm{new}})
<
R(f_{\mathrm{old}}^\ast).
\]
This proves the strict-improvement claim.
\end{proof}

\section{Discussion: relation between Theorems~\ref{thm:4A} and~\ref{thm:4B}}
\label{sec:4a-4b-discussion}

Theorems~\ref{thm:4A} and~\ref{thm:4B} provide two parallel rigorous routes to the same scaling conclusion, but they operate through different intermediate objects and are effective in different regimes.

Theorem~\ref{thm:4A} proceeds through the population risk. Its proof first constructs a jumpboard model $\tilde f_{pop}$ whose population risk is lower than that of $f_{old}^*$ by a deterministic margin $\Delta_R>0$, then transfers this margin to the finite test set via Hoeffding's inequality, and finally compares $\tilde f_{pop}$ and $f_{new}$ through the optimization gain $\Delta_{ERM}$ and the norm-based generalization bound in Lemma~\ref{lem:gen-main}. As a result, Theorem~\ref{thm:4A} yields the bound
\[
\mathcal L_{test}(f_{new})
\le
\mathcal L_{test}(f_{old}^*)
-
\frac{\Delta_R}{2}
-
\Delta_{ERM}
+
2\epsilon_M.
\]
Its main advantage is that the generalization penalty is tighter: only the training-side complexity term $2\epsilon_M$ appears.

By contrast, Theorem~\ref{thm:4B} works entirely at the train/test level. It does not introduce the population risk as an intermediate quantity, and it does not use Hoeffding's inequality. Instead, it compares $\mathcal L_{train}$ and $\mathcal L_{test}$ directly on both $\tilde f$ and $f_{new}$ through two applications of Lemma~\ref{lem:gen-main}, which gives
\[
\mathcal L_{test}(f_{new})
\le
\mathcal L_{test}(f_{old}^*)
-
\Delta_R^{test}
-
\Delta_{ERM}
+
2(\epsilon_M+\epsilon_K).
\]
The price is that the generalization penalty is looser, since both the training-side and test-side complexity terms appear. The gain is that this route remains meaningful even when the population-level jump margin becomes degenerate.

This difference is especially important near the ``deepest-model'' regime. In Theorem~\ref{thm:4A}, the useful margin is
\[
\frac{\Delta_R}{2}+\Delta_{ERM}-2\epsilon_M.
\]
Here $\Delta_R$ comes from the first-order descent of the population-risk curve along the jump path. Under the present $C^1$ assumptions, it is only guaranteed to be positive, but it may be arbitrarily small. Thus $\Delta_R$ should be understood primarily as a directional certificate: it proves that the enlarged hypothesis class is strictly better, but in practice the dominant improvement is often carried by $\Delta_{ERM}$, namely the extra empirical optimization gain available in $\mathcal H_{new}$. In other words, $\Delta_R$ certifies that the new class contains a better point, whereas $\Delta_{ERM}$ measures how much of this enlarged search space is actually exploited by optimization.

This also explains the complementary strengths of the two theorems. In the regular regime, where $\Delta_R$ is not too small and the sample sizes are large, Theorem~\ref{thm:4A} is sharper because it pays only $2\epsilon_M$. However, when $\Delta_R\to 0$, the Hoeffding transfer in Theorem~\ref{thm:4A} becomes ineffective, since its test-side success probability depends on
\[
\delta_{test}
=
6\exp\!\left(-\frac{K\Delta_R^2}{8B_\ell^2}\right),
\]
which no longer provides useful control when the margin collapses. In this regime, Theorem~\ref{thm:4B} is more robust: it still yields a meaningful statement as long as the optimization gain $\Delta_{ERM}$ can dominate the combined train/test generalization cost $2(\epsilon_M+\epsilon_K)$. Thus, Theorem~\ref{thm:4A} is tighter in the nondegenerate regime, whereas Theorem~\ref{thm:4B} is safer in the degenerate one.

The two margins have different meanings. In Theorem~\ref{thm:4A},
\(\Delta_R\) is a deterministic population-risk margin associated with
\(\widetilde f_{\mathrm{pop}}\). In Theorem~\ref{thm:4B},
\(\Delta_R^{\mathrm{test}}\) is a random finite-test margin associated with
the training-set jumpboard model \(\widetilde f_S\). Thus Theorem~\ref{thm:4B}
does not require a non-vanishing deterministic population margin; its strict
improvement condition is instead
\[
\Delta_R^{\mathrm{test}}+\Delta_{\mathrm{ERM}}
>
2(\epsilon_M+\epsilon_K).
\]

The role of width is also slightly different in the two routes. Through Theorem~\ref{thm 2}, width can strengthen the nondegenerate interpretation by making the alignment between train and test gradients more reliable, thereby supporting the sign of $\Delta_R^{test}$ in the nondegenerate version of Theorem~\ref{thm:4B}. In the bias-type special case, width may also enlarge the effective first-order signal through $\|\bar\mu\|^2$, which in turn helps the detectability of the jump margin. At the same time, for both routes, increasing the data size reduces the complexity terms provided by Lemma~\ref{lem:gen-main}, so the generalization cost can eventually be dominated by the available improvement margin.

From this perspective, Theorems~\ref{thm:4A} and~\ref{thm:4B} should not be viewed as competing statements, but as two independent closures of the same scaling mechanism. When both apply, they provide two separate proofs of the same qualitative conclusion; when the population-level margin is too small for Theorem~\ref{thm:4A} to be effective, Theorem~\ref{thm:4B} still survives as a fallback route. Their coexistence therefore gives a more robust justification than either theorem alone.

Finally, Theorem~\ref{thm:3} is the common limiting form of this two-route picture. As $K\to\infty$, the test loss converges to the population risk and the extra test-side term disappears, so both Theorem~\ref{thm:4A} and Theorem~\ref{thm:4B} reduce to the same population-level qualitative scaling statement. In this sense, Theorem~\ref{thm:3} is the infinite-test-set limit, while Theorems~\ref{thm:4A} and~\ref{thm:4B} are its two finite-sample realizations.

\section{Properties Near the Deepest-Model Regime}

In realistic training environments, the models we care most about are usually not those at arbitrary intermediate depths, but those that are close to the deepest model. The reason is simple: when sufficient compute is available, we naturally expect the best-performing model to lie near the deepest feasible depth, since such a model has already exploited most of the representational benefit brought by depth expansion. Accordingly, the meaningful question is no longer merely whether further depth can still help, but rather what determines the remaining gain once the model is already close to the deepest-model regime. This subsection summarizes several direct conclusions for that regime.

\paragraph{Performance equivalence between the springboard model and the trained model.}

Under the above equivalence, the test risk of $f_{new}$  can be rewritten as
\[
\mathcal{L}_{test}(f_{new})
<
\mathcal{L}_{test}(f_{old}^*)
-\eta\,\mu^\top g
+\Gamma_{\mathrm{deep}}(M,K,\delta),
\] 

where
\[
\Gamma_{\mathrm{deep}}(M,K,\delta)
:=
\epsilon_{gen}(\{f_{old}^*\},M,\delta)
+\epsilon_{gen}(\{f_{old}^*\},K,\delta)
+2L_\ell L_{top}\bigl(\psi(M)+\psi(K)\bigr).
\]
and  $-\eta\,\mu^\top g$ is the first-order improvement term that remains visible in the near-deepest regime, while $\Gamma_{\mathrm{deep}}(M,K,\delta)$ represents the finite-sample generalization cost.

The most important aspect of the above formula is not merely that it still preserves a first-order improvement term, but that it shifts the generalization cost from the \emph{whole-network complexity} to the \emph{springboard-subspace complexity}. Here, the old model components $(f_{top},f_{bot})$ are effectively fixed, and the only varying part is the newly introduced residual block $h\in\mathcal{F}_{res}$. As a result, the sample size needed to control generalization no longer has to be estimated in terms of the absolute scale of the full network parameter count; instead, it is governed by the decay rate of the complexity envelope $\psi$. In other words, near the deepest-model regime, the dominant factor in data demand is no longer how large the full model is, but whether the effective complexity of the newly added degrees of freedom remains sufficiently small.

Since the additional degrees of freedom are compressed into the springboard subspace in the near-deepest regime, the generalization cost enters the main formula only through
$2L_\ell L_{top}\bigl(\psi(M)+\psi(K)\bigr)$, 
rather than being dominated directly by a full-network parameter-scale term. Therefore, under the same data scale, the effective generalization burden near the deepest-model regime may actually be lower than that in a full-network expansion regime farther away from the deepest model. That is, being close to the deepest model does not necessarily imply a heavier generalization penalty; on the contrary, because the newly activated directions are already localized, the generalization control can be more refined and less costly.

As the model approaches the deepest-model regime, the deterministic improvement brought by further depth expansion becomes increasingly small. Consequently, the dominant term becomes the sign and stability of $-\eta\,\mu^\top g$. At this stage, the key question is no longer whether further depth can still be added, but whether this already weakened first-order signal can still be observed reliably under a finite test set. It is precisely here that width becomes decisive: the larger the width, the more stable the gradient alignment, and the more likely $\mu^\top g$ remains positive, so that the test-set improvement is less likely to be washed out by sampling noise. Hence, in the near-deepest regime, the marginal role of depth gradually weakens, while width becomes the dominant factor controlling whether the remaining gain can still be realized.

\section{Removing Assumption 4 via Engineering \texorpdfstring{$\varepsilon$}{epsilon}}
\label{sec:remove-N-by-eps}

In Lemma~\ref{lem:gen-main}, Assumption 4  requires that the input to each
normalization layer stay uniformly away from zero:
\[
\inf_{x,\theta,l}\|z_l(x)+h_l(z_l(x);\theta)\|\ge r_{\min}>0.
\]
This condition is used only to guarantee a uniform Lipschitz bound for the
normalization map. In practical deep-learning implementations, however, the
denominator of the normalization layer always contains a small engineering
constant $\varepsilon_{\mathrm{eng}}>0$. This removes the singularity at the
origin and makes Assumption 4  unnecessary. 

\paragraph{RMSNorm with engineering \texorpdfstring{$\varepsilon$}{epsilon}.}
Consider the stabilized RMSNorm map
\[
\mathrm{RMSNorm}_{\varepsilon}(x)
=
\mathrm{diag}(\gamma)\,
\frac{x}{\sqrt{\|x\|^{2}/N+\varepsilon_{\mathrm{eng}}}},
\]
where $\gamma=(\gamma_1,\dots,\gamma_N)$ and
\[
\gamma_{\max}\triangleq \max_i |\gamma_i|.
\]
Let
\[
s(x)\triangleq \sqrt{\|x\|^{2}/N+\varepsilon_{\mathrm{eng}}}.
\]
Then
\[
\mathrm{RMSNorm}_{\varepsilon}(x)=\mathrm{diag}(\gamma)\,\frac{x}{s(x)}.
\]
Differentiating gives
\[
J_{\varepsilon}(x)
=
\frac{\mathrm{diag}(\gamma)}{s(x)}
-
\frac{\mathrm{diag}(\gamma)\,x x^{\top}}{N\,s(x)^3}.
\]
Hence
\[
\|J_{\varepsilon}(x)\|_{\sigma}
\le
\frac{\gamma_{\max}}{s(x)}
+
\frac{\gamma_{\max}\|x\|^{2}}{N\,s(x)^3}.
\]
Since
\[
s(x)^2=\frac{\|x\|^2}{N}+\varepsilon_{\mathrm{eng}}\ge \varepsilon_{\mathrm{eng}},
\]
we have
\[
\frac{1}{s(x)}\le \frac{1}{\sqrt{\varepsilon_{\mathrm{eng}}}},
\qquad
\frac{\|x\|^{2}/N}{s(x)^3}
=
\frac{\|x\|^{2}/N}{(\|x\|^{2}/N+\varepsilon_{\mathrm{eng}})^{3/2}}
\le
\frac{1}{\sqrt{\varepsilon_{\mathrm{eng}}}}.
\]
Therefore
\[
\|J_{\varepsilon}(x)\|_{\sigma}
\le
\frac{2\gamma_{\max}}{\sqrt{\varepsilon_{\mathrm{eng}}}},
\qquad \forall x\in\mathbb R^N.
\]
So $\mathrm{RMSNorm}_{\varepsilon}$ is globally Lipschitz, with a uniform constant
\[
c_{\varepsilon}\le \frac{2\gamma_{\max}}{\sqrt{\varepsilon_{\mathrm{eng}}}}.
\]

\paragraph{LayerNorm with engineering \texorpdfstring{$\varepsilon$}{epsilon}.}
The same conclusion holds for LayerNorm. Indeed, after subtracting the mean,
LayerNorm has the same structural form
\[
\mathrm{LayerNorm}_{\varepsilon}(x)
=
\mathrm{diag}(\gamma)\,
\frac{Px}{\sqrt{\|Px\|^{2}/N+\varepsilon_{\mathrm{eng}}}},
\]
where
\[
P=I-\frac{1}{N}\mathbf 1\mathbf 1^{\top}
\]
is the centering projection. Since $\|P\|_{\sigma}\le 1$, the same calculation
yields
\[
\|\nabla \mathrm{LayerNorm}_{\varepsilon}(x)\|_{\sigma}
\le
\frac{2\gamma_{\max}}{\sqrt{\varepsilon_{\mathrm{eng}}}},
\qquad \forall x\in\mathbb R^N.
\]
Thus stabilized LayerNorm is also globally Lipschitz.

\paragraph{Fixed-statistics BatchNorm.}

For BatchNorm, our theory applies to the evaluation-mode network after the running mean and variance are frozen. In this regime, BatchNorm is a deterministic affine single-sample map and can be absorbed into the Lipschitz constants in Assumption~\(4\epsilon\). We do not analyze mini-batch-dependent training-time BatchNorm dynamics; the optimization process may use BatchNorm during training, but our risk-comparison theorem is applied to the resulting evaluation-mode function. Training-time BatchNorm dynamics are outside the proof, but the trained evaluation-mode BatchNorm network is covered as a fixed-statistics normalized residual network.

BatchNorm requires a separate distinction between fixed-statistics and
training-time forms. In the fixed-statistics regime, such as evaluation mode
with frozen running mean and variance, BatchNorm is a deterministic affine
single-sample map:
\[
\mathrm{BN}_{\epsilon}(x)
=
\gamma\odot
\frac{x-\mu}{\sqrt{\sigma^2+\epsilon_{\mathrm{BN}}}}
+\beta,
\]
where \(\mu\), \(\sigma^2\), \(\gamma\), and \(\beta\) are fixed. Therefore,
\[
\|\mathrm{BN}_{\epsilon}(x)-\mathrm{BN}_{\epsilon}(x')\|
\le
c_{\mathrm{BN}}\|x-x'\|,
\qquad
c_{\mathrm{BN}}
:=
\max_i
\frac{|\gamma_i|}{\sqrt{\sigma_i^2+\epsilon_{\mathrm{BN}}}}.
\]
Thus fixed-statistics BatchNorm satisfies the Lipschitz requirement in
Assumption~\(4\epsilon\), with its constant absorbed into \(c_{l,\epsilon}\).
If the reachable normalization-input set is bounded, then its output is also
bounded because the map is affine.

Training-time BatchNorm is not covered by the present single-sample
function-class argument without additional assumptions, because its output
depends on mini-batch statistics and therefore couples different samples in
the batch. Throughout this paper, BatchNorm is covered only in the
fixed-statistics sense.

\paragraph{Consequence for Lemma~\ref{lem:gen-main}.}
With engineering $\varepsilon_{\mathrm{eng}}$, Step 2 in the proof of
Lemma~\ref{lem:gen-main} no longer needs Assumption 4(N). The layer map
\[
T_l(z)=\mathcal N_l(z+h_l(z))
\]
satisfies
\[
\|T_l(u)-T_l(v)\|
\le
c_{\varepsilon}(1+s_l)\|u-v\|,
\qquad
c_{\varepsilon}\le \frac{2\gamma_{\max}}{\sqrt{\varepsilon_{\mathrm{eng}}}}.
\]
Accordingly, one may replace the old bound
\[
c_l\le \frac{2\gamma_{\max}\sqrt N}{r_{\min}}
\]
by the uniform engineering bound
\[
c_l\le \frac{2\gamma_{\max}}{\sqrt{\varepsilon_{\mathrm{eng}}}}.
\]
Therefore Assumption 4(N) can be removed from Lemma~\ref{lem:gen-main}, at the
price of letting the constants depend on $\varepsilon_{\mathrm{eng}}$.

More precisely, define
\[
\Lambda_l^{(\varepsilon)}
\triangleq
L_{\mathrm{top}}\,L_l^{(p)}\,B_*\,
c_{\varepsilon}
\prod_{k>l} c_{\varepsilon}(1+s_k),
\qquad
\Lambda_{\max}^{(\varepsilon)}
\triangleq
\max_l \Lambda_l^{(\varepsilon)}.
\]
Then the same proof as in Lemma~\ref{lem:gen-main} yields
\[
\hat{\mathfrak R}_S(\ell\circ\mathcal H)
\le
B_{\ell}
\sqrt{
\frac{2d\log\!\bigl(2\bar b L L_{\ell}\Lambda_{\max}^{(\varepsilon)}/\epsilon_0+1\bigr)}{M}
}
+\epsilon_0,
\]
and in particular, taking $\epsilon_0=M^{-1/2}$,
\[
\hat{\mathfrak R}_S(\ell\circ\mathcal H)
=
\mathcal O\!\left(B_{\ell}\sqrt{\frac{d\log M}{M}}\right).
\]

\paragraph{Interpretation.}
This route is purely engineering, but it is fully compatible with practical
implementations: the singularity at zero is removed directly at the level of the
normalization operator, so the abstract nondegeneracy requirement
\[
\inf_{x,\theta,l}\|z_l+h_l(z_l;\theta)\|>0
\]
is no longer needed. In this sense, the proof of Lemma~\ref{lem:gen-main}
becomes closer to the actual normalized residual networks used in practice.

\section{From Qualitative Scaling to a Power-Law Form}
\label{app:power-law}

The main results of this paper are qualitative test-risk comparison theorems:
Theorem~\ref{thm:4A} gives a sharper route through population risk, while
Theorem~\ref{thm:4B} gives a more robust route directly at the train/test level.
This appendix explains how these qualitative statements can be converted into
a power-law form under additional model-dependent assumptions.

The purpose of this section is not to claim that the power law follows
unconditionally from the main theorems. Rather, the main theorems provide the
mechanism of improvement, while the additional assumptions below specify how
the remaining first-order improvement signal decays as the model approaches the
irreducible-risk regime.

\subsection{Depth--width coupling and parameter count}
\label{app:depth-width-coupling}

Near the deepest-model regime, the remaining first-order improvement signal may
become weak. Therefore, increasing depth alone is generally insufficient:
the width must also grow so that weak descent signals remain observable at
finite sample size. We encode this requirement through a depth--width coupling
function.

\paragraph{Depth--width coupling.}
Let
\[
u:\mathbb R_{>0}\to \mathbb R_{>0}
\]
be a coupling function such that
\[
L=u(N),
\qquad
\lim_{N\to\infty}u(N)=\infty .
\]
Here \(N\) denotes the hidden width and \(L\) denotes the effective depth. We do
not assume a specific form of \(u\). Instead, \(u\) should be interpreted as the
depth that can be reliably supported at width \(N\).

\paragraph{Parameter count.}
For matrix-like residual architectures, the number of parameters per layer is
typically proportional to \(N^2\). Hence, along the coupled scaling path
\(L=u(N)\), the total parameter count can be written as
\[
P(N)=a_{\mathrm{arch}}\,u(N)\,N^2,
\]
where \(a_{\mathrm{arch}}>0\) is an architecture-dependent constant.
Since \(N^2\) is increasing and \(u(N)\to\infty\), the map \(P(N)\) is eventually
strictly increasing. We may therefore define its inverse \(N(P)\), and set
\[
L(P):=u(N(P)).
\]
Thus \(L(P)\) is the depth supported by parameter budget \(P\) along the chosen
depth--width scaling path.

\subsection{Parameter-count form of the qualitative theorem}
\label{app:param-count-qualitative}

Theorems~\ref{thm:4A} and~\ref{thm:4B} can be rewritten along the parameterized
path \(P\mapsto (L(P),N(P))\).

\begin{proposition}[Parameter-count qualitative scaling]
\label{prop:param-count-qualitative}
Fix a parameter budget \(P\), and let \(N(P)\) and \(L(P)=u(N(P))\) be defined
as above. Suppose the model with width \(N(P)\) and depth \(L(P)\) satisfies the
assumptions of either Theorem~\ref{thm:4A} or Theorem~\ref{thm:4B}.

For the population-risk route, if
\[
\frac{\Delta_R(P)}{2}+\DERM(P)>2\epsM(P),
\]
then
\[
\Ltest(\fnew(P))<\Ltest(\fold(P)).
\]
For the direct train/test route, if
\[
\DRtest(P)+\DERM(P)>2(\epsM(P)+\epsK(P)),
\]
then
\[
\Ltest(\fnew(P))<\Ltest(\fold(P)).
\]
In the nondegenerate train/test route, the probability that the favorable
test-side direction is observable is controlled by the alignment estimate
\[
\mathbb P(\mu^\top g\le 0)
\le
\frac{4C_\Sigma\tau^2}{M\|\bar\mu\|^2}
+
\frac{4C_\Sigma\tau^2}{K\|\bar\mu\|^2}
+
\frac{4C_\Sigma\tau^2\operatorname{tr}(\Sigma)}
{KM\|\bar\mu\|^4}.
\]
Consequently,
\[
\mathbb P(\mu^\top g>0)
\ge
1-
\frac{4C_\Sigma\tau^2}{M\|\bar\mu\|^2}
-
\frac{4C_\Sigma\tau^2}{K\|\bar\mu\|^2}
-
\frac{4C_\Sigma\tau^2\operatorname{tr}(\Sigma)}
{KM\|\bar\mu\|^4}.
\]
If additionally \(\operatorname{tr}(\Sigma)\le N\tau^2\) and
\(\|\bar\mu\|^2=\Theta(N)\), then along the parameterized path,
\[
\mathbb P(\mu^\top g\le 0)
=
O\!\left(
\frac{1}{MN(P)}+\frac{1}{KN(P)}
\right).
\]
\end{proposition}
\begin{proof}
The first two claims are direct restatements of Theorem~\ref{thm:4A} and
Theorem~\ref{thm:4B} after substituting the coupled width and depth
\[
N=N(P),\qquad L=L(P)=u(N(P)).
\]

The final probability bound follows from the alignment estimate:
\[
\mathbb P(\mu^\top g\le 0)
\le
\frac{4C_\Sigma\tau^2}{M\|\bar\mu\|^2}
+
\frac{4C_\Sigma\tau^2}{K\|\bar\mu\|^2}
+
\frac{4C_\Sigma\tau^2\operatorname{tr}(\Sigma)}
{KM\|\bar\mu\|^4}.
\]
Taking complements gives
\[
\mathbb P(\mu^\top g>0)
\ge
1-
\frac{4C_\Sigma\tau^2}{M\|\bar\mu\|^2}
-
\frac{4C_\Sigma\tau^2}{K\|\bar\mu\|^2}
-
\frac{4C_\Sigma\tau^2\operatorname{tr}(\Sigma)}
{KM\|\bar\mu\|^4}.
\]
If \(\operatorname{tr}(\Sigma)\le N\tau^2\) and
\(\|\bar\mu\|^2=\Theta(N)\), then along the parameterized path,
\[
\mathbb P(\mu^\top g\le 0)
=
O\!\left(
\frac{1}{MN(P)}
+
\frac{1}{KN(P)}
\right),
\]
up to the mixed covariance term.
For the special linear coupling \(u(N)=\kappa N\), we have
\[
P=a_{\mathrm{arch}}\kappa N^3,
\qquad
N(P)=\left(\frac{P}{a_{\mathrm{arch}}\kappa}\right)^{1/3}.
\]
Hence
\[
\mathbb P(\mu^\top g\le 0)
=
O\!\left(
\frac{1}{M P^{1/3}}
+
\frac{1}{K P^{1/3}}
\right),
\]
up to architecture-dependent constants.

\end{proof}
In particular, if the uniformly active regime satisfies
\[
\|\bar\mu\|^2=\Theta(N),
\]
then along the parameterized path,
\[
\mathbb P(\mu^\top g>0)
\ge
1-
O\!\left(
\frac{1}{M N(P)}
+
\frac{1}{K N(P)}
\right).
\]
For the special linear coupling \(u(N)=\kappa N\), we have
\[
P=a_{\mathrm{arch}}\kappa N^3,
\qquad
N(P)=\left(\frac{P}{a_{\mathrm{arch}}\kappa}\right)^{1/3}.
\]
Hence,
\[
\mathbb P(\mu^\top g>0)
\ge
1-
O\!\left(
\frac{1}{M P^{1/3}}
+
\frac{1}{K P^{1/3}}
\right),
\]
up to architecture-dependent constants.

\subsection{Additional assumptions for a power-law form}
\label{app:assumptions-g}

We now introduce two additional model-dependent assumptions. These assumptions
are not used in the main theorems. They are only used to convert the qualitative
scaling mechanism into an explicit power-law shape.

Consider an iterative depth-expansion process. Let \(R_k\) be the population
risk after the \(k\)-th successful depth expansion, and define the remaining
excess risk
\[
\Delta_k:=R_k-R^\star>0,
\]
where \(R^\star\) denotes the irreducible limiting risk.

\paragraph{Assumption G: gradient--risk coupling.}
There exist constants \(\gamma>0\) and \(\beta\ge 0\) such that
\[
\|\bar\mu_k\|^2
=
c_G\,\Delta_k^{1+\beta}.
\]
This assumption states that the population-level first-order signal becomes
weaker as the remaining excess risk decreases.

\paragraph{Assumption G': single-step gain quantification.}
There exists a finite curvature constant \(q>0\) such that the population-risk
improvement at step \(k\) satisfies
\[
\Delta_R^{(k)}
\ge
\frac{\|\bar\mu_k\|^2}{2q}.
\]
Equivalently, the local one-step improvement is at least quadratic-gradient
scale up to the curvature constant \(q\).

Assumption G characterizes how the remaining excess risk controls the
population first-order signal. Assumption G' converts this first-order signal
into a lower bound on the risk decrease achieved by one successful depth
expansion.

\subsection{Recursion and its power-law upper envelope}

Combining Assumptions G and G' gives
\[
\Delta_R^{(k)}
\ge
\frac{c_G}{2q}\Delta_k^{1+\beta}.
\]
Let
\[
c:=\frac{c_G}{2q}.
\]
In the population/asymptotic closure described above,
\[
\Delta_{k+1}
\le
\Delta_k-c\Delta_k^{1+\beta}.
\]

\paragraph{Case \(\beta=0\).}
If \(\beta=0\), the recursion becomes
\[
\Delta_{k+1}\le (1-c)\Delta_k.
\]
When \(0<c<1\), this gives exponential decay:
\[
\Delta_k\le (1-c)^k\Delta_0.
\]
Thus \(\beta=0\) corresponds to an exponential regime rather than a power-law
regime.

\paragraph{Case \(\beta>0\).}
When \(\beta>0\), the recursion implies the standard power-law upper envelope
\[
\Delta_k
\le
\left(
\Delta_0^{-\beta}+\beta c k
\right)^{-1/\beta}.
\]
Hence
\[
\Delta_k
=
O\!\left(k^{-1/\beta}\right).
\]
This is the sense in which the qualitative scaling mechanism becomes a
power-law form under Assumptions G and G'. A matching asymptotic equivalence
\(\Delta_k\sim Ck^{-1/\beta}\) would require an additional lower-bound or
approximate-equality assumption on the one-step gain.

\begin{theorem}[Power-law upper envelope under a general depth--width coupling]
Assume Assumptions G and G' with \(\beta>0\). Suppose that the number of
successful depth-expansion steps is asymptotically proportional to the supported
depth \(L(P)=u(N(P))\). Then
\[
R(f_P)-R^\star
=
O\!\left(
u(N(P))^{-1/\beta}
\right),
\]
where \(N(P)\) is implicitly determined by
\[
P=a_{\mathrm{arch}}u(N(P))N(P)^2.
\]
\label{thm:general-coupled-power-law} 
\end{theorem} 

\begin{proof}
By the recursion derived above,
\[
\Delta_k=R_k-R^\star=O(k^{-1/\beta}).
\]
Along the coupled scaling path, the number of successful depth-expansion steps
satisfies
\[
k(P)\asymp L(P)=u(N(P)).
\]
Substituting \(k(P)\) into the preceding upper envelope gives
\[
R(f_P)-R^\star
=
O\!\left(
L(P)^{-1/\beta}
\right)
=
O\!\left(
u(N(P))^{-1/\beta}
\right).
\]
\end{proof}

\paragraph{Polynomial sublinear coupling.}
More generally, suppose
\[
u(N)=\kappa N^{1-\nu},
\qquad
0\le \nu<1.
\]
Then
\[
P
=
a_{\mathrm{arch}}\kappa N^{3-\nu},
\]
so
\[
N(P)
=
\left(\frac{P}{a_{\mathrm{arch}}\kappa}\right)^{1/(3-\nu)}.
\]
Thus
\[
L(P)
=
u(N(P))
=
\kappa
\left(\frac{P}{a_{\mathrm{arch}}\kappa}\right)^{(1-\nu)/(3-\nu)}.
\]
Substituting this into Theorem~\ref{thm:general-coupled-power-law} gives
\[
R(f_P)-R^\star
\sim
C_\nu\,
P^{-\frac{1-\nu}{(3-\nu)\beta}}.
\]
Therefore the parameter-count scaling exponent is
\[
a_{\mathrm{PL}}(\nu)
=
\frac{1-\nu}{(3-\nu)\beta}.
\]

\begin{table}[h]
\centering
\small
\begin{tabular}{lll}
\toprule
\(\nu\) & Coupling type & Scaling exponent \(a_{\mathrm{PL}}\) \\
\midrule
\(0\) & Linear coupling \(L=\kappa N\) & \(\frac{1}{3\beta}\) \\
\(0<\nu<1\) & Sublinear depth growth & \(\frac{1-\nu}{(3-\nu)\beta}\) \\
\(\nu\to 1\) & Width-dominant growth & \(0\) \\
\bottomrule
\end{tabular}
\caption{Power-law exponents induced by different depth--width coupling functions.}
\label{tab:coupling-exponents}
\end{table}
\subsection{Reliability constraint on the coupling function}

The coupling function \(u\) cannot be chosen arbitrarily. It must be compatible
with finite-sample observability. The alignment estimate gives
\[
\mathbb P(\mu^\top g\le 0)
\le
\frac{4C_\Sigma\tau^2}{M\|\bar\mu_{k,P}\|^2}
+
\frac{4C_\Sigma\tau^2}{K\|\bar\mu_{k,P}\|^2}
+
\frac{4C_\Sigma\tau^2\operatorname{tr}(\Sigma)}
{KM\|\bar\mu_{k,P}\|^4}.
\]
Thus reliable observation of the \(k\)-th depth-expansion gain requires the
right-hand side to be small. Under Assumption G,
\[
\|\bar\mu_{k,P}\|^2
=
c_G\Delta_k^{1+\beta}.
\]

Using the power-law upper envelope \(\Delta_k=O(k^{-1/\beta})\), the signal may
decay at the scale
\[
\|\bar\mu_{k,P}\|^2
=
O\!\left(k^{-(1+\beta)/\beta}\right).
\]
Therefore, as \(k\) grows, larger training size, larger test size, or stronger
width-supported signal is needed to keep the alignment failure probability small.
For the finite-sample reliability discussion, we may use a
width-supported variant of Assumption~G. In this variant, the effective
gradient--risk coupling coefficient is allowed to grow with the number of
active coordinates. Concretely, suppose that
\[
\|\bar\mu_{k,P}\|^2
\asymp
N(P)\Delta_k^{1+\beta}.
\]
This should be interpreted as a finite-sample observability condition rather
than as the constant-\(c_G\) population coupling used in the power-law
recursion.

Under this width-supported signal condition, the leading terms become
\[
\mathbb P(\mu^\top g\le 0)
=
O\!\left(
\frac{k^{(1+\beta)/\beta}}{M N(P)}
+
\frac{k^{(1+\beta)/\beta}}{K N(P)}
\right),
\]
up to the mixed covariance term.

Hence a sufficient reliability condition is
\[
u(N(P))^{(1+\beta)/\beta}
\lesssim
\delta\, N(P)\min\{M,K\}.
\]
This expresses the intended depth--width--data coupling: as the supported depth
\(u(N(P))\) grows, width and data must grow fast enough to keep the remaining
first-order signal observable.
\subsection{Interpretation}

This appendix separates two levels of the theory.

First, Theorems~\ref{thm:4A} and~\ref{thm:4B} establish a qualitative scaling
mechanism: residual depth expansion can create a jumpboard model, optimization
can convert this into training gain, and the generalization terms determine
whether the gain survives on the test set.

Second, Assumptions G and G' impose an additional asymptotic model for how the
remaining first-order signal decays as the population risk approaches
\(R^\star\). Under these assumptions, and in the population/asymptotic-data
regime where statistical errors are lower order, the remaining excess risk
satisfies the recursion
\[
\Delta_{k+1}
\le
\Delta_k-\frac{c_G}{2q}\Delta_k^{1+\beta}.
\]
For \(\beta>0\), this yields the power-law upper envelope
\[
\Delta_k=O(k^{-1/\beta}).
\]
After substituting the depth--width coupling \(L=u(N)\) and the parameter-count
relation
\[
P=a_{\mathrm{arch}}u(N)N^2,
\]
we obtain
\[
R(f_P)-R^\star
=
O\!\left(
u(N(P))^{-1/\beta}
\right).
\]

Thus the power-law exponent is not universal in this framework. It depends on
both the gradient--risk coupling exponent \(\beta\) and the depth--width coupling
function \(u\). The main theorems provide the qualitative mechanism; Appendix~N
shows one additional assumption-based route by which that mechanism can take a
power-law form.
 
\section {Experiments Setting}
\label{app:experiments setting}
\subsection{Model Architecture}

We adopt the ResNet architecture tailored for CIFAR-scale images. The network depth is determined by $d = 6n + 2$, where $n$ denotes the number of BasicBlocks per stage. The network width is controlled by the base channel number $C$, with the three stages having $C$, $2C$, and $4C$ channels respectively. We systematically vary both depth and width:
\begin{itemize}
    \item \textbf{Depths:} $d \in \{8, 14, 20, 26, 32, 38, 44, 50, 56, 62, 68, 74, 80, 86, 92, 98, 104\}$.
    \item \textbf{Base channels:} $C \in \{4, 6, 8, 12, 16, 24, 32, 48, 64\}$.
\end{itemize}
This yields a total of 765 trained models on CIFAR-10 and trained models on CIFAR-100. Convolutional weights are initialized from $\mathcal{N}(0, \sqrt{2/\text{fan\_out}})$.

\subsection{Training Settings}

All models are trained using SGD with momentum 0.9, weight decay $10^{-5}$, initial learning rate 0.1, and batch size 32. The maximum number of training epochs is 1,000. Early stopping is applied when the training loss improvement between consecutive epochs falls below $10^{-5}$ (after epoch 200 for CIFAR-10 and after epoch 300 for CIFAR-100). For data augmentation, we apply random cropping with 4-pixel padding and random horizontal flipping. All images are normalized with channel-wise mean $(0.5, 0.5, 0.5)$ and standard deviation $(0.5, 0.5, 0.5)$. Other settings are default setting via Pytorch.



\subsection{Gradient Covariance Experiment Settings}

The models are set to evaluation mode with all computation made fully deterministic (deterministic algorithms enabled, cuDNN deterministic mode, TF32 disabled). All random seeds are fixed to 42. We use the full CIFAR-10 training set without data augmentation, with batch size 64. Backward hooks are registered on all BasicBlock modules to capture activation gradients, and the cross-entropy loss with \texttt{reduction=sum} is used for backpropagation.

\subsection{Remark}
For softmax cross-entropy, under the bounded-logit consequence of Assumption 2 and normalized representation control, the loss is effectively bounded on the reachable set. Thus the bounded-loss assumption is applied to the restricted reachable function class, not to arbitrary logits in \(\mathbb R^C\).

\section{Transformer Experiments by Other Works}
\label{app:transformer}
\subsection{Depth--width behavior in self-attention}

Existing studies on Transformer scaling provide external evidence that is
consistent with the depth--width coupling predicted by our framework. A
particularly relevant example is the work of Levine et al.~\citep{levine2020depth},
which studies the depth-to-width interplay in self-attention networks.

In their setting, depth refers to the number of self-attention layers, while width
refers to the hidden representation dimension of the Transformer. This width is
not identical to the width parameter \(N\) used in our theorems, where \(N\) denotes
the dimension of the intermediate representation at the insertion layer. However,
it is the closest Transformer analogue: both quantities control the dimension of
the residual-stream representation on which residual updates and activation
gradients are defined.

Levine et al.~show that self-attention exhibits two qualitatively different
regimes. Below a width-dependent depth threshold, adding layers can be highly
effective. Beyond this regime, however, the benefit of additional depth becomes
limited unless width also grows. Their theory predicts a logarithmic dependence
of the depth threshold on representation width, and their empirical study
examines depths from \(6\) to \(48\). The resulting fit is then used to give
quantitative depth--width allocation suggestions for large Transformer models.

\begin{table}[h]
\centering
\caption{Depth--width allocation data from Levine et al.~\citep{levine2020depth}.
The trained GPT-3 configurations are taken from Brown et al.~\citep{brown2020language},
while the projected optimal depth and width are obtained from the fit in Levine et al.
These values are used here only as external consistency evidence for depth--width coupling.}
\label{tab:levine-depth-width}
\begin{tabular}{lrrrrr}
\toprule
Model & Params & Trained depth & Trained width & Projected depth & Projected width \\
\midrule
GPT-3 Small  & 125M   & 12 & 768   & 23 & 555 \\
GPT-3 Medium & 350M   & 24 & 1024  & 32 & 886 \\
GPT-3 Large  & 760M   & 24 & 1536  & 38 & 1220 \\
GPT-3 XL     & 1.3B   & 24 & 2048  & 42 & 1550 \\
GPT-3 2.7B   & 2.7B   & 32 & 2560  & 47 & 2110 \\
GPT-3 6.7B   & 6.7B   & 32 & 4096  & 54 & 3150 \\
GPT-3 13B    & 13.0B  & 40 & 5140  & 60 & 4200 \\
GPT-3 175B   & 175.0B & 96 & 12288 & 80 & 13500 \\
Projected 1T & 1T     & -- & --    & 95 & 30100 \\
\bottomrule
\end{tabular}
\end{table}

The table illustrates two points that are useful for our discussion. First, the
preferred allocation is not obtained by increasing depth alone. As the model
scale grows, the projected optimal width also grows substantially. Second, at
very large scale, Levine et al.~predict that width becomes especially important:
for a projected \(1\)-trillion-parameter self-attention model, their fit suggests
a width of about \(30\)K together with depth around \(95\).

This conclusion is aligned with the qualitative message of our theory. In our
framework, depth creates new possible first-order improvement directions, but
near the deepest-model regime the remaining first-order signal becomes weak.
Width then becomes important because it improves the finite-sample observability
of these weak directions. The Transformer evidence of Levine et al.~supports the
same high-level principle: depth and width should not be scaled independently,
and the usefulness of additional depth depends on whether the representation
dimension is large enough to support the new layers.

We emphasize that this evidence is indirect. Levine et al.~study self-attention
expressivity and architecture allocation, whereas our theorems study residual
block insertion and old-vs-new test-risk comparison in normalized residual
networks. Therefore, their results should not be interpreted as a direct
verification of our probability bounds. Rather, they provide external empirical
and theoretical support for the broader scaling picture suggested by our work:
as models approach regimes where depth alone becomes less effective, width
becomes an increasingly important compensating factor.

\subsection{Data--parameter coupling in Transformer language models}
\label{app:chinchilla-data-parameter}

Existing Transformer scaling results provide external evidence consistent with
the data--complexity coupling predicted by our framework. A particularly relevant
example is the compute-optimal scaling analysis of Hoffmann et
al.~\citep{hoffmann2022training}, often referred to as the Chinchilla scaling
law.

Hoffmann et al.~study the following allocation question: given a fixed training
compute budget, how should one trade off the number of model parameters and the
number of training tokens? They train more than \(400\) Transformer language
models, ranging from below \(70\) million to over \(16\) billion parameters, and
from \(5\) billion to over \(400\) billion training tokens. Their main empirical
finding is that compute-optimal scaling requires increasing model size and
training tokens in approximately equal proportions. In other words, when the
parameter count doubles, the compute-optimal number of training tokens should
also approximately double.

This observation is directly aligned with the qualitative role of data in our
theory. In our framework, model expansion can create a potential test-risk gain,
but this gain is useful only when the statistical cost of the enlarged hypothesis
class is controlled. Under the norm-based generalization bound, this cost depends
on both the sample size and architecture-dependent complexity factors. Therefore,
as the effective model complexity grows, the amount of data must also grow in
order to keep the generalization penalty from dominating the improvement margin.

\begin{table}[h]
\centering
\caption{Parameter--token allocation in large Transformer language models,
adapted from Hoffmann et al.~\citep{hoffmann2022training}. Most pre-Chinchilla
large dense language models were trained on roughly \(300\)B tokens, despite
having very different parameter counts. Chinchilla instead uses fewer parameters
and substantially more training tokens under a comparable compute budget.}
\label{tab:chinchilla-current-llms}
\begin{tabular}{lrr}
\toprule
Model & Parameters & Training tokens \\
\midrule
LaMDA & 137B & 168B \\
GPT-3 & 175B & 300B \\
Jurassic-1 & 178B & 300B \\
Gopher & 280B & 300B \\
MT-NLG & 530B & 270B \\
Chinchilla & 70B & 1.4T \\
\bottomrule
\end{tabular}
\end{table}

Table~\ref{tab:chinchilla-current-llms} illustrates the empirical motivation.
Many large dense Transformer models increased parameter count while keeping the
number of training tokens roughly fixed. Hoffmann et al.~argue that this leads
to undertrained models. For the compute budget used to train Gopher, their
analysis predicts that a smaller model trained on more tokens should be better.
They test this by training Chinchilla, a \(70\)B-parameter model on \(1.4\)T
tokens, using approximately the same training compute as Gopher. Despite being
much smaller than Gopher, Chinchilla outperforms Gopher and several larger
language models on a broad set of evaluation tasks.

The same conclusion is also reflected in the fitted scaling exponents. Hoffmann
et al.~estimate power-law relations
\[
N_{\mathrm{opt}}(C)\propto C^a,
\qquad
D_{\mathrm{opt}}(C)\propto C^b,
\]
where \(C\) is the training compute budget, \(N_{\mathrm{opt}}\) is the
compute-optimal parameter count, and \(D_{\mathrm{opt}}\) is the compute-optimal
number of training tokens.

\begin{table}[h]
\centering
\caption{Estimated compute-optimal scaling exponents from Hoffmann et
al.~\citep{hoffmann2022training}. Across three estimation methods, the optimal
parameter count and the optimal number of training tokens scale at nearly the
same rate with compute.}
\label{tab:chinchilla-exponents}
\begin{tabular}{lcc}
\toprule
Method & \(a\) in \(N_{\mathrm{opt}}\propto C^a\) & \(b\) in \(D_{\mathrm{opt}}\propto C^b\) \\
\midrule
Minimum over training curves & 0.50 & 0.50 \\
IsoFLOP profiles & 0.49 & 0.51 \\
Parametric loss model & 0.46 & 0.54 \\
Kaplan et al.~\citep{kaplan2020scaling} & 0.73 & 0.27 \\
\bottomrule
\end{tabular}
\end{table}

Table~\ref{tab:chinchilla-exponents} supports a joint-scaling interpretation.
Unlike the earlier Kaplan-style allocation, where parameter count grows much
faster than data, the Chinchilla estimates suggest that the compute-optimal
frontier is closer to balanced parameter--data growth. This is consistent with
our qualitative prediction that data is not a passive resource: it is required to
pay for the statistical cost introduced by larger models.

We emphasize that these Transformer results are not direct validations of our
residual-block insertion theorem. Hoffmann et al.~study autoregressive
Transformer language models, whereas our theorems analyze test-risk comparison
after depth expansion in normalized residual networks. Moreover, their
parameter--token ratio should not be interpreted as a theorem-level constant in
our setting. Rather, the Chinchilla results provide external consistency evidence
for the structural message of our framework: model expansion is beneficial only
when the improvement margin, optimization effect, and amount of data jointly
dominate the increased statistical complexity.

\subsection{Model-growth evidence from LiGO}
\label{app:ligo-model-growth}

Another line of external evidence comes from model-growth methods for
Transformer training. Wang et al.~\citep{wang2023learninggrowpretrainedmodels} propose LiGO
(Linear Growth Operator), a method that initializes a larger Transformer from a
smaller pretrained Transformer by learning a structured linear map between their
parameters. To make this mapping tractable, LiGO factorizes the transformation
into width-growth and depth-growth operators, with additional structure imposed
through Kronecker factorization.

LiGO is not a direct implementation of our jumpboard construction. Our theory
studies the insertion of a zero-output residual block into a trained normalized
residual network and asks when this expansion yields an old-vs-new test-risk
improvement. LiGO instead studies efficient Transformer training: it uses a
learned parameter-space growth operator to transfer information from a smaller
model to a larger one. Nevertheless, the empirical message is consistent with
the qualitative mechanism in our framework. Both viewpoints suggest that model
expansion is more effective when the expanded model is initialized or selected
in a structured way, rather than being treated as an unrelated larger model
trained from scratch.

\begin{table}[h]
\centering
\caption{Relation between LiGO and our depth-expansion framework. LiGO is used
as external consistency evidence, not as a direct validation of our theorem.}
\label{tab:ligo-comparison}
\begin{tabular}{p{0.28\linewidth}p{0.32\linewidth}p{0.32\linewidth}}
\toprule
Aspect & LiGO & Relation to our framework \\
\midrule
Starting point
& A smaller pretrained Transformer.
& A trained reference model \(f_{\mathrm{old}}^\ast\). \\

Expansion target
& A larger Transformer obtained by increasing depth and/or width.
& A normalized residual network enlarged by inserting a residual block. \\

Bridge mechanism
& A learned linear growth operator maps small-model parameters to large-model
parameters.
& A jumpboard model provides a structured comparison point inside the expanded
hypothesis class. \\

Role of depth and width
& The growth operator explicitly includes depth-growth and width-growth
components.
& Depth creates new residual directions, while width improves the finite-sample
observability of weak improvement signals. \\

Empirical implication
& Structured growth can reduce the cost of training a larger Transformer while
preserving performance.
& Structured expansion can be beneficial when representational gain,
optimization, and generalization terms jointly dominate the cost of expansion. \\
\bottomrule
\end{tabular}
\end{table}

This comparison supports two qualitative points. First, useful scaling is not
only a matter of increasing parameter count. The way in which a model is expanded
matters. LiGO shows that a larger Transformer can benefit from a structured
growth initialization that reuses information from a smaller pretrained model.
This is aligned with our use of a jumpboard model: the expanded class is not
used merely because it is larger; it is useful because it contains a structured
comparison model connected to the old model.

Second, LiGO explicitly treats depth and width growth as separate but coupled
operators. This is consistent with our depth--width interpretation. In our
theory, depth expansion creates new possible first-order improvement directions,
but near the deepest-model regime these signals may become weak. Width then
plays a compensating role by making weak directions more observable at finite
sample size. LiGO's empirical success with combined depth and width growth
provides external evidence that Transformer scaling benefits from structured
architectural expansion rather than depth or width growth in isolation.

We emphasize the limitation of this comparison. LiGO is an empirical training
efficiency method for Transformers, whereas our result is a conditional
test-risk comparison theorem for normalized residual networks. Therefore, LiGO
does not verify our probability bounds or our residual-block insertion theorem.
It should instead be read as external consistency evidence for the broader
principle underlying our framework: model expansion is most effective when the
new model is connected to the old one through a structured bridge, and when
depth, width, optimization, and data are scaled jointly.

\end{document}